\documentclass[11pt]{article}
\usepackage{fullpage}
\usepackage{tablefootnote}
\usepackage{hyperref,psfrag}

\usepackage{graphicx} % more modern
\usepackage{subfigure}
\usepackage{caption}
\usepackage{natbib}

\usepackage{algorithm}
\usepackage{algorithmic}

\def\viz{{viz.,\ \/}}
\def\Rbb{{\mathbb R}}
\def\Fbb{{\mathbb F}}
\def\tha{{\mbox{\tiny th}}}
\def\Ebb{{\mathbb E}}

\def\order{\mathcal{O}}
\usepackage{hyperref}
\usepackage{color}
\usepackage{amsfonts, mathtools, amssymb, amsthm, mathrsfs, wasysym} %math
\usepackage{graphicx, subfigure,epsfig,epsf, epstopdf,sublabel} %figures
\usepackage{tabularx}

\usepackage{tikz}
\usetikzlibrary{calc,shapes}
\usepackage{alphalph}

\usepackage{perpage} %the perpage package
\MakePerPage{footnote} %the perpage package command

\DeclareMathOperator{\Var}{{Var}}

\DeclareMathOperator{\Tr}{{Tr}}
\DeclareMathOperator{\If}{{ If}}

\DeclareMathOperator{\Else}{{Else}}
\newcommand{\st}{{s.t. }}
\DeclareMathOperator{\sign}{{sign }}
\newcommand{\supp}{{supp}}

\usepackage{setspace}
\def\lnorm{{\lvert\!\lvert\!\lvert}}
\def\rnorm{{\rvert\!\rvert\!\rvert}}
\DeclarePairedDelimiter\gennorm{\lnorm}{\rnorm}
\newtheorem{lemma}{Lemma}
\newtheorem{theorem}{Theorem}
\newtheorem{proposition}{Proposition}
\newtheorem{corollary}{Corollary}

\usepackage{appendix}

\begin{document}

\title{Multi-Step Stochastic ADMM in High Dimensions:  \\
Applications to Sparse Optimization and Noisy Matrix Decomposition}
\author{
  Hanie Sedghi\footnote{University of Southern California, Email: hsedghi@usc.edu}
 % \texttt{first1.last1@xxxxx.com}
  \and
  Anima Anandkumar\footnote{University of California, Irvine, Email: a.anandkumar@uci.edu}%\\
  %\texttt{first2.last2@xxxxx.com}
  \and
  Edmond Jonckheere \footnote{University of Southern California,  Email: jonckhee@usc.edu}
}
%}%\\
%  %\texttt{first2.last2@xxxxx.com}
%  \and
%  Edmond Jonckheere \footnote{University of Southern California,  Email: jonckhee@usc.edu}
%}
%Hanie Sedghi\\
%Department of Electrical Engineering\\
%University of SouthernCalifornia\\
%\texttt{hsedghi@usc.edu}
%\date{\today}
%\date{}
%\editor{}

\maketitle

%\twocolumn[
%\onecolumn
%\title{Guarantees for Stochastic ADMM in High Dimensions}

% unless you've provided the [accepted] option to the icml2014
%\author{Your Name}{email@yourdomain.edu}
%\address{Your Fantastic Institute,
%            314159 Pi St., Palo Alto, CA 94306 USA}
%\icmlauthor{Your CoAuthor's Name}{email@coauthordomain.edu}
%\icmladdress{Their Fantastic Institute,
%            27182 Exp St., Toronto, ON M6H 2T1 CANADA}


%\keywords{Stochastic ADMM, $\ell_1$-regularization, sparse+low rank matrix decomposition,  bounds, high dimensional regime}

%\vskip 0.3in

%\onecolumn

\begin{abstract}\noindent We propose an efficient ADMM method with guarantees for high-dimensional problems. We provide explicit bounds for the sparse optimization problem and the noisy   matrix decomposition problem. For sparse optimization, we establish that the modified ADMM method has   an optimal convergence rate of $\order(s\log d/T)$, where $s$ is the sparsity level, $d$ is the data dimension and $T$ is the number of steps. This matches with the minimax lower bounds for sparse estimation. For  matrix decomposition into sparse and low rank components, we provide the first guarantees for any online method, and prove a convergence rate of $\tilde{\order}((s+r)\beta^2(p) /T) + \order(1/p)$ for a $p\times p$ matrix, where $s$ is the sparsity level, $r$ is the rank and $\Theta(\sqrt{p})\leq \beta(p)\leq \Theta(p)$. Our guarantees match the minimax lower bound with respect to $s,r$ and $T$. In addition, we match the minimax lower bound with respect to the matrix dimension $p$, i.e. $\beta(p)=\Theta(\sqrt{p})$, for many important statistical models including  the independent noise model, the linear Bayesian network and the latent Gaussian graphical model under some conditions. Our ADMM method is based on epoch-based annealing and consists of inexpensive steps which involve projections on to simple norm balls. Experiments show that for both sparse optimization and matrix decomposition problems, our algorithm outperforms  the state-of-the-art  methods. In particular, we reach higher accuracy with same time complexity.
\end{abstract}


\section{Introduction}
Stochastic optimization techniques have been extensively employed for
online machine learning on data which is uncertain, noisy or missing.  Typically it involves performing a large number of  inexpensive iterative updates,  making it scalable for large-scale learning. In contrast,  traditional batch-based techniques involve far more expensive operations for each update step. Stochastic optimization has been analyzed in a number of recent works, e.g., \citep{shalev2011online,boyd2011distributed,AgarwalNW12,NIPS2013_5034,NIPS2013_4937,NIPS2013_4938}.

The alternating direction method of multipliers
(ADMM)  is a popular method for online and distributed optimization on a large scale~\citep{boyd2011distributed}, and is employed in many applications, e.g.,~\citep{wahlberg2012admm},~\citep{esser2010general},~\citep{mota2012distributed}.
It can be viewed as a decomposition procedure where solutions to sub-problems are found locally, and coordinated via constraints to find the global solution. Specifically, it is a form of augmented Lagrangian method which applies partial updates to the dual variables. ADMM is often applied to solve regularized problems, where the function optimization and regularization can be carried out locally, and then coordinated globally via constraints.
Regularized optimization problems are especially relevant in the high dimensional regime since regularization is a natural mechanism to overcome ill-posedness and to encourage parsimony in the optimal solution, e.g., sparsity and low rank. Due to the efficiency of ADMM in solving regularized problems, we employ it in this paper.

In this paper, we design a modified version of the stochastic ADMM method  for high-dimensional problems. We first analyze the simple setting, where the optimization problem consists of a loss function and a single regularizer, and then extend to  the multi-block setting with multiple regularizers and multiple variables. For illustrative purposes, for the first setting, we consider the sparse optimization problem and for the second setting,  the matrix decomposition problem respectively. Note that our results easily extend to other settings, e.g., those in~\citet{negahban2012unified}.

We consider a simple modification to the (inexact) stochastic ADMM method~\citep{ouyang2013stochastic} by incorporating multiple steps or epochs, which can be viewed as a form of annealing. We establish that this simple modification has huge implications in achieving tight convergence rates as the dimensions of the problem instances scale.
In each iteration of the method, we employ  projections on to certain norm balls of appropriate radii, and we decrease the radii in epochs over time. The idea of annealing was first introduced by~\citet{AgarwalNW12} for dual averaging. Yet, that method cannot be extended for multivariable cases.

 For instance, for the sparse optimization problem, we constrain the optimal solution at each step to be within an $\ell_1$-norm ball of the initial estimate, obtained at the beginning of each epoch. At the end of the epoch, an average is computed and passed on to the next epoch as its initial estimate. Note that the $\ell_1$ projection can be solved efficiently in linear time, and can also be parallelized easily~\citep{duchi2008efficient}.

For matrix decomposition with a general loss function, the ADMM method requires  multiple blocks for updating the low rank and sparse components. We apply the same principle and project  the sparse and low rank estimates on to $\ell_1$ and nuclear norm balls, and these projections can be computed efficiently.

\paragraph{ Theoretical implications: } The above simple   modifications to ADMM have huge implications for high-dimensional problems. For sparse optimization, our convergence rate is  $\mathcal{O}(\frac{s\log d}{T})$, for $s$-sparse problems in $d$ dimensions in $T$ steps. Our bound has the best of both worlds: efficient high-dimensional scaling (as $\log d$) and efficient convergence rate (as $\frac{1}{T}$). This also matches the minimax lower bound for the linear model and square loss function~\citep{raskutti2011minimax}, which implies that our guarantee is unimprovable by any (batch or online) algorithm (up to constant factors).  For matrix decomposition, our convergence rate is $\mathcal{O}( (s+r)\beta^2(p) \log p /T))+\mathcal{O}(\max\lbrace s+r,p \rbrace/p^2)$ for a $p\times p$ input matrix in $T$ steps, where the sparse part has $s$ non-zero entries and low rank part has rank $r$. For many natural noise models (e.g. independent noise, linear Bayesian networks), $\beta^2(p)=p$, and the resulting convergence rate is minimax-optimal. Note that our bound is not only on the reconstruction error, but also on the error in recovering the sparse and low rank components.  These are the first convergence guarantees for online matrix decomposition in high dimensions. Moreover, our convergence rate holds {\em with high probability} when noisy samples are input, in contrast to expected convergence rate, typically analyzed in literature.
 See Table~\ref{table:vector},~\ref{table:matrix} for comparison of this work with related frameworks.

\paragraph{ Practical implications: } The proposed algorithms %Regularized Epoch-based Admm for Stochastic Optimization in high-dimensioN, REASON 1 and 2
provide significantly faster convergence in high dimension and better robustness to noise. For  sparse optimization, our method has significantly better accuracy compared to the stochastic ADMM method and { better performance than} RADAR, based on multi-step dual averaging~\citep{AgarwalNW12}. For matrix decomposition,  we compare our method  with the state-of-art inexact ALM~\citep{lin2010augmented} method. While both methods have similar reconstruction performance, our method has significantly better accuracy in recovering the sparse and low rank components.

%\subsection*{Related Work}
\paragraph{ Related Work: ADMM: }
%The alternating direction method of multipliers
%(ADMM), originated as a batch method. Currently it is a popular method for online and distributed optimization on a large scale~\citep{boyd2011distributed}, and is employed in many applications
%%, e.g.,~\citep{wahlberg2012admm},~\citep{esser2010general},~\citep{mota2012distributed}.
% It can be viewed as a decomposition procedure where solutions to sub-problems are found locally, and coordinated via projections to find the global solution. Specifically, it is a form of augmented Lagrangian method which applies partial updates to the dual variables.
Existing online ADMM-based methods lack high-dimensional guarantees.
%}
They scale poorly with the data dimension (as $\order(d^2)$), and also have slow convergence for general problems (as  $\order(\frac{1}{\sqrt{T}})$). Under strong convexity, the convergence rate can be improved to  $\order(\frac{1}{T})$ but only in {\em expectation}: such analyses ignore the per sample error and consider only the expected convergence rate (see Table~\ref{table:vector}). In contrast, our bounds hold with high probability. Some stochastic ADMM methods, \citet{goldstein2012fast},~\citet{deng2012global} and \citet{linearconv}, provide faster rates for stochastic ADMM, than the rate noted in Table~\ref{table:vector}. However, they require strong conditions which are not satisfied for the optimization problems considered here, e.g.,~\citet{goldstein2012fast} require both the loss function and the regularizer to be strongly convex.

It is also worth mentioning that our method provides error contraction, i.e., we can show error shrinkage after specific number of iterations whereas no other ADMM based method can guarantee this.

\paragraph{ Related Work: Sparse Optimization: }For the sparse optimization problem, $\ell_1$ regularization is employed and the underlying true parameter is assumed to be sparse. This is a well-studied problem in a number of works (for details, refer to~\citep{AgarwalNW12}). \citet{AgarwalNW12} propose an efficient online method based on {annealing} dual averaging, which achieves the same optimal rates as the ones derived in this paper. The main difference is that our ADMM method is capable of solving the problem for multiple random variables and multiple conditions while their method cannot incorporate these extensions.

\paragraph{ Related Work: Matrix Decomposition: }To the best of our knowledge, online guarantees for high-dimensional   matrix decomposition  have not been provided before. \citet{wang2013solving} propose a multi-block ADMM method for the matrix decomposition problem  but only provide convergence rate analysis in expectation and it has  poor high dimensional scaling (as $\order(p^4)$ for a $p\times p$ matrix) without further modifications.  Note that they only provide convergence rate on difference between loss function and optimal loss, whereas we provide the convergence rate on individual errors  of the sparse and low rank components $\Vert \bar{S}(T)-S^*\Vert_{\mathbb{F}}^2, \Vert \bar{L}(T)-L^*\Vert_{\mathbb{F}}^2$.
See Table~\ref{table:matrix} for comparison of guarantees for matrix decomposition problem.

We compare our guarantees in the online setting with the batch guarantees of~\citet{agarwal2012noisy}. Although other batch analyses exist for matrix decomposition, e.g.,~\citep{chandrasekaran2011rank,candes2011robust,hsu2011robust}, they require stronger assumptions based on incoherence conditions for recovery, which we do not impose here. The batch analysis by~\citet{agarwal2012noisy} requires fairly mild condition such as ``diffusivity'' of the unknown low rank matrix. Moreover, the convergence rate for the batch setting by~\citet{agarwal2012noisy} achieves the minimax lower bound (under the independent noise model), and is thus, optimal, up to constant factors.

Note that when only the weak diffusivity condition is assumed, the matrix decomposition problem suffers from an approximation error, i.e. an error even in the noiseless setting. Both the minimax rate and the batch rates in~\citep{agarwal2012noisy} have an approximation error. However, our approximation error is worse by a factor of $p$, although it is still decaying with respect to $p$.

\paragraph{ Overview of Proof Techniques: } {Note that in the main text, we provide guarantees for fixed-epoch length. However, if we use variable-length epoch size we can get a $\log d$ improvement in the convergence rate.} Our proof involves the following high-level steps to establish the convergence rate: (1) deriving convergence rate for the modified ADMM method (with variable-length epoch size) at the end of one epoch, where the ADMM estimate is compared with the batch estimate, (2) comparing the batch estimate with the true parameter, and then combining the two steps, and analyzing over multiple epochs to obtain the final bound. We can show that with the proposed parameter setting and varying epoch size, error can be halved by the end of each epoch. For the matrix decomposition problem, additional care is needed to ensure that the errors in estimating the sparse and low rank parts can be decoupled. This is especially non-trivial in our setting since we utilize multiple variables in different blocks which are updated in each iteration. Our careful analysis enables us to establish the first results for online matrix decomposition in the high-dimensional setting which match the batch guarantees for many interesting statistical models. (3) Next, we analyze how guarantees change for fixed epoch length. We prove that although the error halving stops after some iterations but the error does not increase noticeably to invalidate the analysis.

%I remove the contraction column

\begin{table*}[t]
%\label{table:compare1}
\begin{center}
\begin{tabular}{|c| c| c| c|}
\hline
Method & Assumptions & convergence \\ \hline
ST-ADMM~\citep{ouyang2013stochastic}&  L, convexity & $\mathcal{O}({d^2}/{\sqrt{T}})$ \\ \hline
ST-ADMM~\citep{ouyang2013stochastic}&  SC, E & $\mathcal{O}({d^2\log T}/T)$ \\ \hline
%ST-ADMM \citep{ouyang2013stochastic}& LL, LSC & $O(d^2/T)$ & No\\ \hline
BADMM~\citep{BADMM} & convexity, E &$\mathcal{O}(d^2/\sqrt{T})$\\ \hline
RADAR~\citep{AgarwalNW12} & LSC, LL  &$\mathcal{O}(s \log d /T)$ \\
%& LSC &&&\\
%& SGSG &&&\\
\hline
REASON 1 (this paper)& LSC, LL &$\mathcal{O}( s \log d /T)$ \\
%& SGSG&&&\\
%\hline
%Algorithm~\eqref{our_method_final}&  LSC, SGSG &$O(?)$ &$O( (\log d+s\log \log d) /T)$ & Yes\\
%& SGSG &&&\\
\hline
Minimax bound~\citep{raskutti2011minimax} & Eigenvalue conditions & $\mathcal{O}(s \log d /T)$\\
\hline
\end{tabular}
\caption{\em Comparison of online sparse optimization methods under  $s$ sparsity level for the optimal paramter, $d$ dimensional space, and $T$ number of iterations.\\
SC = Strong Convexity, LSC = Local Strong Convexity, LL = Local Lipschitz, L = Lipschitz property, E = in Expectation \\The last row provides minimax-optimal rate on error for any method. The results hold with high probability unless otherwise mentioned.}\label{table:vector}
\end{center}
\end{table*}

\begin{table*}[t]
\label{table:compare2}
\begin{center}
 \begin{tabular}{|c| c| c| c|}
\hline
Method & Assumptions & Convergence rate \\ \hline
\begin{tabular}{c} Multi-block-ADMM\\\citep{wang2013solving}\end{tabular}&  L, SC, E & $\mathcal{O}({p^4} /{T})$\\ \hline
%ST-ADMM~\citep{ouyang2013stochastic}&  LSC, \textit{In Expectation} & $\mathcal{O}({p^4\log T}/T)$\\ \hline
%%ST-ADMM \citep{ouyang2013stochastic}& LL, LSC & $O(d^2/T)$ & No\\ \hline
%BADMM~\citep{BADMM} & convexity &$\mathcal{O}(p^4/\sqrt{T})$\\
%\hline
\begin{tabular}{c} Batch method\\\citep{agarwal2012noisy}\end{tabular} & LL, LSC, DF &$\mathcal{O}((s \log p+ rp) /T)+\mathcal{O}(s/p^2)$\\
%& LSC &&&\\
%& SGSG &&&\\
\hline
REASON 2 (this paper)& LSC, LL, DF &$\mathcal{O}( (s+r)\beta^2(p) \log p /T))+\mathcal{O}(\max\lbrace s+r,p \rbrace/p^2)$\\
%& SGSG&&&\\
%\hline
%Algorithm~\eqref{our_method_final}&  LSC, SGSG &$O(?)$ &$O( (\log d+s\log \log d) /T)$ & Yes\\
%& SGSG &&&\\
\hline
\begin{tabular}{c} Minimax bound\\\citep{agarwal2012noisy}\end{tabular} & $\ell_2$, IN, DF & $\mathcal{O}((s \log p+ rp) /T)+\mathcal{O}(s/p^2)$\\
\hline
\end{tabular}
\end{center}
 \caption[]{\em Comparison of   optimization methods for sparse+low rank matrix decomposition for a $p\times p$ matrix under  $s$ sparsity level and $r$ rank matrices  and $T$ is the number of samples.\\
 SC = Strong Convexity, LSC = Local Strong Convexity, LL = Local Lipschitz, L = Lipschitz for loss function, IN = Independent noise model, DF = diffuse low rank matrix under the optimal parameter. $ \beta(p)=\Omega(\sqrt{p}), \mathcal{O}(p)$ and its value depends the model. The last row provides minimax-optimal rate on error for any method under the independent noise model. The results hold with high probability unless otherwise mentioned.\\
For Multi-block-ADMM~\citep{wang2013solving} the convergence rate is on the difference of loss function from optimal loss, for the rest of works in the table, the convergence rate is on $\Vert \bar{S}(T)-S^*\Vert_{\mathbb{F}}^2+\Vert \bar{L}(T)-L^*\Vert_{\mathbb{F}}^2$.}\label{table:matrix}
\end{table*}

%The goal is to estimate $S^*, L^*$ and we provide direct bound on this estimate.  Yet, they do not provide guarantees on indivudually estimating $S^*$, $L^*$.
%Alekh does not consider multi-block optimization.. he won't have efficient updates in each step ..

\subsection{Notation}
In the sequel, we use lower case letter for vectors and upper case letter for matrices.% Unless otherwise stated, $x \in \mathbb{R}^d$ and $d=p^2$. Moreover, $X \in \mathbb{R}^{p \times p}$. 
$\Vert x \Vert_1$, $\Vert x \Vert_2$ refer to $\ell_1, \ell_2$ vector norms respectively. The term $\Vert X \Vert_*$ stands for nuclear norm of $X$. In addition, $\Vert X \Vert_2$, $\Vert X \Vert_\mathbb{F}$ denote spectral and Frobenius norms respectively. $\gennorm{X}_\infty$ stands for induced infinity norm. We use vectorized $\ell_1,\ell_\infty$ norm for matrices. i.e., $\Vert X \Vert_1=\underset{i,j}{\sum}~\vert X_{ij} \vert$, $\Vert X \Vert_\infty=\underset{i,j}{\max}~\vert X_{ij} \vert$. %Finally $\rho(\Sigma)=\underset{j}{\max}~ \Sigma_{jj}$.
%\subsection{Paper Structure}
%The rest of this paper is organized as follows. Section \ref{sec:ProblemFormulation} specifies the problem and constraints in high dimension. We introduce our solution for sparse case in Section \ref{sec:sparse} and provide guarantees for our online ADMM-based algorithm. In section \ref{sec:latent} we show that our method can be extended to the case where more than one regularizer exists. We introduce our second algorithm that is based on generalized ADMM for more than two variable, prove guarantees and show that as a special case, we can solve Latent Variable Gaussian Graphical Model Selection problem in high dimension \textcolor{red}{with $\mathcal{O}(\frac{(s+r)\log d}{T}$ .}\\
\section{Problem Formulation} \label{sec:ProblemFormulation}

Consider the optimization problem
\begin{equation}
\theta^* \in \underset{\theta \in \Omega}{\arg \min}~ {\mathbb{E}[f(\theta,x)]},
\end{equation}
where $x\in\mathbb{X}$ is a random variable  and $f:\Omega \times \mathbb{X}\rightarrow \mathbb{R}$ is a given loss function. Since only samples are available, we employ the  empirical estimate of $\widehat{f}(\theta):=1/n\sum_{i\in [n]} f(\theta,x_i)$ in the optimization. For high-dimensional $\theta$, we need to impose a regularization $\mathcal{R}(\cdot)$, and
\begin{equation}\label{eqn:batch}\widehat{\theta}:= {\arg \min}\lbrace \widehat{f}(\theta)+\lambda_n\mathcal{R}( {\theta}) \rbrace,\end{equation} is the batch optimal solution.

For concreteness we focus on the sparse optimization and the matrix decomposition problem.
It is straightforward to generalize our results to other settings, say~\citep{negahban2012unified}. For the first case,   the optimum $\theta^*$ is a $s$-sparse solution, and the regularizer is the $\ell_1$ norm, and we have
\begin{align}
\label{eq:sparsebatch}
\widehat{\theta}= &{\arg \min}\,\lbrace \widehat{f}(\theta)+\lambda_n\Vert {\theta}\Vert_1 \rbrace
 \end{align}

We also consider the matrix decomposition problem, where the underlying matrix $M^*=S^*+L^*$ is a combination of a sparse matrix $S^*$ and a low rank matrix $L^*$. Here the unknown parameters are $[S^*;L^*]$, and the regularization $\mathcal{R}(\cdot)$ is a combination of the $\ell_1$ norm, and the nuclear norm $\|\cdot\|_*$ on the sparse and low rank parts respectively. The corresponding batch estimate is given by
\begin{align} \label{eq:latent}
&\widehat{M}:={\arg \min}{\lbrace {f}(M)+\lambda_n\Vert {S}\Vert_1+\mu_n \Vert L \Vert_* \rbrace} \\ \nonumber
&\st \quad M=S+L,\quad \Vert L \Vert_\infty \leq \frac{\alpha}{p}.
\end{align}
The $\|\cdot\|_\infty$ constraint on the low rank matrix will be discussed in detail later, and it is assumed that the true matrix $L^*$ satisfies this condition.

We consider an online version of the optimization problem where we optimize the program in \eqref{eqn:batch} under each data sample instead of using the empirical estimate of $f$ for an entire batch. We consider an inexact version of the online ADMM method, where  we compute the gradient $\hat{g}_i\in \nabla f(\theta, x_i)$ at each step and employ it for optimization. In addition, we consider an epoch based setting, where we constrain the optimal solution to be close to the initial estimate at the beginning of the epoch. This can be viewed as a form of regularization and we constrain more (i.e. constrain the solution to be closer) as time goes by, since we expect to have a sharper estimate of the optimal solution. This limits the search space for the optimal solution and allows us to provide tight guarantees in the high-dimensional regime.

We first consider the simple case of sparse setting in \eqref{eq:sparsebatch}, where the ADMM has double  blocks,and then extend it to the sparse+low rank setting of \eqref{eq:latent}, which involves multi-block ADMM.

% Please note that we do not need Assumption B1 for contraction and convergence analysis of our work. Nevertheless since~\citep{AgarwalNW12} relies on this assumption, we define it here for future comparison.\\

%%%%%%%%%%%%%%%%%%%%%%%%%%%%%%%

%
\section{$\ell_1$ Regularized Stochastic Optimization} \label{sec:sparse}
We consider the optimization problem $\theta^* \in {\arg\min}~ \mathbb{E}[f(\theta,x)]$, ${\theta \in \Omega}$ where $\theta^*$ is a sparse vector. The loss function $f(\theta, x_k)$ is a function of a parameter $\theta \in \mathbb{R}^d$ and samples $x_i$. In stochastic setting, we do not have access to $ \mathbb{E}[f(\theta,x)]$ nor to its subgradients. In each iteration we have access to one noisy sample. In order to impose sparsity we use regularization. Thus we solve a sequence \begin{align} \label{eq:sparse}
{\theta}_k \in \underset{\theta \in \Omega'}{\arg\min}~ f(\theta,x_k)+\lambda \Vert \theta \Vert_1 ,\quad  \Omega' \subset \Omega,
\end{align} 
where the regularization parameter $\lambda >0$ and the constraint sets $\Omega'$ change from epoch to epoch.
\subsection{Epoch-based Online ADMM Algorithm}\label{sec:sparse-algo}
\begin{algorithm}[t]
\caption{Regularized Epoch-based Admm for Stochastic Optimization in high-dimensioN 1 (REASON 1)}
\label{our_method}
\begin{algorithmic}
\STATE \textbf{Input} $\rho,\rho_x > 0$, epoch length $T_0$
%schedule $\lbrace{T_i}\rbrace_{i=1}^{k_T}$
, initial prox center $\tilde{\theta}_1$, initial radius $R_1$, regularization parameter $\lbrace{\lambda_i}\rbrace_{i=1}^{k_T}$.\\
\STATE \textbf{Define} $Shrink_\kappa(\cdot)$ shrinkage operator in \eqref{eq:shrink}
\FOR {Each epoch $i=1,2,...,k_T$}
\STATE Initialize $\theta_0=y_0=\tilde{\theta}_i$\\
\FOR {Each iteration $k=0,1,...,T_0-1$} %T_i-1
\STATE
\begin{align}
%\sublabon{equation}
\label{eq:theta_update}
&\theta_{k+1} =\underset{\Vert \theta-\tilde{\theta}_i\Vert_1 \leq R_i}{\arg\min}\lbrace
\langle \nabla {f}(\theta_k),\theta -\theta_k \rangle-\langle z_k,\theta-y_k\rangle+\frac{\rho}{2} \Vert \theta-y_k \Vert_2^2 +\frac{\rho_x}{2} \Vert \theta-\theta_k \Vert_2^2 \rbrace\\ \nonumber
&y_{k+1}=\text{Shrink}_{\lambda_i/\rho}(\theta_{k+1}-\frac{z_k}{\rho})\\ \nonumber
&z_{k+1}=z_k -\tau(\theta_{k+1}-y_{k+1})
\end{align}
%\sublaboff{equation}
\ENDFOR
\STATE \textbf{Return} : $\overline{\theta}(T_i):=\frac{1}{T}\sum_{k=0}^{T_0-1} \theta_k$ for epoch $i$ and $\tilde{\theta}_{i+1}=\overline{\theta}(T_i)$. %and $\overline{Y}(T_i):=\frac{1}{T}\sum_{k=1}^T Y_k$
\STATE \textbf{Update} : $R_{i+1}^2=R_i^2/2$.
\ENDFOR
\end{algorithmic}
\end{algorithm}
We now describe the modified inexact ADMM algorithm for the sparse optimization problem in \eqref{eq:sparse}, and refer to it as REASON 1, see Algorithm~\ref{our_method}. We consider epochs of length $T_0$, and in each epoch $i$, we constrain the optimal solution to be within an $\ell_1$ ball with radius $R_i$ centered around $\tilde{\theta}_i$, which is the initial estimate of $\theta^*$ at the start of the epoch.
The $\theta$-update is given by
\begin{align*}
\theta_{k+1} =\underset{\Vert \theta-\tilde{\theta}_i\Vert_1^2 \leq R_i^2}{\arg\min}\lbrace
\langle \nabla {f}(\theta_k),\theta -\theta_k \rangle-\langle z_k,\theta-y_k\rangle+\frac{\rho}{2} \Vert \theta-y_k \Vert_2^2 +\frac{\rho_x}{2} \Vert \theta-\theta_k \Vert_2^2 \rbrace
\end{align*} Note that this is an inexact update since we employ the gradient $\nabla f(\cdot)$ rather than optimize directly on the loss function $f(\cdot)$ which is expensive. The above program can be solved efficiently since it is a projection on to the $\ell_1$ ball, whose complexity is linear in the sparsity level of the gradient, when performed serially, and $\mathcal{O}(\log d)$ when performed in parallel using $d$ processors~\citep{duchi2008efficient}. For details of $\theta$-update implementation see Appendix~\ref{sec:implement1}.

For the regularizer, we introduce the variable $y$, and the $y$-update is
\begin{align*}
y_{k+1}=
{\arg\min} \lbrace \lambda_i \Vert y_k \Vert_1-\langle z_k,\theta_{k+1}-y\rangle+\frac{\rho}{2} \Vert \theta_{k+1}-y \Vert_2^2  \rbrace
\end{align*}
This update can be simplified to the form given in  REASON 1, where $\text{Shrink}_\kappa(\cdot)$ is the soft-thresholding or shrinkage function \citep{boyd2011distributed}.
\begin{align}
\label{eq:shrink}
\text{Shrink}_\kappa(a)=(a-\kappa)_+ -(-a-\kappa)_+
\end{align}Thus, each step in the update is extremely simple to implement.
\noindent When an epoch is complete, we carry over  the average $\overline{\theta}(T_i)$ as the next epoch center and reset the other variables.

\subsection{High-dimensional Guarantees}
We now provide convergence guarantees for the proposed method under the following assumptions.

\paragraph{Assumption A1: Local strong convexity (LSC)}: The function $f : S \rightarrow \mathbb{R}$ satisfies an $R$-local form
of strong convexity (LSC) if there is a non-negative constant $\gamma = \gamma(R)$ such that
\begin{align*}
%\label{eq:localsc}
f({\theta}_1) \geq f(\theta_2) + \langle \nabla f(\theta_2), {\theta}_1-\theta_2\rangle +\frac{\gamma}{2}\Vert \theta_2 - {\theta}_1  \Vert_2^2.
\end{align*}
for any $\theta_1, {\theta} _2\in S$ with $\Vert {\theta}_1\Vert_1 \leq R$ and $\Vert {\theta}_2 \Vert_1  \leq R$.

%{comment on how this is easy to satisfy.. no your remark was about strong convexity.. i commented it.. i need a remark on why local strong convexity is easy to satisfy in high dimensions but not strong convexity}\\
%Even the linear regression loss function has rank at most $n$.
%
Note that the notion of strong convexity leads to faster convergence rates in general. Intuitively, strong convexity is a measure of curvature of the loss function, which relates the reduction in the loss function to closeness in the variable domain. Assuming that the function $f$ is twice continuously differentiable, it is strongly convex, if and only if its Hessian is positive semi-definite, for all feasible $\theta$. However, in the high-dimensional regime, where there are fewer samples than data dimension,  the Hessian matrix is often singular and we do not have global strong convexity. A solution is to impose local strong convexity which allows us to provide guarantees for high dimensional problems. The notion of local strong convexity has been exploited before in a number of works on high dimensional analysis, e.g.,~\citep{negahban2012unified,agarwal2012noisy,AgarwalNW12}.

%\textbf{Remark}: Assuming that function $f$ is twice continuously differentiable in the local area, it is strongly convex if and only if its Hessian is positive semi-definite for all possible $\theta$. Nevertheless, It is not necessary for a function to be differentiable in order to be strongly convex. The locally defined notion simplifies things as the Hessian might only be positive semi-definite for a restricted set in space, e.g. $\log \det(\cdot)$.\\

\paragraph{Assumption A2: Sub-Gaussian stochastic gradients:}
Let $e_k(\theta):=\nabla f(\theta , x_k) - \Ebb[\nabla f(\theta, x_k)]$. For all $\theta$ such that $\Vert \theta-\theta^* \Vert_1 \leq R$, there is a constant $\sigma=\sigma(R)$ such that for all $k>0$,
\begin{align*}
\mathbb{E}[\exp(\Vert e_k(\theta)\Vert_\infty^2)/\sigma^2] \leq \exp(1)
\end{align*}
\paragraph{Remark: }The bound holds with $\sigma=\mathcal{O}(\sqrt{\log d}) $ whenever each component of the error vector has sub-Gaussian tails~\citep{AgarwalNW12}.

%\textbf{Assumption A3: Bounded dual variable}: Let $z$ denote the dual variable in ADMM. $\exists B,$ s.t  $\forall\, z,$ $\Vert z \Vert_1 \leq B$~\citep{boyd2011distributed}. \textcolor{red}{Considering the dual feasibility condition, this assumption is equivalent to LSC~\citep{AgarwalNW12}.}\\
%{add this as another assumption and give Boyd reference $\Vert z^*\Vert_1 \leq B$, define $z^*$ to be the dual optimal}
\paragraph{Assumption A3: Local Lipschitz condition:} For each $R > 0 $, there is a constant $G = G(R)$ such
that
\begin{equation}
\label{eq:LocalLip}
|f(\theta_1) - f({\theta_2})| \leq G\Vert \theta_1 -\theta_2 \Vert_1
\end{equation}
for all $\theta_1 , \theta_2  \in S$
 such that  $\Vert \theta -{\theta}^* \Vert_1 \leq R$ and $\Vert \theta_1 -{\theta}^* \Vert_1 \leq R$.

{We choose the algorithm parameters as below where $\lambda_i$ is the regularization for $\ell_1$ term, $\rho$ and $\rho_x$ are penalties in $\theta$-update as in~\eqref{eq:theta_update} and $\tau$ is the step size for the dual update.}
\begin{align}
\label{eq:T0}
%&T_i=C \frac{s^2}{\gamma^2}\left[ \frac{\log d+\gamma/sG+12\sigma_i^2\log(3/\delta_i)}{R_i^2} \right],\\ \nonumber
%&T_i=C \left[\frac{s^2}{{\gamma^2}}\left[ \frac{\log d+\gamma/sG+12\sigma_i^2\log(3/\delta_i)}{R_i^2} \right]+\frac{s}{\gamma}\rho_x+(\frac{s}{\gamma}\frac{G^2}{R_i})^{2/3}\right],\\ \nonumber
&\lambda_i^2=\frac{{\gamma}}{s\sqrt{T_0}}\sqrt{R_i^2 \log d+\frac{ G^2 R_i^2}{T_0}+\sigma_i^2R_i^2w_i^2}\\ \nonumber
%&\lambda_i^2=\frac{{\gamma}}{s\sqrt{T_i}}\sqrt{R_i^2 \log d+\frac{G^4R_i^2}{T_i^2}+\frac{\rho_0^2 G^2 R_i^2}{T_i^2}+\frac{\rho_0^4 R_i^6}{T_i^2}+\frac{\rho_x^2R_i^4}{T_i}+12\sigma_i^2R_i^2w_i^2}\\ \nonumber
%&\rho=c_1\sqrt{T_i},\quad \rho_x > 0, \quad \tau=c_2\sqrt{T_i},\\ \nonumber
%& c_1=\frac{\sqrt{\log d}}{R_i}, \quad c_2=1/R_i. \nonumber
&\rho \propto \frac{\sqrt{T_0 \log d}}{R_i},\quad \rho_x > 0, \quad \tau=\rho.
\end{align}

%%%%%%%%%%%%%%%%%%%%%%%

\begin{theorem}
\label{thm:sparse_const_epoch}
Under Assumptions $A1-A3$, $\lambda _i$ as in~\eqref{eq:T0} 
%and $k_T$ as defined in\eqref{eq:kt_sparse}
, we use fixed epoch length $T_0=T \log d/{k_T}$ where $T$ is the total number of iterations. Assuming this setting ensures $T_0 =\mathcal{O} (\log d)$, for any $\theta^*$ with sparsity $s$, we have
\begin{align*}
\Vert \bar{\theta}_T-\theta^*\Vert_2^2=\mathcal{O}\left(s~\frac{\log d+(w^2+\log(k_T/{\log d}))\sigma^2}{T}~\frac{\log d}{k_T} \right),
\end{align*}
with probability at least $1-3\exp(w^2/12)$, where $\bar{\theta}_T$ is the average for the last epoch for a total of $T$ iterations and
\begin{align*}
%\label{eq:kt_sparse}
k_T=\log_2 \frac{\gamma^2R_1^2 T}{s^2( \log d+12\sigma^2w^2)}.
\end{align*}
\end{theorem}
For proof, see Appendix~\ref{sec:const_sparse_proof}.

{\bf Improvement of $\log d$ factor :} The above theorem covers the practical case where the epoch length $T_0$ is fixed. We can improve the above results using varying epoch lengths (which depend on the problem parameters) such that $\Vert \bar{\theta}_T-\theta^*\Vert_2^2=\mathcal{O}(s\log d/T)$. See Theorem~\ref{thm:sparse_optimal} in Appendix~\ref{appendix:thmsparse}.
%This convergence rate of $\mathcal{O}(s\log d/ T)$ matches the minimax lower bounds for sparse estimation~\citep{raskutti2011minimax}. This implies that our guarantees are {\em unimprovable} up to constant factors.

\paragraph{Optimal Guarantees:} The above results indicate a convergence rate of $\mathcal{O}(s\log d/ T)$ which matches the minimax lower bounds for sparse estimation~\citep{raskutti2011minimax}. This implies that our guarantees are {\em unimprovable} up to constant factors.

\paragraph{Comparison with~\citet{AgarwalNW12}:}The RADAR algorithm proposed by~\citet{AgarwalNW12} also achieves a rate of $\mathcal{O}(s\log d/T)$ which matches with ours. The difference is our method is capable of solving problems with multiple variables and constraints, as discussed in the next section, while RADAR cannot be generalized to do so.

\paragraph{Remark on Lipschitz property: }In fact, our method requires a weaker condition than local Lipschitz property. We only require the following bounds on the dual variable: $\Vert z_{k+1}-z_k\Vert_1$ and $\Vert z_k\Vert_\infty$. Both these are upper bounded by $G+2(\rho_x+\rho)R_i$. In addition the $\ell_1$ constraint does not influence the bound on the dual variable.
%, where $G$ is the local Lipschitz constant, and $\rho_0>0$ is a constant.
For details see Section~\ref{sec:epochopt1}.

\paragraph{Remark on  need for $\ell_1$ constraint: }We use $\ell_1$ constraint in the $\theta$-update step, while the usual ADMM method does not have such a constraint. The $\ell_1$ constraint allows us to provide efficient high dimensional scaling (as $\order(\log d)$). Specifically, this is because one of the terms in our convergence rate consists of   $\langle e_k,\theta_k-\hat{\theta}_i \rangle$, where $e_k$ is the error in the gradient (see Appendix~\ref{sec:ek}). We can use the inequality \begin{align*}
\langle e_k,\theta_k-\hat{\theta}_i \rangle \leq \Vert e_k \Vert_\infty \Vert \theta_k-\hat{\theta}_i\Vert_1.
\end{align*}  From Assumption A$2$, we have a bound on $\|e_k\|_\infty=\order(\log d)$, and by imposing the $\ell_1$ constraint, we also have a bound on the second term, and thus, we have an efficient convergence rate. If instead  $\ell_p$ penalty is imposed for some $p$, the error scales as $\|e(\theta)\|_q^2$, where $\ell_q$ is the dual norm of $\ell_p$. For instance, if $p=2$, we have $q=2$, and the error can be as high as $\mathcal{O}(d/T)$ since $\|e(\theta)\|_2^2\leq d \sigma$. Note that for the $\ell_1$ norm, we have $\ell_\infty$ as the dual norm, and $\|e(\theta)\|_\infty\leq \sigma= \mathcal{O}(\sqrt{\log d})$ which leads to optimal convergence rate in the above theorem. Moreover, this $\ell_1$ constraint can be efficiently implemented, as discussed in Section~\ref{sec:sparse-algo}.

\section{Extension to Doubly Regularized Stochastic Optimization}\label{sec:latent}
We now consider the problem  of  matrix decomposition   into a sparse matrix $S\in \Rbb^{p\times p}$ and a low rank matrix $L\in \Rbb^{p\times p}$ based on the loss function $f$ on $M=S+L$. The batch program is given in Equation~\eqref{eq:latent}
%\begin{align}
%\label{eq:latentSbatch}
%&\widehat{M}:={\arg \min}{\lbrace {f}(M)+\lambda_n\Vert {S}\Vert_1+\mu_n \Vert L_* \Vert\rbrace} \\ \nonumber
%&\st \quad M=S+L,\quad \Vert L \Vert_\infty \leq \frac{\alpha}{p}.
%\end{align} 
and we now design an online program  based on  multi-block ADMM algorithm, where the updates for $M, S, L$ are carried out independently.

In the stochastic setting, we consider the optimization problem $M^* \in {\arg\min}~ \mathbb{E}[f(M,X)]$, where we want to decompose $M$ into a sparse matrix $S\in \Rbb^{p\times p}$ and a low rank matrix $L\in \Rbb^{p\times p}$. $f(M,X_k)$  is a function of parameter $M$ and samples $X_k$. $X_k$ can be a matrix (e.g. independent noise model) or a vector (e.g. Gaussian graphical model). In stochastic setting, we do not have access to $ \mathbb{E}[f(M,X)]$ nor to its subgradients. In each iteration we have access to one noisy sample and update our estimate based on that. We impose the desired properties with regularization.
Thus, we solve a sequence 
\begin{align}
\label{eq:latentS}
{M}_k:={\arg \min}{\lbrace \widehat{f}(M,X_k)+\lambda\Vert {S}\Vert_1+\mu \Vert L\Vert_* \rbrace} \quad \quad
\st \quad M=S+L,\quad \Vert L \Vert_\infty \leq \frac{\alpha}{p}.
\end{align}

%
%Equation \eqref{eq:latent} is suitable for ADMM-based analysis as it consists of three convex functions involved in the objective functions $f(M)$, $h(S)=\lambda_n\Vert {S}\Vert_1$ and $g(L)=\mu_n \Vert L \Vert_*$. This is called multi-block ADMM. In order to solve this problem we merge $L$ and $S$ into the same variable $W$, i.e. $W=[S;L]$ as in \citep{ADMM_latent}. Next, we impose our epoch-based approach and required constraints to meet sampling error criteria. Since $S$ and $L$ are the desired parameters we impose constraints on them and carry them over from epoch to epoch. On the other hand, at the end of each epoch $M$ is reset to meet the criteria. The only challenge exists in the update of $W$ where the Bregman divergence results in coupling between $S$ and $L$. Nevertheless, solving the $W$-update rule inexactly by one step of a proximal gradient method as in \citep{ADMM_latent} decouples the two part. Details are provided in Proof outline in Supplementary material.
\subsection{Epoch-based Multi-Block ADMM Algorithm}
We now extend the ADMM method proposed in REASON~\ref{our_method} to multi-block ADMM. The details are in Algorithm~\ref{our_method_latent}, and we refer to it as REASON 2. Recall that the matrix decomposition setting assumes that the true matrix $M^*=S^*+L^*$ is a combination of a sparse matrix $S^*$ and a low rank matrix $L^*$. In REASON 2, the updates for matrices $M, S, L$ are done independently at each step.

For the $M$-update,  the same linearization approach as in REASON 1 is used
%{i removed notation $W$ from this section}
\begin{align*}
M_{k+1} ={\arg\min}\!\lbrace\!
\Tr(\nabla f(M_k),\!M -M_k)\!
-\Tr\! (Z_k,\!M-S_k-L_k)\!+\!\frac{\rho}{2}\! \Vert M-S_k-L_k \Vert_\mathbb{F}^2\!+\!\frac{\rho_x}{2} \Vert M-M_k \Vert_\mathbb{F}^2\! \rbrace.
%\\
%&=\frac{-\nabla f(M_k)+Z_k+\rho AW_k+\rho_x M_k }{\rho+\rho_x},
\end{align*}  This is an unconstrained quadratic optimization with closed-form updates, as shown in REASON 2.
%\begin{align*}
%M_{k+1} =\frac{-\nabla {f}(M_k)+Z_k+\rho AW_k+\rho_x M_k \rbrace}{\rho+\rho_x}.
%\end{align*}
%t{i need u to give the optimization form of $M_k$ and not the closed form.. u can mention closed form is given in pseudocode}
The update rules for $S$, $L$ are result of doing an inexact proximal update by considering them as a single block, which can then be decoupled as follows. For details, see Section \ref{latent_outline}.
\begin{align}
%\sublabon{equation}
\label{eq:S} %checked
\underset{\Vert S-\tilde{S}_i\Vert_1^2  \leq {R}_i^2}{\arg\min}~\lambda_i \Vert S \Vert_1+\frac{\rho}{2\tau_k}\Vert S-(S_k+\tau_k G_{M_k})\Vert_\mathbb{F}^2,\\ \label{eq:L} %checked
\underset{\underset {\Vert L\Vert_\infty \leq \alpha/p} {\Vert L-\tilde{L}_i\Vert_*^2 \leq \tilde{R}_i^2}}{\arg\min}~\lambda_i \Vert L \Vert_*+\frac{\rho}{2\tau_k}\Vert L-(L_k+\tau_k G_{M_k}) \Vert_\mathbb{F}^2,
\end{align}
%\sublaboff{equation}
where $G_{M_k}=M_{k+1}-S_k-L_k-\frac{1}{\rho}Z_k$.

As before, we consider epochs of length $T_0$ and project the estimates $S$ and $L$ around the epoch initializations $\tilde{S}_i$ and $\tilde{L}_i$. We do not need to constrain the update of matrix $M$. We impose an $\ell_1$-norm project  for the sparse estimate $S$. For the low rank estimate $L$, we impose a nuclear norm projection  around the epoch initialization $\tilde{L}_i$. Intuitively, the nuclear norm projection , which is an $\ell_1$ projection on the singular values, encourages sparsity in the spectral domain leading to low rank estimates. In addition, we impose an $\ell_\infty$ constraint of $\alpha/p$ on each entry of $L$, which is different from the update of $S$. Note that the $\ell_\infty$ constraint is also imposed for the batch version of the problem~\eqref{eq:latent} in~\citep{agarwal2012noisy}, and we assume that the true matrix $L^*$ satisfies this constraint. For more discussions, see Section~\ref{sec:HDguarantees}.

Note that each step of the method is easily  implementable. The $M$-update is in closed form.
The $S$-update involves optimization with projection on to the given $\ell_1$ ball which can be performed efficiently~\citep{duchi2008efficient}, as discussed in Section~\ref{sec:sparse-algo}. For implementation details see Appendix~\ref{sec:implement2}. 

For the $L$-update, we introduce an additional auxiliary variable $Y$ and we have
%projection on to intersection of $\ell_*$ and $\ell_\infty$ balls with different centers. This can be done by modifying~\citep{su2012efficient} to perform a SVD step before projection on to $\ell_1$ ball and re-centering.
% an ADMM update with $X_0=Y_0=L_k, U_0=0$.
\begin{align*}
&L_{k+1}=\underset  {\Vert L-\tilde{L}_i\Vert_*^2 \leq \tilde{R}_i^2}{\min}~~\lambda_i \Vert L \Vert_*-\Tr(U_k,L-Y_k)+\frac{\rho}{2}\Vert L-Y_k\Vert_\mathbb{F}^2, \\
&Y_{k+1}=\underset{\Vert Y\Vert_\infty \leq \alpha/p}{\min}{~~~ \frac{\rho}{2\tau_k}\Vert L-(L_k+\tau_k G_{M_k}) \Vert_\mathbb{F}^2+\frac{\rho}{2}\Vert L_{k+1}-Y\Vert_\mathbb{F}^2-\Tr(U_k,L_{k+1}-Y)},\\
&U_{k+1}=U_k -{\tau}(L_{k+1}-Y_{k+1}).
\end{align*}%

The $L$-update can now  be performed efficiently  by computing a SVD, and then running the projection step~\citep{duchi2008efficient}. Note that approximate SVD computation techniques can be employed for efficiency here, e.g.,~\citep{lerman2012robust}. The $Y$-update
is projection on to the infinity norm ball which can be found easily.
%Let $Y_{(j)}$ denote the $j$-th entry of $Y$. We have
%\begin{align*}
%&\If \quad\vert (L_{k+1}-U_k/\rho)_{(j)}\vert \leq \frac{\alpha}{p},\quad \text{then}\quad Y_{(j)}=(L_{k+1}-U_k/\rho)_{(j)} \\
%&\Else \quad Y_{(j)}=\mbox{sign}(L_{k+1}-U_k/\rho)_{(j)} - \frac{\alpha}{p})\frac{\alpha}{p}
%\end{align*}
%=======
%The $L$-update can be performed by a SVD step before running~\citep{duchi2008efficient}.  The $Y$-update
%is essentially the definition of projection into max norm.
Let $Y_{(j)}$ stand for $j$-th entry of $\text{vector}(Y)$. The for any $j$-th entry of $\text{vector}(Y)$, solution will be as follows
\begin{align*} 
&\If ~~\vert (L_{k+1}+\frac{\tau_k}{\tau_k+1}(G_{M_k}-U_k/\rho))_{(j)}\vert \leq \frac{\alpha}{p},\\
&\quad \text{then}~~ Y_{(j)}=(L_{k+1}+\frac{\tau_k}{\tau_k+1}(G_{M_k}-U_k/\rho))_{(j)} .\\
&\Else \quad Y_{(j)}=\sign\left((L_{k+1}+\frac{\tau_k}{\tau_k+1}(G_{M_k}-U_k/\rho))_{(j)} - \frac{\alpha}{p}\right)\frac{\alpha}{p}.
\end{align*}
As before, the epoch averages are computed and used as initializations for the next epoch.
\begin{algorithm}[t]
\caption{Regularized Epoch-based Admm for Stochastic Optimization in high-dimensioN 2 (REASON 2)}
\label{our_method_latent}
\begin{algorithmic}
\STATE \textbf{Input} $\rho, \rho_x > 0$, epoch length $T_0$
% schedule $\lbrace{T_i}\rbrace_{i=1}^{K_T}$
, regularizers $\lbrace{\lambda_i,\mu_i}\rbrace_{i=1}^{k_T}$, initial prox center $\tilde{S}_1,\tilde{L}_1$, initial radii $R_1, \tilde{R}_1$.
\STATE \textbf{Define}  $Shrink_\kappa(a)$ shrinkage operator in \eqref{eq:shrink}, $G_{M_k}=M_{k+1}-S_k-L_k-\frac{1}{\rho}Z_k$.
%$\Pi_1(\cdot;\tilde{S}_i, R_i)$ projection into $\ell_1$ ball in \eqref{eq:proj_1},
%, regularization parameter $\lambda_i$.\\
\FOR {Each epoch $i=1,2,...,k_T$}
\STATE Initialize $S_0=\tilde{S}_i,L_0=\tilde{L}_i, M_0=S_0+L_0$\\
%{set this right!}
\FOR {Each iteration $k=0,1,...,T_0-1$}
\STATE \begin{align*}
%\sublabon{equation}
%\label{eq:theta_update_latent}
&M_{k+1} =\frac{ -\nabla {f}(M_k)+Z_k+\rho (S_k+L_k)+\rho_x M_k }{\rho+\rho_x}\\
%\label{eq:supdate}
&S_{k+1}=\underset{\Vert S-\tilde{S}_i\Vert_1 \leq {R}_i}{\min}\lambda_i \Vert S \Vert_1+\frac{\rho}{2\tau_k}\Vert S-(S_k+\tau_k G_{M_k})\Vert_\mathbb{F}^2\\ \nonumber
&L_{k+1}=\underset{\Vert L-\tilde{L}_i\Vert_*\leq \tilde{R}_i}{\min}\mu_i \Vert L \Vert_*+\frac{\rho}{2}\Vert L-Y_k-U_k/\rho\Vert_\mathbb{F}^2 \\ \nonumber %-\Tr(U_k,L-Y_k)+\frac{\rho}{2}\Vert L-Y_k\Vert_\mathbb{F}^2
&Y_{k+1}=\underset{\Vert Y\Vert_\infty \leq \alpha/p}{\min}{\frac{\rho}{2\tau_k}\Vert Y-(L_k+\tau_k G_{M_k}) \Vert_\mathbb{F}^2+\frac{\rho}{2}\Vert L_{k+1}-Y-U_k/\rho\Vert_\mathbb{F}^2}\\ \nonumber %-\Tr(U_k,L_{k+1}-Y)
&Z_{k+1}=Z_k -\tau(M_{k+1}-(S_{k+1}+L_{k+1}))\\ \nonumber
&U_{k+1}=U_k -{\tau}(L_{k+1}-Y_{k+1}).
\end{align*}
%\sublaboff{equation}
\ENDFOR
\STATE \textbf{Set}: $\tilde{S}_{i+1}=\frac{1}{T_0}\sum_{k=0}^{T_0-1}  S_k$ and $\tilde{L}_{i+1}:=\frac{1}{T_0}\sum_{k=0}^{T_0-1}  L_k$
%\If {$R_i^2 > 6$} \STATE{Update $R_{i+1}^2=\R_i^2/2$}
% {STOP} \ENDIf
\IF {$R_i^2 > 2(s+r+\frac{(s+r)^2}{p\gamma^2})\frac{\alpha^2}{p}$} \STATE { Update $R_{i+1}^2={R_i^2}/2,\tilde{R}_{i+1}^2={\tilde{R_i}^2}/2$} \ELSE {\STATE{STOP}}\ENDIF
\ENDFOR
\end{algorithmic}
\end{algorithm}
% Note that this implementation leads to $\Vert L \Vert_\infty =\Vert X \Vert_\infty +\hat{\alpha}$ where $\hat{\alpha}$ is the bound in max norm of dual variable. i.e., $\Vert U \Vert_\infty \leq \hat{\alpha} \quad \forall U$.
%\{in pseudocode u are using notation $\hat{f}$.. change it ..only for batch average use $\hat{f}$..everywhere else it is $f$.. use explicit notation $f(\theta, x)$.. what is the notation u are using for samples? actually it matters in the modeling section.. i used $M$ to be the sample which is bad.. i will use $x$ and $X$ for samples depending on if it is vector or matrix same as in previous section.. everywhere make sure to add $f(.. , x_k)$ or $f(.., X_k)$ in the matrix case, clear? }\\

\subsection{High-dimensional Guarantees} \label{sec:HDguarantees}
We now provide guarantees that REASON 2 efficiently recovers both the sparse and the low rank estimates in high dimensions efficiently. We need the following assumptions, in addition to Assumptions A$1$ and A$2$ from the previous section.
\paragraph{Assumption A4: Spectral Bound on the Gradient Error } Let $E_k(M, X_k):=\nabla f(M , X_k) - \Ebb[\nabla f(M, X_k)]$, $\Vert E_k \Vert_2 \leq \beta(p)\sigma$, where $\sigma:= \|E_k\|_\infty$.

Recall from Assumption A2 that $\sigma=\order(\log p)$, under sub-Gaussianity. Here, we require spectral bounds in addition to $\|\cdot\|_\infty$ bound in A2.

\paragraph{Assumption A5: Bound on spikiness of low-rank matrix} $\Vert L^* \Vert_\infty \leq \frac{\alpha}{p}$.

Intuitively, the $\ell_\infty$ constraint controls the ``spikiness'' of $L^*$. If $\alpha\approx 1$, then the entries of $L$ are $\mathcal{O}(1/p)$, i.e. they are ``diffuse'' or  ``non-spiky'', and no entry is too large.  When the low rank matrix $L^*$ has diffuse entries, it cannot be a sparse matrix, and thus, can be separated from the sparse $S^*$ efficiently. In fact, the $\ell_\infty$ constraint is a weaker form of the {\em incoherence}-type assumptions  needed to guarantee identifiability~\citep{chandrasekaran2011rank} for sparse+low rank decomposition.

\paragraph{Assumption A6: Local strong convexity (LSC)} The function $f : \mathbb{R}^{d_1 \times d_2} \rightarrow  \mathbb{R}^{n_1 \times n_2} $ satisfies an $R$-local form
of strong convexity (LSC) if there is a non-negative constant $\gamma = \gamma(R)$ such that
%\begin{align*}x
$f(B_1) \geq f(B_2) + \Tr(\nabla f(B_2)( B_1-B_2)) +\frac{\gamma}{2}\Vert B_2 - B_1  \Vert_\mathbb{F},$
%\end{align*}
for any $\Vert B_1\Vert \leq R$ and $\Vert B_2 \Vert  \leq R$, which is essentially the matrix version of Assumption A$1$. {Note that we only require LSC condition on $S+L$ and not jointly on $S$ and $L$.}
%\paragraph{Assumption A6: Bounded dual variable} Let $Z, U$ denote the dual variables in REASON 2. $\exists B_1,B_2,\st\, \forall z, \Vert Z \Vert_1 \leq B_1, \Vert U \Vert_\infty \leq B_2$~\citep{boyd2011distributed}.\\
%\\

We choose algorithm parameters as below where $\lambda_i,\mu_i$ are the regularization for $\ell_1$ and nuclear norm respectively, $\rho, \rho_x$ correspond to penalty terms in $M$-update and $\tau$ is dual update step size.
\begin{align}
\label{parameter0}
\lambda_i^2&=\frac{\gamma  \sqrt{R_i^2 +\tilde{R}_i^2}}{(s+r)\sqrt{T_0}}  \sqrt{{\log p}+\frac{G^2 }{T_0}+{\beta^2(p)}\sigma_i^2w_i^2}+\frac{\rho_x^2(R_i^2 +\tilde{R}_i^2) }{T_0}
+\frac{\alpha^2}{p^2}
+\frac{\beta^2(p)\sigma^2}{T_0}\left(\log p+w_i^2\right),\\ \nonumber
\mu_i^2&=c_\mu\lambda_i^2, \quad
\rho  \propto \sqrt{\frac{T_0 \log p }{R_i^2+\tilde{R}_i^2}},\quad \rho_x >0, \quad \tau=\rho.
\end{align}

\begin{theorem} \label{thm:latent_optimal_fixedepoch}
Under assumptions $A2-A6$, parameter settings~\eqref{parameter0} %and $k_T$ as in~\eqref{eq:kt_matrix}
, let $T$ denote total number of iterations and $T_0=T \log p/k_T$. Assuming that above setting guarantees $T_0=\mathcal{O}(\log p)$, 
\begin{align}
\label{eq:sparse_optimal2}
&\Vert \bar{S}(T)-S^* \Vert_\mathbb{F}^2 + \Vert \bar{L}(T)-L^* \Vert_\mathbb{F}^2 =\\ \nonumber&
\mathcal{O}\left((s+r)\frac{\log p+\beta^2(p)\sigma^2\left(w^2+\log (k_T/\log p )\right)}{T}\frac{\log p}{k_T}\right)+\left( 1+\frac{s+r}{\gamma^2 p}\right)\frac{\alpha^2}{p},
\end{align}
with probability at least $1-6\exp(-w^2/12)$,
\begin{align*}
%\label{eq:kt_matrix}
k_T\simeq-\log \left( \frac{(s+r)^2}{\gamma^2 R_1^2 T}\left[\log p+\beta^2(p)\sigma^2w^2\right]\right).
\end{align*}
%k_T\simeq&-\log \frac{(s+r+(s+r)/\gamma)^2}{R_1^2 T}\\&-\log \left[\log p+\frac{G}{s+r+(s+r)/\gamma}+12\beta^2(p)\sigma^2\left[(1+G)(\log(6/\delta)+2\log k_T)+\log p\right]\right] 

\end{theorem}
For proof, see Appendix~\ref{sec:const_matrix_proof}

{\bf Improvement of $\log p$ factor :} The above  result can be improved by a $\log p$ factor by considering varying epoch lengths (which depend on the problem parameters). The resulting convergence rate is $\mathcal{O}((s+r) p \log p/T +\alpha^2/p)$.  See Theorem~\ref{thm:latent_optimal} in Appendix~\ref{sec:guaranteesR2}.

\paragraph{Scaling of $\beta(p)$: }We have the following bounds $\Theta(\sqrt{p})\leq\beta(p)\Theta(p) $. This  implies that the convergence rate is $\mathcal{O}((s+r) p \log p/T +\alpha^2/p)$, when $\beta(p)=\Theta(\sqrt{p})$ and when $\beta(p)=\Theta(p)$, it is $\mathcal{O}((s+r) p^2 \log p/T +\alpha^2/p)$. The upper bound on $\beta(p)$ arises trivially by converting the max-norm $\|E_k\|_\infty \leq \sigma$ to the bound on the spectral norm $\|E_k\|_2$.
In many interesting scenarios, the lower bound on $\beta(p)$ is achieved, as outlined in Section~\ref{sec:model-examples}.

%\vspace{-.4cm}
\paragraph{Comparison with the batch result: }\citet{agarwal2012noisy} consider the batch version of the same problem~\eqref{eq:latent}, and provide a convergence rate of $\mathcal{O}(s \log p  + r p )/T + s\alpha^2/p^2)$. This is also the minimax lower bound under the independent noise model.  With respect to the convergence rate, we match their results with respect to the scaling of $s$ and $r$, and also obtain a $1/T$ rate. We match the scaling with respect to $p$ (up to a $\log$ factor), when $\beta(p)=\Theta(\sqrt{p})$ attains the lower bound, and we discuss a few such instances below. Otherwise, we are worse by a factor of $p$ compared to the batch version. Intuitively, this is because we require different bounds on error terms $E_k$ in the online and the batch settings. For  online analysis, we need to bound $\sum_{k=1}^{T_i} \Vert E_k\Vert_2/T_i$ over each epoch, while for the batch analysis, we need to bound $\Vert \sum_{k=1}^{T_i}  E_k\Vert_2/T_i$, which is smaller.
%It involves the spectral norm since it is the dual norm of the nuclear norm, used as the constraint for the low rank update.
Intuitively, the difference for the two settings can be explained as follows: for the batch setting, since we consider an empirical estimate, we operate on the averaged error, while we are manipulating each sample in the online setting and suffer from the error due to that sample.
We can employ efficient concentration bounds for the batch case~\citep{tropp2012user}, while for the online case, no such bounds exist in general. From these observations, we conjecture that   our bounds in Theorem~\ref{thm:latent_optimal} are {\em unimproveable} in the online setting.
% \{we should explore if we can quickly use stochastic lower bound paper to state it formally..discuss}

\paragraph{Approximation Error: }
Note that the  optimal decomposition $M^*=S^*+L^*$ is not identifiable in general without the incoherence-style conditions~\citep{chandrasekaran2011rank,hsu2011robust}.
In this paper, we provide efficient guarantees without assuming such strong incoherence constraints. This implies that  there is an {\em approximation error} which is incurred even in the noiseless setting due to model non-identifiability. \citet{agarwal2012noisy} achieve an approximation error of   $s\alpha^2/p^2$ for their batch algorithm. Our online algorithm has an approximation error of  $\max \lbrace s+r, p \rbrace\alpha^2/p^2$, which is worse, but is still decaying with $p$. It is not clear if this bound can be improved by any other online algorithm.

%, and the non-optimality arises from the fact that we are unable to decouple the error into sparse and low rank parts in some steps of the proof using techniques from the batch setting. This framework is not applicable to multi-block ADMM as we have multiple optimization problems rather than one.
\subsubsection{Optimal Guarantees for Various Statistical Models}\label{sec:model-examples}
We now list some statistical models under which we achieve the batch-optimal rate for sparse+low rank decomposition.
%\vspace{-.4cm}

%\{below I changed from $W_k$ to $N_k$ since u used $W$ to denote block of S and L}

\paragraph{1) Independent Noise Model: }Assume we sample i.i.d. matrices $X_k = S^* + L^* +N_k$, where the noise $N_k$ has independent bounded sub-Gaussian entries with $\max_{i,j}\Var(N_k(i,j))=\sigma^2$. We consider the square loss  function, i.e. $\|X_k-S-L\|_{\Fbb}^2$. In this case, $E_k = X_k - S^*-L^* = N_k$.  From~[Thm. 1.1]\citep{vu2005spectral}, we have w.h.p that $\|N_k\|=\mathcal{O}(\sigma \sqrt{p})$. We match the batch bound of~\citep{agarwal2012noisy} in this setting. Moreover,~\citet{agarwal2012noisy} provide a minimax lower bound for this model, and we match it as well. Thus, we achieve the optimal convergence rate  for online matrix decomposition under the independent noise model.

\paragraph{2) Linear Bayesian Network: }Consider a $p$-dimensional vector $y = A h + n$, where $h\in \Rbb^r$ with $r\leq p$, and $n\in \Rbb^p$. The variable $h$ is hidden, and $y$ is the observed variable.  We assume that the vectors $h$ and $n$ are each zero-mean sub-Gaussian vectors with i.i.d entries, and  are independent of one another.
Let $\sigma_h^2$ and $\sigma_n^2$ be the variances for the entries of $h$ and $n$ respectively.   Without loss of generality, we assume that the columns of $A$ are normalized, as we can always rescale $A$ and  $\sigma_h$  appropriately to obtain the same model.
Let $\Sigma^*_{y,y}$ be the true covariance matrix of $y$. From the independence assumptions, we have   $\Sigma^*_{y,y}= S^*+L^*$, where $S^*=\sigma_n^2 I$ is a diagonal matrix and $L^*= \sigma_h^2 A A^\top$ has rank at most $r$.

%\footnote{We say an event occurs with high probability if it occurs with probability greater than $1-p^{-c}$ for some constant $c>0$.}

In each step $k$, we obtain a sample $y_k$ from the Bayesian network. For the square loss function $f$,   we have the error $E_k = y_k y_k^\top - \Sigma^*_{y,y}$. Applying~[Cor. 5.50]\citep{vershynin2010introduction}, we have, with w.h.p.
\begin{align}
\| n_k n_k^\top - \sigma_n^2 I\|_2 = \order(\sqrt{p}\sigma_n^2 ), \quad \|h_k h_k^\top - \sigma_h^2 I\|_2 =\order(\sqrt{p}  \sigma_h^2). 
\label{eqn:cov-conc}
\end{align}
 We thus have with   probability $1-Te^{-cp}$,
$ \|E_k\|_2 \leq \mathcal{O}\left(\sqrt{p}(\|A\|^2 \sigma^2_h + \sigma^2_n) \right), \quad \forall\, k \leq T.$
When $\|A\|_2$ is bounded, we obtain the optimal bound in Theorem~\ref{thm:latent_optimal}, which matches  the batch bound.
%\aacomment{i removed full rank assumption on $A$. remove in long version..also for generic matrices $\alpha$ is polylog not 1.. change in long version}
If the entries of $A$ are {\em generically} drawn (e.g., from a Gaussian distribution), we have $\|A\|_2 =\mathcal{O}(1+\sqrt{r/p})$. Moreover, such generic matrices $A$ are also ``diffuse'', and thus,   the low rank matrix $L^*$ satisfies Assumption A5, with $\alpha \sim$ polylog$(p)$.  Intuitively, when $A$ is generically drawn, there are diffuse connections from hidden to observed variables, and we have efficient guarantees under this setting.

Thus, our online method   matches the batch guarantees   for linear Bayesian networks when   the entries of the observed vector $y$ are conditionally independent given the latent variable $h$. When this assumption is violated, the above framework is no longer applicable since the true covariance matrix $\Sigma^*_{y,y}$  is {\em not} composed of a sparse matrix. To handle such models, we consider  matrix decomposition of the inverse covariance or the precision matrix $M^*:= {\Sigma^*}^{-1}_{y,y}$, which can be expressed as a combination of sparse and low rank matrices, for the class of latent Gaussian graphical models, described in Section~\ref{sec:graphmodel}. Note that the result cannot be applied directly in this case as loss function is not locally Lipschitz. Nevertheless, in Section~\ref{sec:graphmodel} we show that we can take care of this problem.

\section{Proof Ideas and Discussion}
\subsection{Proof Ideas for REASON 1}
\begin{enumerate}
\item In general, it is not possible to establish error contraction for stochastic ADMM  at the end of  each step. We establish error contracting at the end of certain time epochs, and we impose different levels of regularizations over different epochs. We perform an induction on the error, i.e. if the error at the end of $k^{\tha}$ epoch is $\|\bar{\theta}(T_i)-\theta^*\|_2^2 \leq cR_i^2$, we show that in the subsequent epoch, it contracts as $\|\bar{\theta}(T_{i+1})-\theta^*\|_2^2 \leq cR_i^2/2$ under appropriate choice of $T_i$, $R_i$ and other design parameters. This is possible when we establish feasibility of the optimal solution $\theta^*$ in each epoch. Once this is established, it is straightforward to obtain the result in Theorem~\ref{thm:sparse_optimal}.

\item To show error contraction, we break down the error $\|\bar{\theta}(T_i) -\theta^*\|_2$ into two parts, \viz $\|\bar{\theta}(T_i) -\hat{\theta}(T_i) \|_2$ and $\|\hat{\theta}(T_i) -\theta^*\|_2$, where $\hat{\theta}(T_i) $ is the optimal batch estimate over the $i$-th epoch. The first term $\|\bar{\theta}(T_i) -\hat{\theta}(T_i) \|_2$ is obtained on the lines of analysis of stochastic ADMM, e.g.,~\citep{BADMM}. Nevertheless, our analysis differs from that of~\citep{BADMM}, as theirs is not a stochastic method. i.e., the sampling error is not considered. Moreover, we show that the parameter $\rho_x$ can be chosen as a constant while the earlier work~\citep{BADMM} requires a stronger constraint $\rho_x =\sqrt{T_i}$. For details, see Appendix~\ref{sec:epochopt1}. In addition, the $\ell_1$ constraint that we impose enables us to provide tight bounds for the high dimensional regime. The second term $\|\hat{\theta}(T_i) -\theta^*\|_2$ is obtained by exploiting the local strong convexity properties of the loss function, on lines of~\citep{AgarwalNW12}. There are additional complications in our setting, since we have an auxiliary variable $y$ for update of the regularization term. We relate the two variables through the dual variable, and use the fact that the dual variable is bounded. Note that this is a direct result from local Lipschitz property and it is proved in Lemma~\ref{lemma:dualsparse} in Appendix~\ref{sec:epochopt1}. In fact, in order to prove the guarantees, we need bounded duality which is a weaker assumption than local Lipschitz property. We discuss this in Section~\ref{sec:graphmodel}.

\item For fixed epoch length, the  error shrinkage stops after some epochs but the error does not increase significantly afterwards. Following lines of~\citep{AgarwalNW12}, we prove that for this case the convergence rate is worse by a factor of $\log d$.

\end{enumerate}

\subsection{Proof Ideas for REASON 2} \label{latent_outline}

We now provide a short overview of proof techniques for establishing the guarantees in Theorem~\ref{thm:latent_optimal_fixedepoch}. It builds on the proof techniques used for proving Theorem~\ref{thm:sparse_const_epoch}, but is significantly more involved since we now need to decouple the errors for sparse and low rank matrix estimation, and our ADMM method consists of multiple blocks. The main steps are as follows
\begin{enumerate}
\item It is convenient to define $W=[S;L]$ to merge the variables  $L$ and $S$ into a single variable $W$, as in~\citep{ADMM_latent}.
Let $\phi(W)=\Vert S \Vert_1+\frac{\mu_i}{\lambda_i}\Vert L \Vert_*$, and  $A=[I,I]$. The  ADMM update for $S$ and $L$ in REASON 2, can now be rewritten as a single update for variable $W$. Consider the update  \begin{align*}
%\label{eq:latent_3}
W_{k+1}
=\underset{W}{\arg\min} \lbrace \lambda_i\phi(W)+\frac{\rho}{2} \Vert M_{k+1}-AW-\frac{1}{\rho}Z_k\Vert_\mathbb{F}^2 \rbrace.
\end{align*}
The above  problem is not easy to solve as the $S$ and $L$ parts are coupled together. Instead, we solve it inexactly through one step of a proximal gradient method as in~\citep{ADMM_latent} as
\begin{align}
\label{eq:4.4}
\underset{W}{\arg\min} \lbrace \lambda_i\phi(W)+
\frac{\rho}{2\tau_k} \Vert W-[W_k+\tau_k A^\top
%\begin{array}{c} \left[ I \\ I \right]
%\end{array}
(M_{k+1}-AW_k-\frac{1}{\rho}Z_k)]\Vert_\mathbb{F}^2 \rbrace.
\end{align}
Since the two parts of $W=[S;L]$ are separable in the quadratic part now, Equation~\eqref{eq:4.4} reduces to two decoupled updates on $S$ and $L$ as given by \eqref{eq:S} and \eqref{eq:L}.

\item It is convenient to analyze the $W$ update in Equation~\eqref{eq:4.4} to derive convergence rates for the online update in one time epoch. Once this is obtained, we also need error bounds for the batch procedure, and we employ the guarantees from~\citet{agarwal2012noisy}. As in the previous setting of sparse optimization, we combine the two results to obtain an error bound for the online updates by considering multiple time epochs.

It should be noted that we only require LSC condition on $S+L$ and not jointly on $S$ and $L$. This results in an additional higher order term when analyzing the epoch error and therefore does not play a role in the final convergence bound. The LSC bound provides us with sum of sparse and low rank errors for each epoch. i.e., $ \Vert \hat{S}_i-\bar{S}(T_i)+\hat{L}_i-\bar{L}(T_i)\Vert_\mathbb{F}^2$. Next we need to decouple these errors.

\item An added difficulty in the matrix decomposition problem is decoupling the errors for the sparse and low rank estimates. To this end, we impose norm constraints on the estimates of $S$ and $L$, and carry them over from epoch to epoch. On the other hand, at the end of each epoch $M$ is reset. These norm constraints allows us to control the error.
%using estimated $S$ and $L$.
Special care needs to be taken in many steps of the proof to carefully transform the various norm bounds, where a naive analysis would lead to worse scaling in the dimensionality $p$.
 We instead carefully project the error matrices on to on and off support of $S^*$ for the $\ell_1$ norm term, and similarly onto the range and its complement of $L^*$ for the nuclear norm term. This allows us to have a  convergence rate with a $s+r$ term, instead of $p$.
 
 \item For fixed epoch length, the  error shrinkage stops after some epochs but the error does not increase significantly afterwards. Following lines of~\citep{AgarwalNW12}, we prove that for this case the convergence rate is worse by a factor of $\log p$.
\end{enumerate}

Thus, our careful analysis leads to tight guarantees for online matrix decomposition.
For Proof outline and detailed proof of Theorem~\ref{thm:latent_optimal_fixedepoch} see Appendix~\ref{sec:outline2} and~\ref{sec:theorem2proof} respectively.

\subsection{Graphical Model Selection} \label{sec:graphmodel}
Our framework cannot directly handle the case where loss function is the log likelihood objective. This is because for log likelihood function Lipschitz constant can be large and this leads to loose bounds on error. Yet,  as we discuss shortly, our analysis
needs conditions weaker than Local Lipschitz property. We consider both
settings, i.e., fully observed graphical models and latent Gaussian graphical models. We apply sparse optimization to the former and tackle the latter with sparse $+$ low rank decomposition.

\subsubsection{ Sparse optimization for learning  Gaussian graphical models}\label{sec:glasso}
Consider a $p$-dimensional Gaussian random vector $[x_1, . . . ,x_p]^\top$ with a sparse inverse covariance or precision matrix  $\Theta^*$.  Consider the $\ell_1$-regularized maximum likelihood estimator (batch estimate),
\begin{equation}
\label{glasso}
\widehat{\Theta}:= \underset{\Theta \succ 0}{\arg \min}{\lbrace\Tr(\widehat{\Sigma}\Theta)-\log\det\lbrace\Theta\rbrace+\lambda_n\Vert {\Theta}\Vert_1}\rbrace,
\end{equation}
where $\widehat{\Sigma}$ is the empirical covariance matrix for the batch. This is a well-studied method for  recovering the edge structure in a Gaussian graphical model, i.e. the sparsity pattern of $\Theta^*$~\citep{Ravikumar&etal:08Arxiv}. We have that the loss function is strongly convex for all $\Theta$ within a ball\footnote{Let $Q=\lbrace \theta \in \mathbb{R}^n : \alpha I_n \preceq \Theta \beta I_n \rbrace$ then $-\log \det \Theta$ is strongly convex on $Q$ with $\gamma=\frac{1}{\beta^2}$~\citep{d2008first}.}.

However,  the above loss function is not (locally) Lipschitz in general,
since the gradient\footnote{The gradient computation can be expensive since it involves computing the matrix inverse. However, efficient  techniques for computing an approximate inverse can be employed, on lines of~\citep{hsieh2011sparse}.} $\nabla f(x, \Theta)= x x^\top-\Theta^{-1}$ is not bounded in general. Thus, the bounds derived in Theorem~\ref{thm:sparse_const_epoch} do not directly apply here. However, our conditions for recovery are somewhat weaker than local Lipschitz property, and we provide guarantees for this setting under some additional constraints.

Let $\Gamma^*= {\Theta^*}^{-1} \otimes {\Theta^*}^{-1}$ denote the Hessian of log-determinant barrier at true information matrix. Let  $Y_{(j,k)}:=X_j X_k-\mathbb{E}[X-jX_k]$ and note that $\Gamma^*_{(j,k),(l,m)}=\mathbb{E}[Y_{(j,k)Y_{(l,m)}}]$~\citep{Ravikumar&etal:08Arxiv}.  A bound on $\gennorm{\Gamma^*}_\infty$ limits the influence of the edges on each other, and we need this bound for guaranteed convergence. Yet, this bound contributes to a higher order term and  does not show up in the convergence rate.

%\{$R_1 < \lambda_{\min}(\Theta^*)=1/\| \Sigma^*\|$ is redundant now}

\begin{corollary} Under Assumptions A1, A2 when
 the initialization radius $R_1$ satisfies $R_1 \leq \frac{0.25}{\Vert \Sigma^*\Vert_\mathbb{F}} $,  under the negative log-likelihood loss function, REASON 1 has the following bound (for dual update step size $\tau=\sqrt{T_0}$)
\begin{align*}
\Vert \bar{\theta}_T-\theta^*\Vert_2^2 \leq c_0 \frac{s}{\gamma^2 T}\cdot \frac{\log d}{k_T}\left[ \log d+\sigma^2\left(w^2+24\log (k_T/\log d)\right) \right]
\end{align*}\end{corollary}

The proof does not follow directly from Theorem~\ref{thm:sparse_const_epoch}, since it does not utilize Lipschitz property. However, the conditions for Theorem~\ref{thm:sparse_const_epoch} to hold are weaker than (local) Lipschitz property and we utilize it to provide the above result. For proof, see Appendix~\ref{sec:sparse_graph_proof}. Note that in case epoch length is not fixed and depends on the problem parameters, the bound can be improved by a $\log d$ factor.

Comparing to Theorem~\ref{thm:sparse_const_epoch}, the local Lipschitz constant $G^4$ is replaced by $\sigma^2 \gennorm{\Gamma^*}^2$. We have $G=\order(d)$, and thus we can obtain better bounds in the above result, when $\gennorm{\Gamma^*}$ is small and the initialization radius $R_1$ satisfies the above condition.
Intuitively, the initialization condition (constraint on $R_1$) is dependent on the strength of the correlations. For the weak-correlation case, we can initialize with large error compared to the strongly correlated setting.

\subsubsection{Sparse + low rank decomposition for learning latent Gaussian graphical models}
%\paragraph{3) : }
Consider the Bayesian network on $p$-dimensional observed variables as \begin{equation}\label{eqn:bayesian} y = A \,h + B \, y + n,\quad y, n \in \Rbb^p, \, h\in \Rbb^r,\end{equation} as in Figure~\ref{fig:latent} where $h, y$ and $n$ are drawn from a zero-mean multivariate Gaussian distribution. The vectors $h$ and $n$ are independent of one another, and $n \sim \mathcal{N}(0, \sigma_n^2 I)$. Assume that $A$ has full column rank. Without loss of generality, we assume that $A$ has normalized columns, and that $h$ has independent entries~\citep{pitman2012archimedes}. For simplicity, let $h \sim \mathcal{N}(0, \sigma_h^2 I)$ (more generally, its covariance is a diagonal matrix).
Note that the matrix $B=0$ in the previous setting (the previous setting allows  for more general sub-Gaussian distributions, and here, we limit ourselves to the Gaussian distribution).
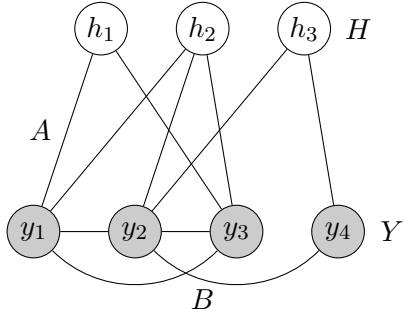
\begin{figure}
\begin{center}
\begin{tikzpicture}
  [
    scale=0.9,
    observed/.style={circle,minimum size=0.7cm,inner
sep=0mm,draw=black,fill=black!20},
    hidden/.style={circle,minimum size=0.7cm,inner sep=0mm,draw=black},
  ]
  \node [hidden,name=h2] at ($(0,0)$) {$h_2$};
    \node [hidden,name=h3] at ($(1.5,0)$) {$h_3$};
      \node [hidden,name=h1] at ($(-1.5,0)$) {$h_1$};
  \node [observed,name=x1] at ($(-2.5,-3)$) {$y_1$};
  \node [observed,name=x2] at ($(-1,-3)$) {$y_2$};
  \node [observed,name=x3] at ($(.5,-3)$) {$y_3$};
    \node [observed,name=x4] at ($(2,-3)$) {$y_4$};
  %\node at ($(0.5,-1)$) {$\dotsb$};
  \node at ($(2.3,0)$) {$H$};
     \node at ($(-2.4,-1.5)$) {$A$};
    \node at ($(2.8,-3)$) {$Y$};
       \node at ($(0,-4)$) {$B$};
  \draw [-] (h1) to (x1);
  %\draw [-] (h1) to (x2);
  \draw [-] (h1) to (x3);
    \draw [-] (h2) to (x1);
  \draw [-] (h2) to (x2);
  \draw [-] (h2) to (x3);
    \draw [-] (h3) to (x2);
  \draw [-] (h3) to (x4);
      \draw [-] (x3) to (x2);
      \draw [-] (x1) to (x2);
  \draw [-] (x2) to[out=-45,in=-135] (x4);
    \draw [-] (x1) to[out=-45,in=-135] (x3);
\end{tikzpicture}
\end{center}
\caption{\small Graphical representation of a latent variable model.}
\label{fig:latent}
\end{figure}
For the model in \eqref{eqn:bayesian}, the precision matrix $M^*$ with respect to the marginal distribution on the observed vector $y$ is given by
\begin{equation}\label{eqn:schur} M^*\coloneqq {\Sigma^*}^{-1}_{y,y}= {\widetilde{M}^*}_{y,y}-{\widetilde{M}^*}_{y,h} ({\widetilde{M}^*}_{h,h})^{-1}
{\widetilde{M}^*}_{h,y},  \end{equation} where $\widetilde{M}^*= {\Sigma^*}^{-1}$, and $\Sigma^*$ is the joint-covariance matrix of vectors $y$ and $h$. It is easy to see that the second term in \eqref{eqn:schur} has rank at most $r$.
The first term in \eqref{eqn:schur} is sparse under some natural constraints, \viz when the matrix $B$ is sparse,  and there are a small number of {\em colliders} among the observed variables $y$. A triplet of variables consisting of two {\em parents} and their {\em child} in a Bayesian network is termed as a collider. The presence of colliders results in additional edges when the Bayesian network on $y$ and $h$ is converted to an undirected graphical model, whose edges are given by the sparsity pattern ${\widetilde{M}^*}_{y,y}$, the first term in \eqref{eqn:schur}.
Such a process is known as {\em moralization}~\citep{Lauritzen:book}, and it involves introducing new edges between the parents in the directed graph (the graph of the Bayesian networks), and removing the directions to obtain an undirected model.
Therefore, when the matrix $B$ is sparse,  and there are a small number of colliders among the observed variables $y$, the resulting sub-matrix ${\widetilde{M}^*}_{y,y}$ is also sparse.

We thus have the precision matrix $M^*$  in \eqref{eqn:schur} as $M^*= S^*+ L^*$, where $S^*$ and $L^*$ are sparse and low rank components. We can find this decomposition via regularized maximum likelihood. The batch estimate is given by~\citet{chandrasekaran2012latent}
\begin{align}
\label{eq:latent_s}
\{ \hat{S}, \hat{L}\}:=&{\arg \min}{\lbrace\Tr (\widehat{\Sigma}_n M)-\log\det M+\lambda_n\Vert {S}\Vert_1+\mu_n \Vert L \Vert_*\rbrace}, \\
&\st \quad M=S+L.
\end{align}
This is a special case of \eqref{eq:latent} with the loss function $f(M)=\Tr (\widehat{\Sigma}_n \, M)-\log\det M$.  In this case, we have the error $E_k=y_k y_k^\top-{M^*}^{-1}$. Since $y = (I-B)^{-1} (Ah +n)$,   we have the following  bound  w.h.p.
 \[ \|E_k\|_2 \leq \mathcal{O}\left(\frac{\sqrt{p}\cdot(\|A\|_2^2 \sigma^2_h + \sigma^2_n)\log( p T)}{\sigma_{\min}(I-B)^2} \right), \quad \forall\, k \leq T,\] where $\sigma_{\min}(\cdot)$ denotes the minimum singular value. The above result is obtained by alluding to \eqref{eqn:cov-conc}.

When $\|A\|_2$  and $\sigma_{\min}(I-B)$ are bounded, we thus achieve  optimal scaling for our proposed online method. As discussed for the previous case, when $A$ is generically drawn, $\|A\|_2$ is bounded. To bound $\sigma_{\min}(I-B)$, a sufficient condition is {\em walk-summability} on the  sub-graph among the observed variables $y$. The class of walk-summable models  is efficient for inference~\citep{malioutov2006walk} and structure learning~\citep{AnimaW}, and they contain the class of attractive models. Thus, it is perhaps not surprising that we obtain efficient guarantees for such models for our online algorithm.

 We need to slightly change the algorithm REASON 2 for  this scenario as follows: for the $M$-update in REASON 2, we add a $\ell_1$ norm constraint on $M$ as $\|M_k - \tilde{S}_i - \tilde{L}_i\|_1^2 \leq \breve{R}^2$, and this can still be computed efficiently, since it involves projection on to the $\ell_1$ norm ball, see  Appendix~\ref{sec:implement1}.   We assume a good initialization $M$ which satisfies $\|M-M^*\|_1^2 \leq \breve{R}^2$.

This ensures that $M_k$ in subsequent steps is non-singular, and that the gradient of the loss function $f$ in \eqref{eq:latent_s}, which involves $M^{-1}_k$,  can be computed. As observed in section~\ref{sec:glasso}  on sparse graphical model selection, the method can be made more efficient by computing approximate matrix inverses~\citep{bigquick}. As observed before, the loss function $f$ satisfies the local strong convexity property, and the guarantees in Theorem~\ref{thm:latent_optimal_fixedepoch} are applicable.

There is another reason for using the $\ell_1$ bound. Note that the loss function is not generally Lipschitz in this case.
However, our conditions for recovery are somewhat weaker than local Lipschitz property, and we provide guarantees for this setting under some additional constraints. Let $\Gamma^*=M^* \otimes M^*$. As explained in Section~\ref{sec:glasso}, a bound on $\gennorm{\Gamma^*}_\infty$ limits the influence on the edges on each other, and we need this bound for guaranteed convergence. Yet, this bound contributes to a higher order term and  does not show up in the convergence rate.

\begin{corollary} Under Assumptions A1, A2, A4, A5, when
 the radius $\breve{R}$ satisfies $\breve{R} \leq \frac{0.25}{\Vert \Sigma^*\Vert_\mathbb{F}} $,  under the negative log-likelihood loss function, REASON 2 has the following bound (for dual update step size $\tau=\sqrt{T_0}$)
\begin{align*}
&\Vert \bar{S}(T)-S^* \Vert_\mathbb{F}^2 + \Vert \bar{L}(T)-L^* \Vert_\mathbb{F}^2 \leq \\&
\frac{c_0(s+r)}{T}\cdot \frac{\log p}{k_T}\left[\log p+\beta^2(p)\sigma^2\left(w^2+\log (k_T/\log p)\right)\right]+\max \lbrace s+r,p \rbrace\frac{\alpha^2}{p}.
\end{align*}\end{corollary}

The proof does not follow directly from Theorem~\ref{thm:latent_optimal_fixedepoch}, since it does not utilize Lipschitz property. However, the conditions for Theorem~\ref{thm:latent_optimal_fixedepoch} to hold are weaker than (local) Lipschitz property and we utilize it to provide the above result. For proof, see Appendix~\ref{sec:latent_graph_proof}. Note that in case epoch length is not fixed and depends on the problem parameters, the bound can be improved by a $\log p$ factor.

\section{Experiments}
\subsection{REASON 1}
For sparse optimization problem we compare REASON 1 with RADAR and ST-ADMM under the least-squares regression setting. Samples $(x_t,y_t)$ are generated such that $x_t \in \text{Unif}[-B,B]$ and $y_t=\langle \theta^*,x \rangle+n_t$. $\theta^*$ is $s$-sparse with $s=\lceil \log d \rceil$. $n_t \sim \mathcal{N}(0,\eta^2)$. With $\eta^2=0.5$ in all cases. We consider $d=20, 2000, 20000$ and $s=1,3,5$ respectively.
%\begin{wrapfigure}{r}{0.49\textwidth}
\begin{figure}[t]
%\begin{center}
\centering {\begin{psfrags}
\psfrag{tttt}[c][t]{\scriptsize  Error}
\psfrag{rr}[c]{\scriptsize  Iteration Number}
\psfrag{t1}[c][t]{\scriptsize  \quad REASON 1, $d_2$}
\psfrag{t2}[c][t]{\scriptsize  REASON 1, $d_1$}
\psfrag{t3}[c][t]{\scriptsize  ST-ADMM, $d_2$}
\psfrag{t4}[c][t]{\scriptsize  ST-ADMM, $d_1$}
\includegraphics[width=1\textwidth]{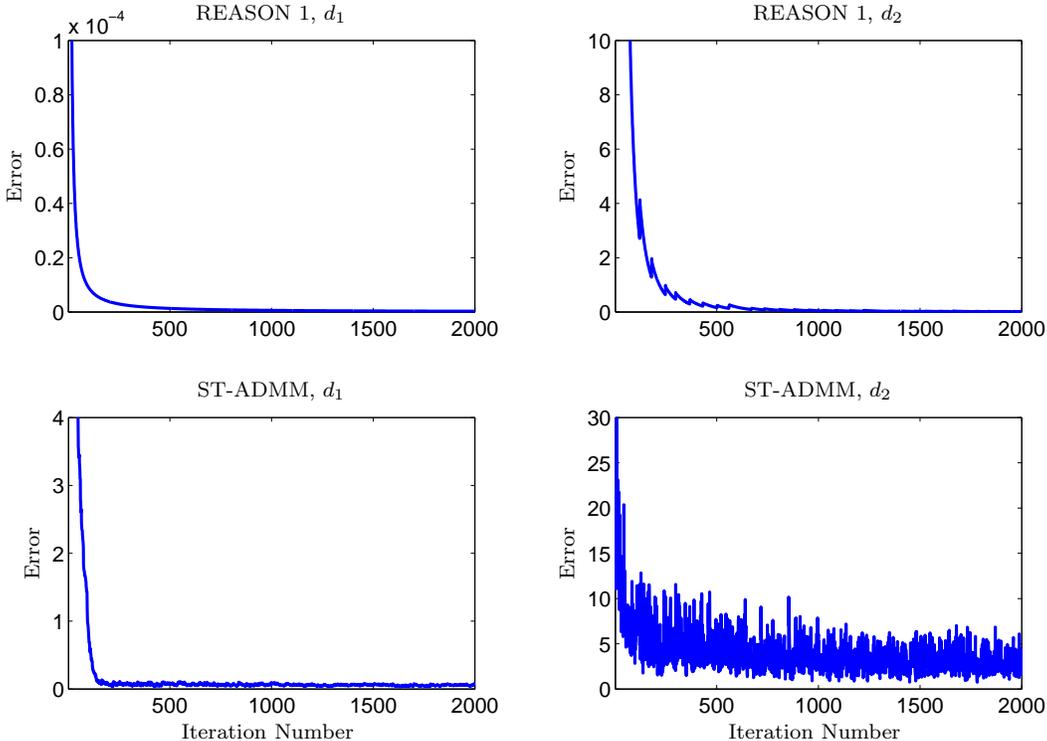}\end{psfrags}}
%\end{center} %admmflucd20k_norm2.eps
\vspace{-15pt}
\caption{\small Least square regression, Error$=\frac{\Vert {\theta}-\theta^* \Vert_2}{\Vert \theta^* \Vert_2}$ vs. iteration number, $d_1=20$ and $d_2=20000$.}
\label{fig:fluc}
%\end{wrapfigure}
\end{figure}
The experiments are performed on a $2.5$ GHz Intel Core i5 laptop with 8 GB RAM. See Table~\ref{table:REASONf} for experiment results. It should be noted that { RADAR is provided with information of $\theta^*$ for epoch design and recentering. In addition, both RADAR and REASON 1 have the same initial radius. Nevertheless, REASON 1 reaches better accuracy within the same run time even for small time frames.} 
In addition, we compare relative error ${\Vert {\theta}-\theta^* \Vert_2}/{\Vert \theta^* \Vert_2}$ in REASON 1 and  ST-ADMM in the first epoch.
%\aacomment{in summary of results i said reason1 is comparable to radar.. change that?}
%{\bf REASON 1 vs. ADMM}
We observe that
%as the dimension increases it becomes harder to tune ADMM whereas since REASON 1 restricts the $\theta$-update, it is less critical to tune.
in higher dimension error fluctuations for ADMM increases noticeably (see Figure~\ref{fig:fluc}).
%In case of tuned algorithms for medium sized dimensions they behave almost the same in the first epoch. Afterwards, if we look at the  error $\Vert\bar{\theta}(k)-\theta^*\Vert$ right before end of an epoch in REASON and at corresponding iteration number in ADMM, we notice that REASON 1 shows a faster convergence.
% Although both algorithms get slower as the error shrinks noticeably, REASON 1 performs faster.
%This is intuitive as REASON 1 is directing the answer towards $\theta^*$ (i.e. information from epoch center value that we know is approaching $\theta^*$ in theory) while ADMM is not using additional information.
Therefore, projections of REASON 1 play an important role in denoising and obtaining good accuracy. %For higher dimensions ADMM obviously lags behinds REASON even in the first epoch(see Table~\ref{table:REASONf})
\begin{table}[t]
\begin{center}
\begin{tabular}{c|c|c||ccc}
\hline
\hline
Dimension& Run Time (s) &Method &error at 0.02T&error at 0.2T  & error at T \\ \hline\hline
%d=20000 & &&&\\
&&ST-ADMM&1.022 &1.002 &0.996 \\
d=20000& T=50&RADAR &0.116 &2.10e-03	&6.26e-05\\
&&REASON 1 &1.5e-03& 2.20e-04	& 1.07e-08 \\ \hline
%d=2000 & &&&\\
&&ST-ADMM &0.794 &0.380	& 0.348	\\
d=2000&T=5&RADAR&0.103 &4.80e-03	& 1.53e-04\\
&&REASON 1 &0.001& 2.26e-04 &	1.58e-08	 \\ \hline
%d=20  & &&&\\
&&ST-ADMM &0.212 &0.092	& 0.033\\
d=20&T=0.2& RADAR&0.531&4.70e-03&	4.91e-04\\
&&REASON 1 &0.100 &2.02e-04	&1.09e-08\\ \hline
\end{tabular}
\caption{\em  Least square regression problem, epoch size $T_i=2000$, Error$=\frac{\Vert {\theta}-\theta^* \Vert_2}{\Vert \theta^* \Vert_2}$.}
\label{table:REASONf}
\end{center}
\end{table} 

\begin{table}[t]
\begin{center}

\begin{tabular}{c||ccc|cccc}
\hline
Run Time &\multicolumn{3}{c|}{$T=50$ sec} & \multicolumn{3}{c}{$T=150$ sec} \\
\hline

 Error & $\frac{\Vert M^*-S-L\Vert_\mathbb{F}}{\Vert M^*\Vert_\mathbb{F}} $ & $\frac{\Vert S-S^*\Vert_\mathbb{F}}{\Vert S^*\Vert_\mathbb{F}}$ &$\frac{\Vert L^*-L\Vert_\mathbb{F}}{\Vert L^*\Vert_\mathbb{F}}$  & $\frac{\Vert M^*-S-L\Vert_\mathbb{F}}{\Vert M^*\Vert_\mathbb{F}} $ & $\frac{\Vert S-S^*\Vert_\mathbb{F}}{\Vert S^*\Vert_\mathbb{F}}$ &$\frac{\Vert L^*-L\Vert_\mathbb{F}}{\Vert L^*\Vert_\mathbb{F}}$\\ \hline \hline
%\end{center} \end{table}
\begin{tabular}{c} REASON 2 \\ inexact ALM\end{tabular} &\begin{tabular}{c} 2.20e-03 \\ 5.11e-05 \end{tabular} &\begin{tabular}{c}  0.004 \\ 0.12 \end{tabular} &\begin{tabular}{c} 0.01\\  0.27 \end{tabular}  &\begin{tabular}{c} 5.55e-05 \\ 8.76e-09 \end{tabular} &\begin{tabular}{c}  1.50e-04 \\ 0.12 \end{tabular} &\begin{tabular}{c} 3.25e-04\\  0.27 \end{tabular}
\\ \hline
\end{tabular}
\caption{\em REASON 2 and inexact ALM, matrix decomposition problem. $p=2000$, $\eta^2=0.01$}
\label{table:REASON22w}
\end{center}
\end{table}

\paragraph{Epoch Size}
For fixed- epoch size, if epoch size is designed such that the relative error defined above has shrunk to a stable value, then we move to the next epoch and the algorithm works as expected. If we choose a larger epoch than this value we do not gain much in terms of accuracy at a specific iteration. On the other hand if we use a small epoch size such that the relative error is still noticeable, this delays the error reduction and causes some local irregularities.

\subsection{REASON 2}
We compare REASON 2 with state-of-the-art inexact ALM method for matrix decomposition problem\footnote{{ ALM codes are downloaded from {\url{http://perception.csl.illinois.edu/matrix-rank/home.html}} and REASON 2 code is available at {\url{https://github.com/haniesedghi/REASON2}}.}} In this problem $M$ is the noisy sample the algorithm receives. Since we have direct access to $M$, the $M$-update is eliminated.

Table~\ref{table:REASON22w} shows that  with equal time, inexact ALM reaches smaller $\frac{\Vert M^*-S-L\Vert_\mathbb{F}}{\Vert M^*\Vert_\mathbb{F}} $ error while in fact this does not provide a  good  decomposition. On the other hand, REASON 2 reaches useful individual errors in the same time frame.  Experiments with $\eta^2 \in [0.01,1]$ reveal similar results. This emphasizes the importance of projections in REASON 2.
Further investigation on REASON 2 shows that performing one of the projections (either $\ell_1$ or nuclear norm) suffices to reach this performance. The same precision can be reached using only one of the projections. Addition of the second projection improves the performance marginally. Performing nuclear norm projections are much more expensive since they require SVD. Therefore, it is more efficient to perform the  $\ell_1$ projection.
 Similar experiments on exact ALM shows worse performance than inexact ALM  and are thus omitted.

\section{Conclusion}
In this paper, we consider a modified version of the stochastic ADMM method  for high-dimensional problems. We first analyze the simple setting, where the optimization problem consists of a loss function and a single regularizer, and then extend to  the multi-block setting with multiple regularizers and multiple variables.
For the sparse optimization problem, we showed that we reach the minimax-optimal rate in this case, which implies that our guarantee is unimproveable by any (batch or online) algorithm (up to constant factors).
We then consider the matrix decomposition problem into sparse and low rank components, and propose a modified version of the multi-block ADMM algorithm. Experiments show that for both sparse optimization and matrix decomposition problems, our algorithm outperforms  the state-of-the-art  methods. In particular, we reach higher accuracy with same time complexity.
There are various future problems to consider. One is to provide lower bounds on error for matrix decomposition problem in case of strongly convex loss if possible.~\citet{agarwal2012noisy} do not provide bounds for strongly convex functions.
Another approach can be to extend our method to address nonconvex programs.~\citet{loh2013regularized} and \citet{wang2013optimal} show that if the problem is nonconvex but has additional properties, it can be solved by methods similar to convex loss programs. % methods for optimizing convex programs.
In addition, we can extend our method to coordinate descent methods such as~\citep{roux2012stochastic}.

\subsection*{Acknowledgment}
We acknowledge detailed discussions with Majid Janzamin and thank him for valuable comments on sparse and sparse $+$ low rank recovery. The authors thank Alekh Agarwal for detailed discussions about his work and the minimax bounds. We thank him for pointing out that the condition regarding bounded dual variables is related to  local Lipschitz condition.

A. Anandkumar
is supported in part by Microsoft Faculty Fellowship, NSF Career award CCF-$1254106$, NSF Award CCF-$1219234$, and ARO YIP Award W$911$NF-$13$-$1$-$0084$.

\newpage
\appendix
\section{Guarantees for REASON 1} \label{appendix:thmsparse}
First, we provide guarantees for the theoretical case such that epoch length depends on epoch radius. This provides intuition on how the algorithm is designed. The fixed-epoch algorithm is a special case of this general framework. We first state and prove guarantees for general framework. Next, we leverage these results to prove Theorem~\ref{thm:sparse_const_epoch}.

Let the design parameters be set as
\begin{align}
\label{eq:Ti}
&T_i=C \frac{s^2}{\gamma^2}\left[ \frac{\log d+12\sigma_i^2\log(3/\delta_i)}{R_i^2} \right],\\ \nonumber
&\lambda_i^2=\frac{{\gamma}}{s\sqrt{T_i}}\sqrt{R_i^2 \log d+\frac{ G^2 R_i^2+\rho_x^2 R_i^4}{T_i}+\sigma_i^2R_i^2\log(3/\delta_i)},\\ \nonumber
&\rho \propto \frac{\sqrt{\log d}}{R_i\sqrt{T_i}},\quad \rho_x > 0, \quad \tau=\rho. %\propto\frac{1}{R_i\sqrt{T_i}}.
\end{align}

%%%%%%%%%%%%%%%%%%%%%%%
\begin{theorem} %equivalent to theorem 4 Agarwal
\label{thm:sparse_optimal} Under assumptions $A1-A3$ and parameter settings~\eqref{eq:Ti}, there exists a constant $c_0 >0$ such that REASON 1 satisfies
 for all $T>k_T$,
\begin{align}
\label{eq:sparse_optimal}
\Vert \bar{\theta}_T-\theta^*\Vert_2^2 \leq c_0 \frac{s}{\gamma^2 T}\left[ e\log d +\sigma^2w^2+\log k_T) \right],
\end{align}
with probability at least $1-6\exp(-w^2/12)$, where  $k_T=\log_2 \frac{\gamma^2R_1^2 T}{s^2( \log d+12\sigma^2\log(\frac{6}{\delta}))}$, and $c_0$ is a universal constant.
\end{theorem}
For Proof outline and detailed proof of Theorem~\ref{thm:sparse_optimal} see Appendix~\ref{sec:outline1} and~\ref{sec:theorem1proof} respectively.

\subsection{Proof outline for Theorem \ref{thm:sparse_optimal}} \label{sec:outline1}
The foundation block for this proof is Proposition~\ref{thm:prop1}. %Let $\log d=\log d$.
\begin{proposition}
\label{thm:prop1}
Suppose $f$ satisfies Assumptions $A1, A2$ with parameters $\gamma$ and $\sigma_i$ respectively and assume that $\Vert \theta^*-\tilde{\theta}_i \Vert_1^2 \leq R_i^2$. %explain f,assp,theta^*
We apply the updates in REASON 1 with parameters as in \eqref{eq:Ti}. Then, there exists a universal constant $c$ such that for any radius $R_i$ %and sparsity level $s$ such that $\gamma-8s \tau >0$, we have
\begin{align}
\sublabon{equation}
\label{eq:prop1aval}
 & f(\bar{\theta}(T_i))-f(\hat{\theta}_i)+\lambda_i \Vert \bar{y}(T_i) \Vert_1-\lambda_i \Vert \hat{\theta}_i \Vert_1
\leq  \frac{ R_i \sqrt{ \log d}}{\sqrt{T_i}}+\frac{ G R_i}{T_i}+\frac{\rho_x R_i^2}{{T_i}}\\ \nonumber &+\frac{R_i \sigma_i}{\sqrt{T_i}} \sqrt{12 \log (3/\delta_i)}, \\
\label{eq:prop1dovom}
&\Vert \bar{\theta}(T_i)-\theta^* \Vert^2_1 \leq \frac{c'}{\sqrt{C}}R_i^2.
\end{align}
\sublaboff{equation}
where $\rho_0=\rho_x+\rho$ and both bounds are valid with probability at least $1-\delta_i$.
\end{proposition}
Note that our proof for epoch optimum improves proof of~\citep{BADMM} with respect to $\rho_x$. For details, see Section~\ref{sec:epochopt1}.

In order to prove Proposition~\ref{thm:prop1}, we need to prove some more lemmas.

To move forward from here please note the following notations: $\Delta_i=\hat{\theta}_i-\theta^*$ and
 $\hat{\Delta}(T_i)=\bar{\theta}_i-\hat{\theta}_i$.
 \begin{lemma}
 \label{lemma:A1}
 At epoch $i$ assume that $\Vert \theta^*-\tilde{\theta}_i \Vert_1 \leq R_i$. Then the error $ {\Delta}_i $ satisfies the bounds
\begin{align}
\sublabon{equation}
\label{eq:lemma1a_aval}
&\Vert \hat{\theta}_i-\theta^* \Vert_2 \leq \frac{4}{{\gamma}}\sqrt{s}\lambda_i,\\
\label{eq:lemma1a_dovom}
&\Vert \hat{\theta}_i-\theta^* \Vert_1 \leq \frac{8}{{\gamma}}s\lambda_i.
\end{align}
\sublaboff{equation}
\end{lemma}
\begin{lemma}
\label{lemma:A2}
Under the conditions of Proposition~\ref{thm:prop1} and with parameter settings \eqref{eq:Ti} , we have
\begin{align*}
%\frac{\bar{\gamma}}{2} \Vert \hat{\Delta}(T_i) \Vert_2^2 \leq C_0\sqrt{\frac{s}{\gamma}}R_i^2
\Vert \hat{\Delta}(T_i) \Vert_2^2 \leq \frac{c'}{\sqrt{C}}\frac{1}{s} R_i^2,
\end{align*}
 with probability at least $1-\delta_i$.
 \end{lemma}

\section{Proof of Theorem~\ref{thm:sparse_optimal} } \label{sec:theorem1proof}
The first step is to ensure that $\Vert \theta^*-\tilde{\theta}_i\Vert \leq R_i$ holds at each epoch so that Proposition~\ref{thm:prop1} can be applied in a recursive manner. We prove this by induction on the epoch index. By construction, this bound holds at the first epoch. Assume that it holds for epoch $i$. Recall that $T_i$ is defined by~\eqref{eq:Ti} where $C \geq 1$ is a constant we can choose. By substituting this $T_i$ in inequality ~\eqref{eq:prop1dovom}, the simplified bound~\eqref{eq:prop1dovom} further yields
\begin{align*}
\Vert \bar{\theta}(T_i)-\theta^*\Vert_1^2 \leq  \frac{c'}{\sqrt{C}}R_i^2.
\end{align*}
Thus, by choosing $C$ sufficiently large, we can ensure that $ \Vert \bar{\theta}(T_i)-\theta^*\Vert_1^2 \leq  R_i^2/2:=R_{i+1}^2$.
Consequently, if $\theta^*$ is feasible at epoch $i$, it stays feasible at epoch $i+1$. Hence, by induction we are guaranteed the feasibility of $\theta^*$ throughout the run of algorithm.

As a result, Lemma~\ref{lemma:A2} applies and we find that
\begin{align}
\label{eq:A42}
\Vert \hat{\Delta}(T_i) \Vert_2^2 \leq \frac{c}{s}R_i^2.
\end{align}
We have now bounded $\hat{\Delta}(T_i)=\bar{\theta}(T_i)-\hat{\theta}_i$ and Lemma \ref{lemma:A1} provides a bound on $\Delta_i=\hat{\theta}_i-\theta^*$, such that the error $\Delta^*(T_i)=\bar{\theta}(T_i)-\theta^*$ can be controlled by triangle inequality. In particular, by combining \eqref{eq:lemma1a_aval} with \eqref{eq:A42}, we get
\begin{align*}
\Vert \Delta^*(T_i)\Vert_2^2 \leq c \lbrace\frac{1}{s}R_i^2+\frac{16}{s}R_i^2 \rbrace,
\end{align*}
i.e.
\begin{align}
\label{eq:A43}
\Vert \Delta^*(T_i)\Vert_2^2 \leq c\frac{R_1^22^{-(i-1)}}{s}.
\end{align}
The bound holds with probability at least $1-3\exp(-w_i^2/12)$. Recall that $R_i^2=R_1^22^{-(i-1)}$.
Since $w_i^2=w^2+24\log i$, we can apply union bound to simplify the error probability as $1-6\exp(-w^2/12)$. Throughout this report we use $\delta_i=3\exp(-w_i^2/12)$ and $\delta=6\exp(-w^2/12)$ to simplify the equations.

To complete the proof we need to convert the error bound \eqref{eq:A43} from its dependence on the number of epochs $k_T$ to the number of iterations needed to complete $k_T$ epochs, i.e. $T(K)=\sum_{i=1}^k T_i$. Note that here we use $T_i$ from~\eqref{eq:Ticomplete}, to show that when considering the dominant terms, the definition in~\eqref{eq:Ti} suffices. Here you can see how negligible terms are ignored.
\begin{align*}
T(k)&=\sum_{i=1}^k C \left[\frac{s^2}{{\gamma^2}}\left[ \frac{\log d+12\sigma_i^2\log(3/\delta_i)}{R_i^2} \right]+\frac{s}{\gamma}\frac{G}{R_i}+\frac{s}{\gamma}\rho_x\right]\\ &=
C\sum_{i=1}^k  \left[\frac{s^2\lbrace \log d+\gamma/sG+\sigma^2(w^2+24\log k) \rbrace 2^{i-1}}{\gamma ^2 R_1^2} +\frac{s G}{\gamma R_1}\sqrt{2}^{i-1}+\frac{s}{\gamma}\rho_x\right]. 
\end{align*}
{Hence},
\begin{align*}
T(k) &\leq C  \left[\frac{s^2}{\gamma ^2 R_1^2}\lbrace \log d+\sigma^2(w^2+24\log k) \rbrace 2^k+\frac{s}{\gamma R_1}G\sqrt{2}^{k}+\frac{s}{\gamma}\rho_x \right].
\end{align*}
$T(k) \leq S(k)$, therefore $k_T \geq S^{-1}(T)$.
\begin{align*}
S(k)&=C  \left[\frac{s^2}{\gamma ^2 R_1^2}\lbrace \log d+\sigma^2(w^2+24\log k) \rbrace 2^k+\frac{s}{\gamma R_1}G\sqrt{2}^{k}+\frac{s}{\gamma}\rho_x\right]. 
\end{align*}
Ignoring the dominated terms and using a first order approximation for $\log (a+b)$,
\begin{align*}
&\log (T) \simeq \log C+k_T+\log \left[\frac{s^2}{\gamma ^2 R_1^2}\lbrace \log d+\sigma^2(w^2+24\log k)\rbrace\right],\\
&k_T\simeq \log T-\log C - \log \left[\frac{s^2}{\gamma ^2 R_1^2}\lbrace \log d+\sigma^2(w^2+24\log k)\rbrace\right].
\end{align*}
Therefore,
%\begin{align*}
%k_T= \log_2 \frac{T R_1^2 \gamma^2}{s\gamma \log d+s^2 \sigma^2 w^2}
%\end{align*}
\begin{align*}
2^{-k_T}=\frac{Cs^2}{\gamma^2T R_1^2}\lbrace \log d+\sigma^2(w^2+24\log k) \rbrace.
\end{align*}
Putting this back into \eqref{eq:A43}, we get that
\begin{align*}
\Vert \Delta^*(T_i)\Vert_2^2 &\leq c\frac{R_1^2}{s}\frac{Cs^2}{\gamma^2T R_1^2}\lbrace \log d+\sigma^2(w^2+24\log k) \rbrace \\ &\leq c\frac{s}{\gamma^2 T}\lbrace \log d+\sigma^2(w^2+24\log k) \rbrace.
\end{align*}
Using the definition $\delta=6\exp(-w^2/12)$, above bound holds with probability $1-\delta$. Simplifying the error in terms of $\delta$ by replacing $w^2$ with $12\log(6/\delta)$, gives us \eqref{eq:sparse_optimal}.

\subsection{Proofs for Convergence within a Single Epoch for Algorithm \ref{our_method}} \label{sec:epochopt1}
%\subsection{Convergence Analysis for Algorithm \eqref{our_method}}
%As shown before, $f(\theta,x_t)=\Tr (\hat{\Sigma},\theta)- \log \det \theta$ is locally strongly convex.% and locally Lipschitz.\\
\begin{lemma} \label{thm:epoch_avg_opt}
For $ \bar{\theta}(T_i) $ defined in Algorithm \ref{our_method} and $\hat{\theta}_i$ the optimal value for epoch $i$, let $\rho=c_1\sqrt{T_i}$, $\rho_x$  some positive constant, $\rho_0=\rho+\rho_x$ and $\tau=\rho$ where $c_1=\frac{\sqrt{\log d}}{R_i}$. We have that
\begin{align}
\label{eq:epoch_avg_opt}
&f(\bar{\theta}(T_i))-f(\hat{\theta}_i)+\lambda_i \Vert \bar{y}(T_i) \Vert_1-\lambda_i \Vert \hat{\theta}_i \Vert_1 \leq \\ \nonumber
&\frac{ R_i \sqrt{ \log d}}{\sqrt{T_i}}+\frac{ G R_i}{T_i}+\frac{\rho_x R_i^2}{{T_i}}+\frac{\sum_{k=1}^{T_i} \langle e_k,\hat{\theta}_i-\theta_k \rangle}{T_i}.
\end{align}
\end{lemma}
\paragraph{Remark :}Please note that as opposed to ~\citep{BADMM} we do not require $\rho_x\propto \sqrt{T_i}$. We show that our parameter setting also works.
\begin{proof}
First we show that our update rule for $\theta$ is equivalent to not linearizing $f$ and using another Bregman divergence. This helps us in finding a better upper bound on error that does not require bounding the subgradient. Note that linearization does not change the nature of analysis. The reason is that we can define $B_f(\theta,\theta_k)=f(\theta)-f(\theta_k)+\langle \nabla f(\theta_k),\theta-\theta_k\rangle$, which means $f(\theta)-B_f(\theta,\theta_k)=f(\theta_k)+\langle \nabla f(\theta_k),\theta-\theta_k\rangle$.

Therefore,
\begin{align*}
{\underset{\Vert \theta-\tilde{\theta}_i\Vert_1^2 \leq R_i^2}{\arg\min}}\lbrace
\langle \nabla f(\theta_k),\theta -\theta_k \rangle \rbrace=
{\underset{\Vert \theta-\tilde{\theta}_i\Vert_1^2 \leq R_i^2}{\arg\min}}\lbrace f(\theta)-B_f(\theta,\theta_k)\rbrace.
\end{align*}
As a result, we can write down the update rule of $\theta$ in REASON 1
%\eqref{eq:lin_B_update} 
as
\begin{align*}
\theta_{k+1} &={\underset{\Vert \theta-\tilde{\theta}_i\Vert_1^2 \leq R_i^2}{\arg\min}}\lbrace
f(\theta)-B_f(\theta,\theta_k)+z_k^T(\theta-y_k)+\rho B_\phi(\theta,y_k) \\
&\quad \quad \quad  \quad \quad ~~+\rho_x B_{\phi'_x}(\theta, \theta_k) \rbrace.
\end{align*}
We also have that $B_{\phi_x}(\theta, \theta_k)=B_{\phi'_x}(\theta, \theta_k)-\frac{1}{\rho_x}B_f(\theta,\theta_k)$, which simplifies the update rule to
\begin{align}
\label{eq:BADMM7eq}
\theta_{k+1} ={\underset{\Vert \theta-\tilde{\theta}_i\Vert_1^2 \leq R_i^2}{\arg\min}}&\lbrace
f(\theta)+\langle z_k,\theta-y_k \rangle+\rho B_\phi(\theta,y_k)+\rho_xB_{\phi_x}(\theta, \theta_k) \rbrace.
\end{align}
We notice that equation \eqref{eq:BADMM7eq} is equivalent to Equation (7)~\citep{BADMM}. Note that as opposed to~\citep{BADMM}, in our setting $\rho_x$ can be set as a constant. Therefore, for completeness we provide proof of convergence and the convergence rate for our setting.
\begin{lemma} 
Convergence of REASON 1: The optimization problem defined in REASON 1 converges.
\end{lemma}
\begin{proof}
 On lines of~\citep{BADMM}, let $\mathbf{R}(k+1)$ stand for residuals of optimality condition. For convergence we need to show that $\underset{k\rightarrow \infty}{\text{lim}}\mathbf{R}(k+1)=0$. Let $w_k=(\theta_k,y_k,z_k)$. Define
\begin{align*}
D(w^*,w_k)=\frac{1}{\tau \rho}\Vert z^*-z_k\Vert_2^2+B_\phi(y^*,y_k)+\frac{\rho_x}{\rho}B_\phi(\theta^*,\theta_k).
\end{align*}
By Lemma 2~\citet{BADMM}
\begin{align*}
\mathbf{R}(t+1) \leq D(w^*,w_k)-D(w^*,w_{k+1}).
\end{align*}
Therefore,
\begin{align*}
\sum_{k=1}^\infty \mathbf{R}(t+1) &\leq D(w^*,w_0) \\
&=\frac{1}{\tau \rho}\Vert z^*\Vert_2^2+B_\phi(y^*,y_0)+\frac{\rho_x}{\rho}B_\phi(\theta^*,\theta_0)\\
& \leq \underset{T\rightarrow \infty}{\text{lim}}\frac{R_i^2}{\log d~T}\Vert \nabla f(\theta^*)\Vert_2^2+2R_i^2+\frac{\rho_x}{\sqrt{T \log d}}R_i^3.
\end{align*}
Therefore, $\underset{k\rightarrow \infty}{\text{lim}}\mathbf{R}(k+1)=0$ and the algorithm converges.
\end{proof}

If in addition we incorporate sampling error, then Lemma 1~\citep{BADMM} changes to
\begin{align*}
&f(\theta_{k+1})-f(\hat{\theta}_i)+\lambda_i \Vert y_{k+1}\Vert_1-\lambda_i \Vert \hat{\theta}_i\Vert_1 \leq
\\& -\langle z_k,\theta_{k+1}-y_{k+1} \rangle-\frac{\rho}{2}\lbrace\Vert \theta_{k+1}-y_k\Vert_2^2+\Vert \theta_{k+1}-y_{k+1}\Vert_2^2 \rbrace+
\langle e_k,\hat{\theta}_i-\theta_k \rangle \\&+\frac{\rho}{2}\lbrace\Vert \hat{\theta}_i-y_k\Vert_2^2-\Vert \hat{\theta}_i-y_{k+1}\Vert_2^2 \rbrace+\rho_x \lbrace B_{\phi_x}(\hat{\theta}_i,\theta_k)-B_{\phi_x}(\hat{\theta}_i,\theta_{k+1})\\
&-B_{\phi_x}(\theta_{k+1},\theta_k)\rbrace.
\end{align*}
The above result follows from convexity of $f$, the update rule for $\theta$ (Equation~\eqref{eq:BADMM7eq}) and the three point property of Bregman divergence.

Next, we show the bound on the dual variable.
\begin{lemma} \label{lemma:dualsparse}
The dual variable in REASON 1 is bounded. i.e., 
\begin{align*}
\Vert z_k \Vert_1 \leq G+2\rho_0 R_i, ~~\text{where}~~ \rho_0:= \rho_x+\rho.
\end{align*}
\end{lemma}
\begin{proof}
 Considering the update rule for $\theta$, we have the Lagrangian
\begin{align*}
\mathcal{L}=f(\theta)+\langle z_k,\theta-y_k \rangle+\rho B_\phi(\theta,y_k)+\rho_xB_{\phi_x}(\theta, \theta_k)+\zeta \left(\Vert \theta_{k+1}-\tilde{\theta}_i\Vert_1-R_i \right),
\end{align*}
where $\zeta$ is the Lagrange multiplier corresponding to the $\ell_1$ bound. We hereby emphasize that $\zeta$ does not play a role in size of the dual variable. i.e., considering the $\ell_1$ constraint, three cases are possible:
\begin{enumerate}
\item $\Vert \theta_{k+1}-\tilde{\theta}_i\Vert_1 > R_i$. By complementary slackness, $\zeta=0$.
\item $\Vert \theta_{k+1}-\tilde{\theta}_i\Vert_1 < R_i$. By complementary slackness, $\zeta=0$.
\item $\Vert \theta_{k+1}-\tilde{\theta}_i\Vert_1 = R_i$. This case is equivalent to the non-constrained update and no projection will take place. Therefore, $z$ will be the same as in the non-constrained update.
\end{enumerate}
Having above analysis in mind, the upper bound on the dual variable can be found as follows
By optimality condition on $\theta_{k+1}$, we have
\begin{align}
\label{eq:dual}
-z_k=\nabla f(\theta_{k+1})+\rho_x(\theta_{k+1}-\theta_k)+\rho(\theta_{k+1}-y_k).
\end{align}
By definition of the dual variable and the fact that $\tau=\rho$, we have that
\begin{align*}
z_{k}=z_{k-1}-\rho(\theta_{k}-y_k)
\end{align*}
Hence, we have that $-z_{k-1}=\nabla f(\theta_{k+1})+(\rho_x+\rho)(\theta_{k+1}-\theta_k)$.
Therefore,
\begin{align*}
\Vert z_{k-1} \Vert_1 \leq G+2\rho_0 R_i, ~~\text{where}~~ \rho_0:= \rho_x+\rho.
\end{align*}
It is easy to see that this is true for all $z_k$ at each epoch.
\end{proof}

Consequently,
\begin{align*}
\frac{-1}{\tau} \langle z_k,z_{k}-z_{k+1} \rangle &= \frac{1}{\tau} \langle 0- z_k, z_{k}-z_{k+1} \rangle \\
&=\frac{1}{2\tau} \left(
\Vert z_{k+1} \Vert^2-\Vert z_{k} \Vert^2-\Vert z_{k+1}-z_k \Vert^2 \right).
\end{align*}
Ignoring the negative term in the upper bound and noting $z_0=0$, we get
\begin{align*}
\frac{1}{T_i}\sum_{k=1}^{T_i}-\langle z_k,\theta_{k+1}-y_{k+1} \rangle &\leq \frac{1}{2\tau T_i} \Vert z_{T_i} \Vert^2 \leq \frac{1}{2\tau T_i}(G+2\rho_0 R_i)^2 \\&\simeq \frac{R_i \sqrt{\log d}}{\sqrt{T_i}}+\frac{ G R_i}{T_i}.
\end{align*}
Note that since we consider the dominating terms in the final bound, terms with higher powers of $T_i$ can be ignored throughout the proof.
%By theory of ADMM~\citep{boyd2011distributed}, $\exists B,\st \forall z, \Vert z \Vert_1 \leq B$. Therefore, $\Vert z \Vert_2 \leq B$.
Next, following the same approach as in Theorem 4~\citep{BADMM} and considering the sampling error, we get,
\begin{align*}
&f(\bar{\theta}(T_i))-f(\hat{\theta}_i)+\lambda_i \Vert \bar{y}(T_i) \Vert_1-\lambda_i \Vert \hat{\theta}_i \Vert_1
 \\ & \leq \frac{R_i \sqrt{\log d}}{\sqrt{T_i}}+\frac{ G R_i}{T_i}+\frac{c_1}{\sqrt{T_i}}\Vert \hat{\theta}_i-y_0\Vert_2^2+\frac{\rho_x}{T_i} B_{\phi_x}(\hat{\theta}_i,\theta_0) 
+\frac{1}{T_i}\sum_{k=1}^{T_i} \langle e_k,\hat{\theta}_i-\theta_k \rangle.
\end{align*}
We have $\theta_0=y_0=\tilde{\theta}_i$ and $z_0=0$. Moreover, $B_{\phi_x}(\theta, \theta_k)=B_{\phi'_x}(\theta, \theta_k)-\frac{1}{\rho_x}B_f(\theta,\theta_k)$. Therefore,
\begin{align*}
&f(\bar{\theta}(T_i))-f(\hat{\theta}_i)+\lambda_i \Vert \bar{y}(T_i) \Vert_1-\lambda_i \Vert \hat{\theta}_i \Vert_1  \\ \nonumber
&\leq  \frac{R_i \sqrt{\log d}}{\sqrt{T_i}}+\frac{ G R_i}{T_i} +\! \frac{c_1}{\sqrt{T_i}} \Vert \hat{\theta}_i-\tilde{\theta}_i \Vert_2^2\!+\!\frac{\rho_x}{T_i}  \lbrace B_{\phi'_x}(\hat{\theta}_i,\tilde{\theta}_i)\!-\!B_f(\hat{\theta}_i,\tilde{\theta}_i) \rbrace\!+\!\sum_{k=1}^{T_i}\! \langle e_k,\hat{\theta}_i-\theta_k \rangle   \\
& \leq  \frac{R_i \sqrt{\log d}}{\sqrt{T_i}}+\frac{ G R_i}{T_i}+ \frac{\sqrt{\log d}}{R_i\sqrt{T_i}} \Vert \hat{\theta}_i-\tilde{\theta}_i \Vert_2^2+\frac{\rho_x}{T_i}  B_{\phi'_x}(\hat{\theta}_i,\tilde{\theta}_i) +\sum_{k=1}^{T_i} \langle e_k,\hat{\theta}_i-\theta_k \rangle .
\end{align*}
We note that $ \rho_xB_{\phi'_x}( \hat{\theta}_i,\tilde{\theta}_i)=\frac{\rho_x}{2}\Vert  \hat{\theta}_i-\tilde{\theta}_i \Vert_2^2  $.\\
%\noindent The term $\frac{\rho_x} {2T_i}{\Vert \hat{\theta}_i -q_i \Vert_2^2}$ is equal to $\frac{\rho_x} {2T_i} B_{\phi'_x}(\hat{\theta}_i,\theta_0)$, where at each epoch $\theta_0=q_i$.\\
Considering the $\ell_2$ terms,
remember that for any vector $x$, if $s>r>0$ then $\Vert x \Vert_s \leq \Vert x \Vert_r$. Therefore,
\begin{align*}
\frac{\sqrt{\log d}}{R_i}\Vert \hat{\theta}_i -\tilde{\theta}_i\Vert_2^2 &\leq \frac{\sqrt{\log d}}{R_i}\Vert \hat{\theta}_i -\tilde{\theta}_i \Vert_1^2 \leq \frac{\sqrt{\log d}}{R_i}R_i^2=R_i \sqrt{\log d}.
\end{align*}
\end{proof}
 \subsection{Proof of Proposition~\ref{thm:prop1}: Inequality \eqref{eq:prop1aval}} \label{sec:ek}
Note the shorthand $e_k=\hat{g}_k-\nabla f(\theta_k)$, where $\hat{g}_k$ stands for empirically calculated subgradient of $f(\theta_k)$.

From Lemma \ref{thm:epoch_avg_opt}, we have that
\begin{align*}
&f(\bar{\theta}(T_i))-f(\hat{\theta}_i)+\lambda_i \Vert \bar{y}(T_i) \Vert_1-\lambda_i \Vert \hat{\theta}_i \Vert_1 \\&\leq
\frac{ R_i \sqrt{ \log d}}{\sqrt{T_i}}+\frac{ G R_i}{T_i}+\frac{\rho_x R_i^2}{T_i}+\frac{\sum_{k=1}^{T_i} \langle e_k,\hat{\theta}_i-\theta_k \rangle}{T_i}.
\end{align*}
Using Lemma 7 from \citep{AgarwalNW12}, we have that
\begin{align*}
 &f(\bar{\theta}(T_i))-f(\hat{\theta}_i)+\lambda_i \Vert \bar{y}(T_i) \Vert_1-\lambda_i \Vert \hat{\theta}_i \Vert_1 \\& \leq
 \frac{ R_i \sqrt{ \log d}}{\sqrt{T_i}}+\frac{ G R_i}{T_i}+\frac{\rho_x R_i^2}{{T_i}}+\frac{R_i \sigma_iw_i}{\sqrt{T_i}}\\ &=
\frac{ R_i \sqrt{ \log d}}{\sqrt{T_i}}+\frac{ G R_i}{T_i}+\frac{ \rho_xR_i^2}{{T_i}}+\frac{R_i \sigma_i}{\sqrt{T_i}} \sqrt{12 \log (3/\delta_i)}.
\end{align*}
with probability at least $1-\delta_i$. In the last equality we use $\delta_i=3\exp(-w_i^2/12)$.

\subsection{Proof of Lemma \ref{lemma:A1}}
Proof follows the same approach as Lemma 1 \citep{AgarwalNW12}. Note that since we assume exact sparsity the term $\Vert \theta^*_{S^c} \Vert_1$ is zero for our case and is thus eliminated. Needless to say, it is an straightforward generalization to consider approximate sparsity from this point.

\subsection{Proof of Lemma \ref{lemma:A2}}
Using LSC assumption and the fact that $\hat{\theta}_i$ minimizes $f(\cdot)+\Vert \cdot\Vert_1$, we have that
\begin{align*}
\frac{\gamma}{2}\Vert \hat{\Delta}(T_i)\Vert_2^2 &\leq f(\bar{\theta}(T_i))-f(\hat{\theta}(T_i))+\lambda_i(\Vert \bar{y}(T_i)\Vert_1-\Vert \hat{\theta}_i\Vert_1) \\ &\leq
\frac{ R_i \sqrt{ \log d}}{\sqrt{T_i}}+\frac{ G R_i}{T_i}+\frac{\rho_x R_i^2}{{T_i}}+\frac{R_i \sigma_i}{\sqrt{T_i}} \sqrt{12 \log \frac{3}{\delta_i}},
\end{align*}
with probability at least $1-\delta_i$.

\subsection{Proof of Proposition~\ref{thm:prop1}: Inequality \eqref{eq:prop1dovom}}
Throughout the proof, let $\Delta^*(T_i)=\bar{\theta}_i-\theta^*$ and
 $\hat{\Delta}(T_i)=\bar{\theta}_i-\hat{\theta}_i$, we have that $\Delta^*(T_i)-\hat{\Delta}(T_i)=\hat{\theta}_i-\theta^*$. Now we want to convert the error bound in \eqref{eq:prop1aval} from function values into $\ell_1$ and $\ell_2$-norm bounds by exploiting the sparsity of $\theta^*$. Since the error bound in \eqref{eq:prop1aval} holds for the minimizer $\hat{\theta}_i$, it also holds for any other feasible vector. In particluar, applying it to $\theta^*$ leads to,
\begin{align*}
& f(\bar{\theta}(T_i))-f({\theta}^*)+\lambda_i \Vert \bar{y}(T_i) \Vert_1-\lambda_i \Vert {\theta}^* \Vert_1 \\&\leq
\frac{ R_i \sqrt{ \log d}}{\sqrt{T_i}}+ \frac{G R_i}{T_i}+\frac{\rho_x R_i^2}{{T_i}}+\frac{R_i \sigma_i}{\sqrt{T_i}} \sqrt{12 \log \frac{3}{\delta_i}},
\end{align*}
with probability at least $1-\delta_i$.

For the next step, we find a lower bound on the left hand side of this inequality.
\begin{align*}
 &f(\bar{\theta}(T_i))-f(\theta^*)+\lambda_i \Vert \bar{y}(T_i) \Vert_1-\lambda_i \Vert {\theta}^* \Vert_1 \geq \\
 &f(\theta^*)-f(\theta^*)+\lambda_i \Vert \bar{y}(T_i) \Vert_1-\lambda_i \Vert {\theta}^* \Vert_1 = \\
 &\lambda_i \Vert \bar{y}(T_i) \Vert_1-\lambda_i \Vert {\theta}^* \Vert_1,
\end{align*}
where the first inequality results from the fact that $\theta^*$ optimizes $f(\theta)$. Thus,
\begin{align*}
\Vert \bar{y}(T_i) \Vert_1&\leq \Vert {\theta}^* \Vert_1+\frac{ R_i \sqrt{ \log d}}{\lambda_i\sqrt{T_i}}+\frac{ G R_i}{\lambda_iT_i} +\frac{\rho_x R_i^2}{\lambda_i{T_i}}+\frac{R_i \sigma_i}{\lambda_i\sqrt{T_i}} \sqrt{12 \log \frac{3}{\delta_i}}.
\end{align*}
Now we need a bound on $\Vert \bar{\theta}(T_i)-\bar{y}(T_i) \Vert_1$, we have
\begin{align*}
\Vert \bar{\theta}(T_i)-\bar{y}(T_i) \Vert_1&= \Vert \frac{1}{T_i}\sum_{k=0}^{T_i -1} (\theta_k-y_k) \Vert_1
\\&=\Vert \frac{1}{\tau T_i}\sum_{k=0}^{T_i -1} (z_{k+1}-z_k) \Vert_1 \\&= \frac{1}{\tau T_i}\Vert z_{T_i}\Vert_1
\\&\leq \frac{G+2\rho_0 R_i}{T_i\tau}=\frac{G R_i}{T_i \sqrt{T_i}\sqrt{\log d}}+\frac{R_i}{T_i}.
\end{align*}
%By theory of ADMM~\citep{boyd2011distributed}, $\exists B,\st \forall z, \Vert z \Vert_1 \leq B$.\\
By triangle inequality
\begin{align*}\Vert \bar{\theta}(T_i)\Vert_1-\Vert\bar{y}(T_i)\Vert_1 \leq \Vert \bar{\theta}(T_i)-\bar{y}(T_i)\Vert_1,\end{align*} Hence, after ignoring the dominated terms,
\begin{align*}
\Vert \bar{\theta}(T_i) \Vert_1 \leq &\Vert {\theta}^* \Vert_1+\frac{ R_i \sqrt{ \log d}}{\lambda_i\sqrt{T_i}}+\frac{ G R_i}{\lambda_i T_i}+\frac{\rho_x R_i^2}{\lambda_i{T_i}}+\frac{R_i \sigma_i}{\lambda_i\sqrt{T_i}} \sqrt{12 \log (3/\delta_i)}+\frac{R_i}{T_i}.
\end{align*}
 By Lemma 6 in~\cite{AgarwalNW12},
 \begin{align*}
 \Vert \Delta^*(T_i)_{S^c}\Vert_1 \leq &\Vert \Delta^*(T_i)_{S}\Vert_1+\frac{ R_i \sqrt{ \log d}}{\lambda_i\sqrt{T_i}}+\frac{ G R_i}{\lambda_iT_i}+\frac{\rho_x R_i^2}{\lambda_i{T_i}}+\frac{R_i \sigma_i}{\lambda_i\sqrt{T_i}} \sqrt{12 \log (3/\delta_i)}+\frac{R_i}{T_i}.
\end{align*}
with probability at least $1-3\exp(-w_i^2/12)$. 

We have $\Delta^*(T_i)-\hat{\Delta}(T_i)=\hat{\theta}_i-\theta^*$.
Therefore,
 \begin{align*}
 &\Vert \hat{\theta}_i-\theta^* \Vert_1=\\&\Vert \Delta^*_S(T_i)-\hat{\Delta}_S(T_i) \Vert_1+\Vert \Delta^*_{S^c}(T_i)-\hat{\Delta}_{S^c}(T_i) \Vert_1 \geq\\
 & \lbrace \Vert \Delta^*_S(T_i)\Vert_1-\Vert \hat{\Delta}_S(T_i) \Vert_1\rbrace-\lbrace \Vert \Delta^*_{S^c}(T_i)\Vert_1-\Vert \hat{\Delta}_{S^c}(T_i) \Vert_1\rbrace.
 \end{align*}
Consequently,
 \begin{align*}
 &\Vert \hat{\Delta}_{S^c}(T_i) \Vert_1-\Vert \hat{\Delta}_{S}(T_i) \Vert_1 \leq
 \Vert \Delta^*_{S^c}(T_i)\Vert_1-\Vert \Delta^*_{S}(T_i)\Vert_1+\Vert \hat{\theta}_i-\theta^* \Vert_1.
 \end{align*}
 Using Equation~\eqref{eq:lemma1a_dovom}, we get
 \begin{align*}
 \Vert \hat{\Delta}_{S^c}(T_i) \Vert_1  \leq & \Vert \hat{\Delta}_{S}(T_i) \Vert_1+\frac{8s\lambda_i}{\gamma}+
  \frac{ R_i \sqrt{ \log d}}{\lambda_i\sqrt{T_i}}+\frac{ G R_i}{\lambda_iT_i}+\frac{\rho_x R_i^2}{\lambda_i{T_i}}+\frac{R_i \sigma_i}{\lambda_i\sqrt{T_i}} \sqrt{12 \log (3/\delta_i)}+\frac{R_i}{T_i}.
 \end{align*}
Hence, further use of the inequality $\Vert \hat{\Delta}_{S}(T_i) \Vert_1 \leq \sqrt{s}\Vert \hat{\Delta}(T_i) \Vert_2$ allows us to conclude that there exists a universal constant $c$ such that
\begin{align}
\label{eq:A63}
\Vert \hat{\Delta}(T_i) \Vert_1^2 \leq  4s\Vert \hat{\Delta}(T_i) \Vert_2^2\!+\!c\!\left[ \frac{s^2\lambda_i^2}{\gamma^2}\!+\!
  \frac{R_i^2{\log d}}{\lambda_i^2 T_i}\!+\!\frac{G^2R_i^2}{\lambda_i^2T_i^2}\!+\!\frac{\rho_x^2R_i^4}{\lambda_i^2 T_i^2}\!+\!\frac{12R_i^2 \sigma_i^2 \log(\frac{3}{\delta_i})}{T_i \lambda_i^2}+\!\frac{ R_i^2}{T_i^2} \right],
\end{align}
with probability at least $1-\delta_i$.

Optimizing the above bound with choice of $\lambda_i$ gives us \eqref{eq:Ti}.
From here on all equations hold with probability at least $1-\delta_i$, we have
\begin{align*}
\Vert \hat{\Delta}(T_i) \Vert_1^2 &\leq \frac{8s}{\gamma} \left[ f(\bar{\theta}(T_i))-f(\hat{\theta}(T_i))+ \lambda_i(\Vert \bar{Y}(T_i)\Vert_1-\Vert \hat{\theta}_i\Vert_1) \right]\\&\!+\!\frac{2cs}{{\gamma\sqrt{T_i}}} \!\left[\! R_i\sqrt{\log d}+\frac{ G R_i}{\sqrt{T_i}}+\frac{\rho_x R_i^2}{\sqrt{T_i}}+ R_i \sigma_i \sqrt{12 \log (\frac{3}{\delta_i})}\right]+\frac{ R_i^2}{T_i^2}.
\end{align*}
Thus, for some other $c$, we have that
\begin{align}
\label{eq:A68}
\Vert \hat{\Delta}(T_i) \Vert_1^2 \leq 
c\frac{s}{\gamma}
 \left[
\frac{R_i \sqrt{ \log d}}{\sqrt{T_i}}+\frac{G R_i}{T_i}+\frac{\rho_x R_i^2}{T_i}+\frac{R_i \sigma_i}{\sqrt{T_i}} \sqrt{12 \log (\frac{3}{\delta_i})}\right]+\frac{R_i^2}{T_i^2}.
\end{align}
Combining the above inequality with error bound \eqref{eq:lemma1a_dovom} for $\hat{\theta}_i$ and using triangle inequality leads to
\begin{align*}
\Vert {\Delta^*}(T_i) \Vert_1^2 &\leq 2\Vert \hat{\Delta}(T_i) \Vert_1^2 +2\Vert \theta^*- \hat{\theta_i}\Vert_1^2 \\ & \leq 2\Vert \hat{\Delta}(T_i) \Vert_1^2+\frac{64}{{\gamma}^2}s^2\lambda_i^2 \\
& \leq c'\frac{s}{{\gamma}} \!\left[\!
\frac{R_i \sqrt{ \log d}}{\sqrt{T_i}}+\frac{ G R_i}{T_i}+\frac{\rho_x R_i^2}{{T_i}}+\frac{R_i \sigma_i}{\sqrt{T_i}} \sqrt{12 \log \frac{3}{\delta_i}}\!\right]+\frac{ R_i^2}{T_i^2}.
\end{align*}
Finally, in order to use $\bar{\theta}(T_i)$ as the next prox center $\tilde{\theta}_{i+1}$, we would also like to control the error $\Vert \bar{\theta}(T_i)-\hat{\theta}_{i+1}\Vert_1^2$. Since $\lambda_{i+1}\leq \lambda_i$ by assumption, we obtain the same form of error bound as in \eqref{eq:A68}. We want to run the epoch till all these error terms drop to $R_{i+1}^2 := R_i^2/2$. Therefore, we set the epoch length $T_i$ to ensure that.
All above conditions are met if we choose the epoch length
\begin{align}
\label{eq:Ticomplete}
T_i=C \left[\frac{s^2}{{\gamma^2}}\left[ \frac{\log d+12\sigma_i^2\log(3/\delta_i)}{R_i^2} \right]+\frac{s G}{\gamma R_i}+\frac{s}{\gamma}\rho_x\right],
\end{align}
for a suitably large universal constant $C$. Note that since we consider the dominating terms in the final bound, the last two terms can be ignored. By design of $T_i$, we have that
\begin{align*}
\Vert {\Delta^*}(T_i) \Vert_1^2 \leq \frac{c'}{\sqrt{C}}R_i^2,
\end{align*}
which completes this proof.
\subsection{Proof of Guarantees with Fixed Epoch Length, Sparse Case}\label{sec:const_sparse_proof}
This is a special case of Theorem~\ref{thm:sparse_optimal} (Appendix).
The key difference between this case and optimal epoch length setting of Theorem~\ref{thm:sparse_optimal} is that in the latter we guaranteed error halving by the end of each epoch whereas with fixed epoch length that statement may not be possible after the number of epochs becomes large enough. Therefore, we need to show that in such case the error does not increase much to invalidate our analysis. Let $k^*$ be the epoch number such that error halving holds true until then. Next we demonstrate that error does not increase much for $k>k^*$.

Given a fixed epoch length $T_0=\mathcal{O}(\log d)$, we define
\begin{align}
\label{eq:k_star}
k^*:=\sup \left\lbrace i:2^{j/2+1} \leq \frac{cR_1\gamma}{s}\sqrt{\frac{T_0}{\log d+\sigma_i^2w^2}} ~~\text{for all epochs}~~ j \leq i \right\rbrace,
\end{align}
where $w=\log(6/\delta)$.

First we show that if we run REASON 1 with fixed epoch length $T_0$ it has error halving behavior for the first $k^*$ epochs.
\begin{lemma}
\label{lemma:A5}
For $T_0=\mathcal{O}(\log d)$ and $k^*$ as in~\eqref{eq:k_star}, we have
\begin{align*}
\Vert \tilde{\theta}_k-\theta^*\Vert_1 \leq R_k ~~\text{and} ~~\Vert \tilde{\theta}_k-\bar{\theta}_k\Vert_1 \leq R_k ~~\text{for all}~~ 1\leq k \leq k^*+1.
\end{align*}
with probability at least $1-3k\exp(-w^2/12)$. Under the same conditions, there exists a universal constant $c$ such that
\begin{align*}
\Vert \tilde{\theta}_k-\theta^*\Vert_2 \leq c\frac{R_k}{\sqrt{s}} ~~\text{and} ~~\Vert \tilde{\theta}_k-\bar{\theta}_k\Vert_2 \leq c\frac{R_k}{\sqrt{s}} ~~\text{for all}~~ 2\leq k \leq k^*+1.
\end{align*}
\end{lemma}
Next, we analyze the behavior of REASON 1 after the first $k^*$ epochs. Since we cannot guarantee error halving, we can also not guarantee that $\theta^*$ remains feasible at later epochs. We use Lemma~\ref{lemma:A3} to control the error after the first $k^*$ epochs.
\begin{lemma}
\label{lemma:A3}
Suppose that Assumptions $A1-A3$ in the main text are satisfied at epochs $i=1,2,\dots$. Assume that at some epoch $k$, the epoch center $\tilde{\theta}_k$ satisfies the bound $\Vert \tilde{\theta}_k-\theta^*\Vert_2 \leq c_1 R_k/\sqrt{s}$ and that for all epochs $j \geq k$, the epoch lengths satisfy the bounds
\begin{align*}
\frac{s}{\gamma}\sqrt{\frac{\log d+\sigma_i^2w_i^2}{T_j}} \leq \frac{R_k}{2}~~ \text{and}~~\frac{\log d}{T_i} \leq c_2.
\end{align*}
Then for all epochs $j \geq k$, we have the error bound $\Vert q_j-\theta^* \Vert_2^2 \leq c_2 \frac{R_k^2}{s} $ with probability at least $1-3\sum_{i=k+1}^j \exp(-w_i^2/12)$.
\end{lemma}
In order to check the condition on epoch length in Lemma~\ref{lemma:A3}, we notice that with $k^*$ as in~\eqref{eq:k_star}, we have
\begin{align*}
c\frac{s}{\gamma}\sqrt{\frac{\log d+\sigma_i^2w^2}{T_0}} \leq R_1 2^{-k^*/2-1}=\frac{R_{k^*+1}}{2}.
\end{align*} 
Since we assume that constants $\sigma_k$ are decreasing in $k$, the inequality also holds for $k \geq k^*+1$, therefore Lemma~\ref{lemma:A3} applies in this setting.

The setting of epoch length in Theorem~\ref{thm:sparse_const_epoch} ensures that the total number of epochs we perform is 
\begin{align*}
k_0=\log \left(\frac{R_1 \gamma}{s}\sqrt{\frac{T}{\log d+\sigma^2w^2}}\right).
\end{align*}
Now we have two possibilities. Either $k_0 \leq k^*$ or $k_0 \geq k^*$. In the former, Lemma~\ref{lemma:A5} ensures that the error bound $\Vert \tilde{\theta}_{k_0}-\theta^*\Vert_2^2 \leq c R^2_{k_0}/s$. In the latter case, we use Lemma~\ref{lemma:A3} and get the error bound $c R^2_{k^*}/s$. Substituting values of $k_0$, $k^*$ in these bounds completes the proof.

Proof of Lemma~\ref{lemma:A5} and Lemma~\ref{lemma:A3} follows directly from that of Lemma 5 and Lemma 3 in~\citep{AgarwalNW12}.
\subsection{Proof of Guarantees for Sparse Graphical Model selection Problem} \label{sec:sparse_graph_proof}
Here we prove Corollary 1. According to C.1, in order to prove guarantees, we first need to bound $\Vert z_{k+1}-z_k \Vert_1$ and $\Vert z_k \Vert_\infty$. According to Equation~\eqref{eq:dual} and considering the imposed $\ell_1$ bound, this is equivalent to bound $\Vert g_{k+1}-g_k \Vert_1$ and $\Vert g_k \Vert_\infty$. The rest of the proof follows on lines of Theorem 1 proof. On the other hand, Lipschitz property requires a bound on $\| g_k\|_1$, which is much more stringent. 

Assuming we are in a close proximity of $\Theta^*$, we can use Taylor approximation to locally approximate $\Theta^{-1}$ by ${\Theta^*}^{-1} $ as in~\citep{Ravikumar&etal:08Arxiv}
\begin{align*}
\Theta^{-1}={\Theta^*}^{-1}-{\Theta^*}^{-1}\Delta{\Theta^*}^{-1}+\mathcal{R}(\Delta),
\end{align*}
where $\Delta=\Theta-\Theta^*$ and $\mathcal{R}(\Delta)$ is the remainder term.
We have
\begin{align*}
\Vert g_{k+1}-g_k \Vert_1 \leq \gennorm {\Gamma^*}_\infty \Vert \Theta_{k+1}-\Theta_k\Vert_1,
\end{align*}
and
\begin{align*}
\Vert g_k \Vert_\infty &\leq \Vert g_k-\mathbb{E}(g_k)\Vert_\infty+\Vert \mathbb{E}(g_k) \Vert_\infty  \\&\leq\Vert e_k\Vert_\infty+\Vert \Sigma^*-\Theta_k^{-1}\Vert_\infty \leq \sigma+\Vert \Gamma^* \Vert_\infty \Vert \Theta_{k+1}-\Theta_k\Vert_1.
\end{align*}
The term $\Vert \Theta_{k+1}-\Theta_k\Vert_1$ is bounded by $2R_i$ by construction. We assume $\gennorm{\Gamma^*}_\infty$ and $\Vert \Gamma^*\Vert_\infty$ are bounded.

 The error $\Delta$ needs to be ``small enough'' for the   $\mathcal{R}(\Delta)$ to be negligible, and we now provide the conditions for this.
By definition, $\mathcal{R}(\Delta)=\sum_{k=2}^{\infty} (-1)^k ({\Theta^*}^{-1}\Delta)^k {\Theta^*}^{-1}$.
Using triangle inequality and sub-multiplicative property for Frobenious norm,
\begin{align*}
\Vert \mathcal{R}(\Delta)\Vert_\mathbb{F} \leq \frac{\Vert {\Theta^*}^{-1}\Vert_\mathbb{F}\Vert\Delta {\Theta^*}^{-1}\Vert_\mathbb{F}^2}{1-\Vert\Delta {\Theta^*}^{-1}\Vert_\mathbb{F}}.
\end{align*}
For $\|\Delta\|_{\mathbb{F}}\leq 2R_i \leq  \frac{0.5}{\Vert {\Theta^*}^{-1}\Vert_\mathbb{F}} $, we get
\begin{align*}
\Vert \mathcal{R}(\Delta)\Vert_\mathbb{F} \leq \Vert {\Theta^*}^{-1}\Vert_\mathbb{F}.
\end{align*}
We assume  $\Vert \Sigma^* \Vert_\mathbb{F}$ is bounded.

Note that $\lbrace R_i \rbrace_{i=1}^{k_T}$ is a decreasing sequence and we only need to bound $R_1$. Therefore, if the variables are closely-related we need to start with a small $R_1$. For weaker correlations, we can start in a bigger ball.
The rest of the proof follows the lines of proof for Theorem~\ref{thm:sparse_optimal}, by replacing $G^2$ by  $\gennorm{\Gamma^*}_\infty  R_i (\sigma+\Vert \Gamma^*\Vert_\infty  R_i)$. Ignoring the higher order terms gives us Corollary 1.

\section{Guarantees for REASON 2} \label{sec:guaranteesR2}
First, we provide guarantees for the theoretical case such that epoch length depends on epoch radius. This provides intuition on how the algorithm is designed. The fixed-epoch algorithm is a special case of this general framework. We first state and prove guarantees for general framework. Next, we leverage these results to prove Theorem~\ref{thm:sparse_const_epoch}.
Let the design parameters be set as
\begin{align}
\label{parameter}
T_i& \simeq C \left[ (s+r+\frac{s+r}{\gamma})^2\left( \frac{\log p+\beta^2(p)\sigma_i^2\log(6/\delta_i)+}{R_i^2}\right) +(s+r+\frac{s+r}{\gamma})\left(\frac{G}{R_i}+\rho_x\right)\right],\\ \nonumber
\lambda_i^2&=\frac{\gamma}{(s+r)\sqrt{T_i}}  \sqrt{(R_i^2 +\tilde{R}_i^2){\log p}+\frac{ G^2(R_i^2 +\tilde{R}_i^2)}{T_i}+{\beta^2(p)}(R_i^2+\tilde{R}_i^2)\sigma_i^2\log\frac{3}{\delta_i}} \\ \nonumber
&\quad
+\frac{\rho_x (R_i^2+\tilde{R}_i^2)}{T_i}+\frac{\alpha^2}{p^2}+\frac{\beta^2(p)\sigma^2}{T_i}\left(\log p+\log\frac{1}{\delta_i}\right),\\ \nonumber
\\ \nonumber
\mu_i^2&=c_\mu\lambda_i^2, \quad
\rho  \propto \sqrt{\frac{T_i \log p }{R_i^2+\tilde{R}_i^2}},\quad \rho_x >0, \quad \tau=\rho.
\end{align}
\begin{theorem} %equivalent to theorem 4 Agarwal
\label{thm:latent_optimal}
Under assumptions $A2-A6$ and parameter settings as in~\eqref{parameter}, there exists a constant $c_0 >0$ such that REASON 2 satisfies the following for all $T>k_T$,
\begin{align*}
&\Vert \bar{S}(T)-S^* \Vert_\mathbb{F}^2 + \Vert \bar{L}(T)-L^* \Vert_\mathbb{F}^2 \leq \\&
\frac{c_0(s+r)}{T}\left[\log p+\beta^2(p)\sigma^2\left(w^2+\log k_T\right)\right]+\left( 1+\frac{s+r}{\gamma^2 p}\right)\frac{\alpha^2}{p}.
\end{align*}
with probability at least $1-6\exp(-w^2/12)$ and
\begin{align*}
k_T\simeq-\log \left( \frac{(s+r)^2}{\gamma^2 R_1^2 T} \left[\log p+\beta^2(p)\sigma^2w^2\right] \right).
\end{align*}
\end{theorem}
\noindent For Proof outline and detailed proof of Theorem~\ref{thm:latent_optimal} see Appendix~\ref{sec:outline2} and~\ref{sec:theorem2proof} respectively.

\subsection{Proof outline for Theorem \ref{thm:latent_optimal}} \label{sec:outline2}
The foundation block for this proof is Proposition \ref{thm:prop1_latent}.
\begin{proposition}
\label{thm:prop1_latent}
Suppose $f$ satisfies Assumptions $A1- A6$ with parameters $\gamma$ and $\sigma_i$ respectively and assume that $\Vert S^*-\tilde{S}_i \Vert_1^2 \leq R_i^2$, $\Vert L^*-\tilde{L}_i \Vert_1^2 \leq \tilde{R}_i^2$. %explain f,assp,theta^*
We apply the updates in REASON 2 with parameters as in \eqref{parameter}. Then, there exists a universal constant $c$ such that for any radius $R_i,\tilde{R}_i$, $\tilde{R}_i=c_r R_i, 0 \leq c_r \leq 1$,
\begin{align}
\sublabon{equation}
\label{eq:prop1aval_latent}
&f(\bar{M}(T_i))+\lambda_i \phi(\bar{W}(T_i))-f(\hat{M}_i)-\lambda_i \phi(\hat{W}(T_i)) \\ \nonumber
& \leq
\sqrt{\frac{R_i^2 +\tilde{R}_i^2}{T_i}}\sqrt{\log p}+\frac{R_i^2 +\tilde{R}_i^2}{T_i}\rho_x+\frac{G\sqrt{R_i^2 +\tilde{R}_i^2}}{T_i}\\ \nonumber&~~+\frac{{\beta(p)}(R_i+\tilde{R}_i)\sigma_i\sqrt{12\log \frac{3}{\delta_i}}}{\sqrt{T_i}}+\frac{{\beta(p)}G (R_i+\tilde{R}_i)\sigma_i\sqrt{12\log \frac{3}{\delta_i}}}{T_i\sqrt{\log p}},\\ %\nonumber
\label{eq:prop1dovom_latent}
&\Vert {\bar{S}(T_i)-S^*} \Vert_1^2 \leq \frac{c'}{\sqrt{C}}R_i^2+c(s+r+\frac{(s+r)^2}{p\gamma^2})\frac{\alpha^2}{p},\\ \nonumber
&\Vert {\bar{L}(T_i)-L^*} \Vert_*^2 \leq \frac{c'}{\sqrt{C}}\frac{1}{1+\gamma}R_i^2+c\frac{(s+r)^2}{p\gamma^2}\frac{\alpha^2}{p}.
\end{align}
\sublaboff{equation}
where both bounds are valid with probability at least $1-\delta_i$.
\end{proposition}
\noindent In order to prove Proposition \ref{thm:prop1_latent}, we need two more lemmas.

To move forward, we use the following notations: $\Delta(T_i)=\hat{S}_i-S^*+\hat{L}_i-L^*$, $\Delta^*(T_i)=\bar{S}(T_i)-S^*+\bar{L}(T_i)-L^*$ and
 $\hat{\Delta}(T_i)=\bar{S}_i-\hat{S}_i+\bar{L}_i-\hat{L}_i$. In addition $\Delta_S(T_i)=\hat{S}_i-S^*$, with alike notations for $\Delta_L(T_i)$. For on and off support part of $\Delta(T_i)$, we use $(\Delta(T_i))_\supp$ and $(\Delta(T_i))_{\supp^c}$.
 \begin{lemma}
 \label{lemma:A1_latent}
At epoch $i$ assume that $\Vert S^*-\tilde{S} \Vert_1^2 \leq R_i^2$, $\Vert L^*-\tilde{L} \Vert_1^2 \leq \tilde{R}_i^2$. Then the errors $ \Delta_S(T_i), \Delta_L(T_i)$ satisfy the bound
\begin{align*}
\Vert \hat{S}_i-S^*\Vert_\mathbb{F}^2+\Vert \hat{L}_i-L^*\Vert_\mathbb{F}^2 \leq c\lbrace s\frac{\lambda_i^2}{\gamma^2}+r\frac{\mu_i^2}{\gamma^2} \rbrace.
\end{align*}
\end{lemma}
\begin{lemma}
\label{lemma:A2_latent}
Under the conditions of Proposition~\ref{thm:prop1_latent} and with parameter settings \eqref{parameter}, \eqref{parameter}, we have
\begin{align*}
& \Vert \hat{S}_i-\bar{S}(T_i)  \Vert_\mathbb{F}^2+\Vert \hat{L}_i-\bar{L}(T_i)\Vert_\mathbb{F}^2 \\
 &\leq \frac{2}{\gamma}\left(\sqrt{\frac{R_i^2 +\tilde{R}_i^2}{T_i}}\sqrt{\log p}+\frac{R_i^2 +\tilde{R}_i^2}{T_i}\rho_x+\frac{ G\sqrt{R_i^2 +\tilde{R}_i^2}}{T_i}\right.\\&\quad\left.+\frac{{\beta(p)}(R_i+\tilde{R}_i)\sigma_i\sqrt{12\log(3/\delta_i)}}{\sqrt{T_i}}+\frac{{\beta(p)}G (R_i+\tilde{R}_i)\sigma_i\sqrt{12\log(3/\delta_i)}}{T-i \sqrt{\log p}}\right)+(\frac{2\alpha}{\sqrt{p}}+\frac{{p}}{\tau T_i})^2,
\end{align*}
with probability at least $1-\delta_i$.
 \end{lemma}

%It is shown in \citep{ADMM_latent} that this algorithm satisfies global convergence. However, convergence rate analysis has not been addressed. Here we prove an upper bound for average epoch error and analyze the overall contraction. Equation \eqref{eq:4.4} can be interpreted as a BADMM's $Y$ update step with $\rho_z \neq 0$ whereas this parameter was zero for the no latent variable case.\\}

\section{Proof of Theorem \ref{thm:latent_optimal} } \label{sec:theorem2proof}
The first step is to ensure that  $\Vert S^*-\tilde{S}_i \Vert_1^2 \leq R_i^2$, $\Vert L^*-\tilde{L}_i \Vert_1^2 \leq \tilde{R}_i^2$ holds at each epoch so that Proposition~\ref{thm:prop1_latent} can be applied in a recursive manner. We prove this in the same manner we proved Theorem 1, by induction on the epoch index. By construction, this bound holds at the first epoch. Assume that it holds for epoch $i$. Recall that $T_i$ is defined by \eqref{parameter} where $C \geq 1$ is a constant we can choose. By substituting this $T_i$ in inequality \eqref{eq:prop1dovom_latent}, the simplified bound \eqref{eq:prop1dovom_latent} further yields
\begin{align*}
\Vert {\Delta^*_S}(T_i) \Vert_1^2 \leq \frac{c'}{\sqrt{C}}R_i^2+c(s+r+\frac{(s+r)^2}{p\gamma^2})\frac{\alpha^2}{p},
\end{align*}
Thus, by choosing $C$ sufficiently large, we can ensure that $ \Vert \bar{S}(T_i)-S^*\Vert_1^2 \leq  R_i^2/2:=R_{i+1}^2$.
Consequently, if $S^*$ is feasible at epoch $i$, it stays feasible at epoch $i+1$. Hence, we guaranteed the feasibility of $S^*$ throughout the run of algorithm by induction.
As a result, Lemma \ref{lemma:A1_latent} and \ref{lemma:A2_latent} apply and for $\tilde{R}_i=c_rR_i$, we find that
\begin{align*}
&\Vert {\Delta}^*_S(T_i) \Vert_\mathbb{F}^2 \leq \frac{1}{s+r}R_i^2+(1+\frac{s+r}{\gamma^2p})\frac{2\alpha^2}{p}.
\end{align*}
The bound holds with probability at least $1-3\exp(-w_i^2/12)$. The same is true for $ \Vert {\Delta}^*_L(T_i) \Vert_\mathbb{F}^2$. Recall that $R_i^2=R_1^22^{-(i-1)}$.
Since $w_i^2=w^2+24\log i$, we can apply union bound to simplify the error probability as $1-6\exp(-w^2/12)$. Let $\delta=6\exp(-w^2/12)$, we write the bound in terms of $\delta$, using $w^2=12\log(6/\delta)$.

Next we convert the error bound  from its dependence on the number of epochs $k_T$ to the number of iterations needed to complete $k_T$ epochs, i.e. $T(K)=\sum_{i=1}^k T_i$.
Using the same approach as in proof of Theorem \ref{thm:sparse_optimal}, we get
\begin{align*}
k_T\simeq -\log \frac{(s+r+(s+r)/\gamma)^2}{R_1^2 T}-\log \left[\log p+12\beta^2(p)\sigma^2w^2\right]. 
\end{align*}
As a result
\begin{align*}
\Vert {\Delta}^*_S(T_i) \Vert_\mathbb{F}^2 \leq
\frac{C(s+r)}{T}\left[\log p+\beta^2(p)\sigma^2\left(w^2+\log k_T)\right]\right]+\frac{\alpha^2}{p}.
\end{align*}
For the low-rank part, we proved feasibility in proof of Equation \eqref{eq:prop1dovom_latent}, consequently The same bound holds for $\Vert {\Delta}^*_L(T_i) \Vert_\mathbb{F}^2$.
\subsection{Proofs for Convergence within a Single Epoch for Algorithm \ref{our_method_latent}}
We showed that our method is equivalent to running Bregman ADMM on $M$ and $W=[S;L]$. Consequently, our previous analysis for sparse case holds true for the error bound on sum of loss function and regularizers within a single epoch. With $\rho=c_2\sqrt{T_i},\tau=\rho, c_2=\frac{\sqrt{\log p}}{\sqrt{R_i^2+\tilde{R}_i^2}}$. We use the same approach as in Section~\ref{sec:epochopt1} for bounds on dual variable $Z_k$. Hence,
\begin{align*}
&f(\bar{M}(T_i))+\lambda_i \phi(\bar{W}(T_i))-f(\hat{M}_i)-\lambda_i \phi(\hat{W}(T_i)) \\& \leq
\frac{c_2\Vert A\hat{W}(T_i)-AW_0\Vert_\mathbb{F}^2}{\sqrt{T_i}}+\frac{\rho_x\Vert\hat{M}(T_i)-M_0\Vert_\mathbb{F}^2}{T_i}+\frac{ G R_i}{T_i}+\frac{R_i \sqrt{\log p}}{\sqrt{T_i}}\\ \nonumber
&\quad+\frac{\sum_{k=1}^{T_i} \Tr( E_k,\hat{M}_i-M_k )}{T_i} \\ \nonumber
&\leq \left[\frac{c_2}{\sqrt{T_i}}+\frac{\rho_x}{T_i}\right]\Vert \hat{S}_i-\tilde{S}_i+\hat{L}_i-\tilde{L}_i\Vert_\mathbb{F}^2+\frac{ G R_i}{T_i}+\frac{R_i \sqrt{\log p}}{\sqrt{T_i}}\\&\quad+\frac{\sum_{k=1}^{T_i} \Tr( E_k,\hat{M}_i-M_k )}{T_i}.
\end{align*}
By the constraints enforced in the algorithm, we have
\begin{align*}
&f(\bar{M}(T_i))+\lambda_i \phi(\bar{W}(T_i))-f(\hat{M}_i)-\lambda_i \phi(\hat{W}(T_i)) \\
& \leq \sqrt{\frac{R_i^2 +\tilde{R}_i^2}{T_i}}\sqrt{\log p}+\frac{R_i^2 +\tilde{R}_i^2}{T_i}\rho_x+\frac{ G\sqrt{R_i^2 +\tilde{R}_i^2}}{T_i}+\frac{\sum_{k=1}^{T_i} \Tr( E_k,\hat{M}_i-M_k )}{T_i}.
\end{align*}

\begin{lemma}
The dual variable in REASON 2 is bounded. i.e., 
\begin{align*}
\Vert Z_k \Vert_1 \leq G+2\rho_0 R_i, ~~\text{where}~~ \rho_0:= \rho_x+\rho.
\end{align*}
\end{lemma}
\begin{proof}
The proof follows the same line as in proof of Lemma~\ref{lemma:dualsparse} and replacing $\theta,y$ by $M,W$ where $W=[S;L]$. Hence,

\begin{align*}
\Vert Z_k \Vert_1 \leq G+2\rho_0 R_i, ~~\text{where}~~ \rho_0:= \rho_x+\rho.
\end{align*}
\end{proof}
%We used $\Vert A\Vert_\mathbb{F} \leq \Vert A\Vert_1$ and $\Vert W-A^+Q_i\Vert_1^2 \leq R_i^2$.
\subsection{Proof of Proposition \ref{thm:prop1_latent}: Equation \eqref{eq:prop1aval_latent}}
In this section we bound the term $\frac{\sum_{k=1}^{T_i} \Tr( E_k,\hat{M}_i-M_k )}{T_i}$.
%\begin{align*}
%\Tr(AB)=\sum_j \sum_i a_{ji}b_ij \leq \Vert A \Vert_\infty \sum_j \sum_i b_ij \leq \Vert A \Vert_\infty \Vert B\Vert_1
%\end{align*}
We have
\begin{align*}
M_k-\hat{M}_i=S_k-\hat{S}_i+L_k-\hat{L}_i+(Z_{k+1}-Z_k)/\tau.
\end{align*}
Hence,
\begin{align*}
&[\Tr( E_k,\hat{M}_i-M_k )]^2 \\&\leq [\Vert E_k \Vert_\infty \Vert S_k-\hat{S}_i \Vert_1+\Vert E_k \Vert_2^2 \Vert L_k-\hat{L}_i \Vert_* +
\Vert E_k \Vert_\infty \Vert(Z_{k+1}-Z_k)/\tau \Vert_1]^2 \\
& \leq [2R_i \Vert E_k \Vert_\infty+2\tilde{R}_i \Vert E_k \Vert_2+ (G+2\rho_0 R_i)/\tau \Vert E_k \Vert_\infty]^2
\\ & \leq \Vert E_k \Vert_2 ^2[2R_i+2\tilde{R}_i+(G+2\rho_0 R_i)/\tau]^2.
\end{align*}
Consider the term $\Vert E_k \Vert_2$. Using Assumption A$4$, our previous approach in proof of Equation \eqref{eq:prop1aval}, holds true with addition of a $\beta(p)$ term.
Consequently,
\begin{align*}
&f(\bar{M}(T_i))+\lambda_i \phi(\bar{W}(T_i))-f(\hat{M}_i)-\lambda_i \phi(\hat{W}(T_i)) \\ &\leq
\sqrt{\frac{R_i^2 +\tilde{R}_i^2}{T_i}}\sqrt{\log p}+\frac{R_i^2 +\tilde{R}_i^2}{T_i}\rho_x+\frac{ G\sqrt{R_i^2 +\tilde{R}_i^2}}{T_i}\\
&\quad+\frac{{\beta(p)}(R_i+\tilde{R}_i)\sigma_i\sqrt{12\log(3/\delta_i)}}{\sqrt{T_i}}+\frac{{\beta(p)}G (R_i+\tilde{R}_i)\sigma_i\sqrt{12\log(3/\delta_i)}}{T_i\sqrt{\log p}}.
\end{align*}
with probability at least $1-\delta_i$.

\subsection{Proof of Lemma \ref{lemma:A1_latent}}
We use Lemma 1~\citep{negahban2012unified} for designing $\lambda_i$ and $\mu_i$. This Lemma requires that for optimization problem $\underset{\Theta}{\min}\lbrace {L}(\Theta)+\lambda_i Q(\Theta)\rbrace$, we design the regularizer coefficient $\lambda_i \geq 2Q^*(\nabla {L}(\Theta^*))$, where $L$ is the loss function, $Q$ is the regularizer and $Q^*$ is the dual regularizer. For our case $\Theta$ stands for $[S;L]$.
\begin{align*}
{L}(\Theta)=\frac{1}{n}\sum_{k=1}^n {f}_k(\Theta,x),
\end{align*}
and
\begin{align*}
Q^*(\nabla {L}(\Theta^*))&=Q^*\left[\mathbb{E}(\nabla f(\Theta^*)+\frac{1}{n}\sum_{k=1}^n \lbrace \nabla f_k(\Theta^*))-\mathbb{E}(\nabla f(\Theta^*)) \rbrace\right]\\&=Q^*(\frac{1}{n}\sum_{k=1}^n E_k),
\end{align*}
where $E_k=g_k-\mathbb{E}(g_k)$ is the error in gradient estimation as defined earlier.
\\
Using Theorem 1~\citep{agarwal2012noisy} in this case, if we design
\begin{align}
\label{eq:reg_cond}
\lambda_i \geq 4\left\Vert \frac{1}{n}\sum_{k=1}^n E_k \right\Vert_\infty+\frac{4\gamma \alpha}{p}\quad \text{and} \quad
\mu_i \geq 4\left\Vert \frac{1}{n}\sum_{k=1}^n E_k \right\Vert_2,
\end{align}
%We design $R_1$ such that $T_i >p$ and our regularizer parameters satisfy
%\begin{align}
%\label{eq:reg_cond}
%\lambda_i \geq 32\Tr(\Sigma^*)\sqrt{\frac{\log p }{T_i}}+4\frac{\alpha}{T_i}, \quad \text{and} \quad \mu_i \geq 16\Vert \Sigma^* \Vert_2 \sqrt{\frac{p}{T_i}}.
%\end{align}
%Consequently, we can use Theorem 1 in\citep{agarwal2012noisy} to get
then we have
\begin{align}
\label{eq:noisybatch}
\Vert \hat{S}_i-S^*\Vert_{\mathbb{F}}^2+\Vert \hat{L}_i-L^*\Vert_\mathbb{F}^2 \leq c\lbrace s\frac{\lambda_i^2}{\gamma^2}+r\frac{\mu_i^2}{\gamma^2} \rbrace.
\end{align}
%It is easy to verify that conditions in~\eqref{eq:reg_cond} hold in our setting.
\begin{lemma}
\label{lemma:lambda2}
Assume $X \in \mathbb{R}^{p\times p}$. If $\Vert X \Vert_2 \leq B$ almost surely then with probability at least $1-\delta$ we have
\begin{align*}
\left\Vert \frac{1}{n} \sum_{k=1}^n X_k - \mathbb{E}(X_k)\right\Vert_2 \leq \frac{6B}{\sqrt{n}}\left( \sqrt{\log p}+\sqrt{\log \frac{1}{\delta}}\right).
\end{align*}
\end{lemma}
Note that this lemma is matrix Hoeffding bound and provides a loose bound on matrix. Whereas using matrix Bernstein provided tighter results using $\mathbb{E}(E_k E_k^\top)$.
Moreover, since the elementwise max norm $\Vert X \Vert_\infty \leq  \Vert X \Vert_2$, we use the same upper bound for both norms.

By definition $ \mathbb{E}(E_k)=0 $. According to Assumption A4, $\Vert E_k \Vert_2 \leq \beta(p)\sigma$. Thus it suffices to design
\begin{align*}
&\lambda_i \geq  \frac{24\beta(p)\sigma_i}{\sqrt{T_i}}\left( \sqrt{\log p}+\sqrt{\log \frac{1}{\delta_i}}\right)+\frac{4\gamma \alpha}{p} \\
&\quad \text{and} \quad \\
&\mu_i \geq \frac{24\beta(p)\sigma_i}{\sqrt{T_i}}\left( \sqrt{\log p}+\sqrt{\log \frac{1}{\delta_i}}\right).
\end{align*}
Then, we can use Equation~\eqref{eq:noisybatch}.

\subsection{Proof of Lemma \ref{lemma:A2_latent}}
By LSC condition on $X=S+L$
\begin{align*}
&\frac{\gamma}{2} \Vert \hat{S}_i-\bar{S}(T_i)+\hat{L}_i-\bar{L}(T_i)\Vert_\mathbb{F}^2\\ \nonumber & \leq  f(\bar{X}(T_i))+{\lambda_i}\Vert \bar{S}(T_i) \Vert_1+{\mu_i}\Vert \bar{L}(T_i) \Vert_*-f(\hat{X}_i)-{\lambda_i}\Vert \hat{S}(T_i) \Vert_1-{\mu_i}\Vert \hat{L}(T_i) \Vert_*
\end{align*}
We want to use the following upper bound for the above term.
\begin{align*}
&f(\bar{M}(T_i))+\lambda_i \phi(\bar{X}(T_i))-f(\hat{M}_i)-\lambda_i \phi(\hat{X}(T_i))  \leq \\
&\sqrt{\frac{R_i^2 +\tilde{R}_i^2}{T_i}}\sqrt{\log p}+\frac{R_i^2 +\tilde{R}_i^2}{T_i}\rho_x+\frac{ G\sqrt{R_i^2 +\tilde{R}_i^2}}{T_i}\\
&+\frac{{\beta(p)}(R_i+\tilde{R}_i)\sigma_i\sqrt{12\log\frac{3}{\delta_i}}}{\sqrt{T_i}}+\frac{{\beta(p)}G (R_i+\tilde{R}_i)\sigma_i\sqrt{12\log\frac{3}{\delta_i}}}{{T_i}},
\end{align*}
%with probability at least $1-\delta_i$.
$\hat{M}_i=\hat{X}_i$, i.e.,  all the terms are the same except for $f(\bar{M}(T_i)), f(\bar{X}(T_i))$. We have $\bar{M}(T_i)=\bar{X}(T_i)+\frac{Z_T}{\tau T_i}$. This is a bounded and small term $\mathcal{O}(R_i/(T_i\sqrt{T_i}))$. We accept this approximation giving the fact that this is a higher order term compared to $\mathcal{O}(1/\sqrt{T_i})$ . Hence, it will not play a role in the final bound on the convergence rate.
%\phi(\bar{W}(T_i))={\lambda_i}\Vert \bar{S}(T_i) \Vert_1+{\mu_i}\Vert \bar{L}(T_i) \Vert_*
%{\lambda_i}\Vert \hat{S}(T_i) \Vert_1+{\mu_i}\Vert \hat{L}(T_i) \Vert_*
Therefore, 
\begin{align}
\label{eq:LSC}
&\frac{\gamma}{2} \Vert \hat{S}_i-\bar{S}(T_i)+\hat{L}_i-\bar{L}(T_i)\Vert_\mathbb{F}^2\\ \nonumber & \leq  
\sqrt{\frac{R_i^2 +\tilde{R}_i^2}{T_i}}\sqrt{\log p}+\frac{R_i^2 +\tilde{R}_i^2}{T_i}\rho_x+\frac{ G\sqrt{R_i^2 +\tilde{R}_i^2}}{T_i}\\
&\quad+\frac{{\beta(p)}(R_i+\tilde{R}_i)\sigma_i\sqrt{12\log\frac{3}{\delta_i}}}{\sqrt{T_i}}+\frac{{\beta(p)}G (R_i+\tilde{R}_i)\sigma_i\sqrt{12\log\frac{3}{\delta_i}}}{{T_i}\sqrt{\log p}}, \nonumber
\end{align}
with probability at least $1-\delta_i$.

For simplicity, we use
 \begin{align*}
 H_1&=\sqrt{\frac{R_i^2 +\tilde{R}_i^2}{T_i}}\sqrt{\log p}+\frac{R_i^2 +\tilde{R}_i^2}{T_i}\rho_x+\frac{ G\sqrt{R_i^2 +\tilde{R}_i^2}}{T_i}\\
 &\quad+\frac{{\beta(p)}(R_i+\tilde{R}_i)\sigma_i\sqrt{12\log\frac{3}{\delta_i}}}{\sqrt{T_i}}+\frac{{\beta(p)}G(R_i+\tilde{R}_i)\sigma_i\sqrt{12\log\frac{3}{\delta_i}}}{{T_i\sqrt{\log p}}}.
 \end{align*}
We have,
\begin{align*}
-\frac{\gamma}{2}\Tr(\hat{\Delta}_S \hat{\Delta}_L)=\frac{\gamma}{2}\lbrace \Vert \hat{\Delta}_S  \Vert_\mathbb{F}^2+\Vert \hat{\Delta}_L  \Vert_\mathbb{F}^2 \rbrace-\frac{\gamma}{2}\lbrace \Vert \hat{\Delta}_S +\hat{\Delta}_L  \Vert_\mathbb{F}^2 \rbrace,
\end{align*} In addition,
\begin{align*}
\gamma \Vert\Tr(\hat{\Delta}_S(T_i) \hat{\Delta}_L(T_i))| \leq \gamma \Vert \hat{\Delta}_S(T_i)  \Vert_1\Vert \hat{\Delta}_L(T_i)  \Vert_\infty .
\end{align*}
We have,
\begin{align*}
\Vert \hat{\Delta}_L(T_i)  \Vert_\infty &\leq \Vert \hat{L}_i \Vert_\infty+\Vert \bar{L}(T_i) \Vert_\infty \\ %\leq \frac{2\gamma \alpha}{p},
\end{align*}
\begin{align*}
\Vert \bar{L}(T_i) \Vert_\infty &\leq \Vert \bar{Y}(T_i)\Vert_\infty +\Vert \bar{L}(T_i)-\bar{Y}(T_i)\Vert_\infty \\ &\leq \Vert \bar{Y}(T_i)\Vert_\infty +\Vert \frac{\sum_{k=0}^{T_i-1} (L_k-Y_k)}{T_i}\Vert_\infty \\
&=\Vert \bar{Y}(T_i)\Vert_\infty +\Vert \frac{\sum_{k=0}^{T_i-1} (U_k-U_{k+1})}{\tau T_i}\Vert_\infty \\
&=\Vert \bar{Y}(T_i)\Vert_\infty +\Vert \frac{-U_{k+1}}{\tau T_i}\Vert_\infty\\
&\leq \frac{\alpha}{p}+\frac{\sqrt{p}}{\tau T_i}.
\end{align*}
In the last step we incorporated the constraint $\Vert Y\Vert_\infty \leq \frac{\alpha}{p}$, and the fact that $U_0=0$. Moreover, we used
\begin{align*}
\Vert U_{k+1} \Vert_\infty=\Vert \nabla \lbrace \Vert L\Vert_* \rbrace \Vert_\infty \leq \sqrt{\text{rank}(L)} \leq \sqrt{p}.
\end{align*}
Last step is from the analysis of~\citet{Watson1992}.
Therefore,
\begin{align*}
\gamma \Vert\Tr(\hat{\Delta}_S(T_i) \hat{\Delta}_L(T_i))| \leq \gamma(\frac{2\alpha}{p}+\frac{\sqrt{p}}{\tau T_i}) \Vert \hat{\Delta}_S(T_i)  \Vert_1 .
\end{align*}
Consequently,
\begin{align*}
\frac{\gamma}{2} \Vert \hat{\Delta}_S(T_i) +\hat{\Delta}_L(T_i)  \Vert_\mathbb{F}^2 \geq \frac{\gamma}{2}\lbrace \Vert \hat{\Delta}_S(T_i)  \Vert_\mathbb{F}^2+\Vert \hat{\Delta}_L(T_i)  \Vert_\mathbb{F}^2 \rbrace - \frac{\gamma}{2}(\frac{2\alpha}{p}+\frac{\sqrt{p}}{\tau T_i})\Vert \hat{\Delta}_S(T_i) \Vert_1.
\end{align*}
Combining the above equation with \eqref{eq:LSC}, we get
\begin{align*}
&\frac{\gamma}{2}\lbrace \Vert \hat{\Delta}_S(T_i)  \Vert_\mathbb{F}^2+\Vert \hat{\Delta}_L(T_i)  \Vert_\mathbb{F}^2 \rbrace -\frac{\gamma}{2}(\frac{2\alpha}{p}+\frac{\sqrt{p}}{\tau T_i})\Vert \hat{\Delta}_S(T_i) \Vert_1\leq H_1.
\end{align*}
Using  $\Vert S \Vert_1 \leq \sqrt{p} \Vert S \Vert_\mathbb{F}$, %and  $\Vert S \Vert_* \leq \sqrt{p} \Vert A \Vert_\mathbb{F}$, we get
\begin{align*}
& \Vert \hat{\Delta}_S(T_i)  \Vert_\mathbb{F}^2+\Vert \hat{\Delta}_L(T_i) \Vert_\mathbb{F}^2 \\ &\leq \frac{2}{\gamma}\lbrace\sqrt{\frac{R_i^2 +\tilde{R}_i^2}{T_i}}\sqrt{\log p}+\frac{R_i^2 +\tilde{R}_i^2}{T_i}\rho_x+\frac{ G\sqrt{R_i^2 +\tilde{R}_i^2}}{T_i}\\
&\quad+\frac{{\beta(p)}(R_i+\tilde{R}_i)\sigma_i\sqrt{12\log(3/\delta_i)}}{\sqrt{T_i}}\\
&\quad+\frac{{\beta(p)}G (R_i+\tilde{R}_i)\sigma_i\sqrt{12\log(3/\delta_i)}}{T_i \sqrt{\log p}}\rbrace+(\frac{2\alpha}{\sqrt{p}}+\frac{{p}}{\tau T_i})^2,
 \end{align*}
with probability at least $1-\delta_i$.
\subsection{Proof of Proposition \ref{thm:prop1_latent}: Equation \eqref{eq:prop1dovom_latent}}
%Throughout the proof, let $\Delta^*(T_i)=\bar{M}_i-M^*$ and
% $\hat{\Delta}(T_i)=\bar{M}_i-\hat{M}_i$, we have that $\Delta^*(T_i)-\hat{M}(T_i)=\hat{M}_i-M^*$.
 Now we want to convert the error bound in \eqref{eq:prop1aval_latent} from function values into vectorized $\ell_1$ and Frobenius-norm bounds.
 % by exploiting the sparsity of $M^*$.
 Since the error bound in~\eqref{eq:prop1aval_latent} holds for the minimizer $\hat{M}_i$, it also holds for any other feasible matrix. In particular, applying it to $M^*$ leads to,
\begin{align*}
 & f(\bar{M}(T_i))-f({M}^*)+\lambda_i \phi(\bar{W}(T_i))-\lambda_i \phi(W^*) \\ &\leq
\sqrt{\frac{R_i^2 +\tilde{R}_i^2}{T_i}}\sqrt{\log p}+\frac{R_i^2 +\tilde{R}_i^2}{T_i}\rho_x+\frac{ G\sqrt{R_i^2 +\tilde{R}_i^2}}{T_i}\\
&\quad+\frac{{\beta(p)}(R_i+\tilde{R}_i)\sigma_i\sqrt{12\log(3/\delta_i)}}{\sqrt{T_i}}+\frac{{\beta(p)}G (R_i+\tilde{R}_i)\sigma_i\sqrt{12\log(3/\delta_i)}}{T_i \sqrt{\log p}},
\end{align*}
with probability at least $1-\delta_i$.

For the next step, we find a lower bound on the left hand side of this inequality.
\begin{align*}
 &f(\bar{M}(T_i))-f(M^*)+\lambda_i \phi(\bar{W}(T_i))-\lambda_i \phi(W^*)  \geq \\
 &f(M^*)-f(M^*)+\lambda_i \phi(\bar{W}(T_i))-\lambda_i \phi(W^*)  = \\
 & \lambda_i \phi(\bar{W}(T_i))-\lambda_i \phi(W^*) ,
\end{align*}
where the first inequality results from the fact that $M^*$ optimizes $M$.

From here onward all equations hold with probability at least $1-\delta_i$. We have
\begin{align} \label{eq:p1b1}
\phi(\bar{W}(T_i))- \phi(W^*)  \leq {H_1}/\lambda_i.
\end{align}
i.e.
\begin{align*}
\Vert \bar{S}(T_i)\Vert_1+\frac{\mu_i}{\lambda_i} \Vert \bar{L}(T_i) \Vert_* \leq \Vert S^*\Vert_1+\frac{\mu_i}{\lambda_i} \Vert L^* \Vert_*+H_1/\lambda_i
\end{align*}
%i.e.
%\begin{align*}
%\Vert \bar{S}(T_i)\Vert_1 \leq \Vert S^*\Vert_1+\frac{\mu_i}{\lambda_i}\Vert L^* \Vert_*-\frac{\mu_i}{\lambda_i}\Vert \bar{L}(T_i) \Vert_*+H_1/\lambda_i %\leq \Vert S^*\Vert_1+\frac{\mu_i}{\lambda_i} \Vert L^* - \bar{L}(T_i) \Vert_*+H_1/\lambda_i
%\end{align*}
Using $\bar{S}(T_i)=\Delta^*_S+S^*$, $\bar{L}(T_i)=\Delta^*_L+L^*$. We split $\Delta^*_S$ into its on-support and off-support part. We also divide $\Delta^*_L$ into its projection onto $V$ and $V^{\perp}$. $V$ is range of $L^*$. Meaning $\forall X \in V, \Vert X\Vert_* \leq r$. Therefore,
\begin{align*}
\Vert (\bar{S}(T_i))_\supp\Vert_1 &\geq \Vert (S^*)_\supp\Vert_1-\Vert (\Delta^*_S)_\supp\Vert_1 \\
\Vert (\bar{S}(T_i))_{\supp^c}\Vert_1 &\geq -\Vert (S^*)_{\supp^c}\Vert_1+\Vert (\Delta^*_S)_{\supp^c}\Vert_1,
\end{align*}
and
\begin{align*}
\Vert (\bar{L}(T_i))_V\Vert_* &\geq \Vert (L^*)_V\Vert_*-\Vert (\Delta^*_L)_V \Vert_* \\
\Vert (\bar{L}(T_i))_{V^{\perp}}\Vert_* &\geq -\Vert (L^*)_{V^{\perp}}\Vert_*+\Vert (\Delta^*_L)_{V^{\perp}}\Vert_*.
\end{align*}
Consequently,
\begin{align}
\label{eq:supp_split}
&\Vert (\Delta^*_S)_{\supp^c}\Vert_1+\frac{\mu_i}{\lambda_i}\Vert (\Delta^*_L)_{V^{\perp}}\Vert_* \leq \Vert (\Delta^*_S)_\supp\Vert_1+\frac{\mu_i}{\lambda_i}\Vert (\Delta^*_L)_V \Vert_*+H_1/\lambda_i.
\end{align}
$\Delta_S^*(T_i)-\hat{\Delta}_S(T_i)=\hat{S}_i-S^*$.
Therefore,
 \begin{align*}
 &\Vert \hat{S}_i-S^* \Vert_1=\\ &\Vert (\Delta^*_S(T_i))_\supp-(\hat{\Delta}_S(T_i))_\supp \Vert_1+ \Vert (\Delta^*_{S}(T_i))_{\supp^c}-(\hat{\Delta}_{S}(T_i))_{\supp^c} \Vert_1 \geq\\
 & \left\lbrace \Vert (\Delta^*_S(T_i))_\supp\Vert_1-\Vert (\hat{\Delta}_S(T_i))_\supp \Vert_1\right\rbrace-\left\lbrace \Vert (\Delta^*_{S}(T_i))_{\supp^c}\Vert_1-\Vert (\hat{\Delta}_{S}(T_i))_{\supp^c} \Vert_1\right\rbrace.
 \end{align*}
Hence,
 \begin{align*}
&\Vert (\hat{\Delta}_{S}(T_i))_{\supp^c}\Vert_1-\Vert (\hat{\Delta}_S(T_i))_\supp\Vert_1 \\ &\leq
\Vert (\Delta^*_{S}(T_i))_{\supp^c}\Vert_1-\Vert (\Delta^*_S(T_i))_\supp\Vert_1+\Vert \hat{S}_i-S^*\Vert_1.
 \end{align*}
As Equation \eqref{eq:reg_cond} is satisfied, we can use Lemma 1~\citep{negahban2012unified}. Combining the result with Lemma \ref{lemma:A1_latent}, we have $\Vert \hat{S}_i-S^*\Vert_1^2 \leq (4s+3r)(s\frac{\lambda_i^2}{\gamma^2}+r\frac{\mu_i^2}{\gamma^2} )$.
Consequently, further use of Lemma \ref{lemma:A1_latent} and the inequality $\Vert (\hat{\Delta}_{S}(T_i))_\supp \Vert_1 \leq \sqrt{s}\Vert \hat{\Delta}(T_i) \Vert_\mathbb{F}$ allows us to conclude that there exists a universal constant $c$ such that
\begin{align*}
\Vert \hat{\Delta}_S(T_i) \Vert_1^2
&\leq  4s\Vert \hat{\Delta}_S(T_i) \Vert_\mathbb{F}^2+(H_1/\lambda_i)^2+c(s+r)(s\frac{\lambda_i^2}{\gamma^2}+r\frac{\mu_i^2}{\gamma^2} ) \\&\quad+cr\frac{\mu_i^2}{\lambda_i^2}\left[\frac{2}{\gamma}H_1+(\frac{\alpha}{\sqrt{p}}+\frac{{p}}{\tau T_i})^2+s\frac{\lambda_i^2}{\gamma^2}+r\frac{\mu_i^2}{\gamma^2}\right]\\
&\leq 4s\left[\frac{2}{\gamma}H_1+(\frac{\alpha}{\sqrt{p}}+\frac{{p}}{\tau T_i})^2\right]+(H_1/\lambda_i)^2+c(s+r)(s\frac{\lambda_i^2}{\gamma^2}+r\frac{\mu_i^2}{\gamma^2} )\\ 
&\quad+cr\frac{\mu_i^2}{\lambda_i^2}\left[\frac{2}{\gamma}H_1+(\frac{\alpha}{\sqrt{p}}+\frac{{p}}{\tau T_i})^2 +s\frac{\lambda_i^2}{\gamma^2}+r\frac{\mu_i^2}{\gamma^2}\right],
\end{align*}
with probability at least $1-\delta_i$.
Optimizing the above bound with choice of $\lambda_i$ and complying with the conditions in Lemma \ref{lemma:lambda2}, leads to
\begin{align*}
\lambda_i^2=\frac{\gamma}{s+r}H_1+\frac{\alpha^2}{p^2}+\frac{\beta^2(p)\sigma^2}{T_i}\left(\log p+\log\frac{1}{\delta}\right).
\end{align*}
Repeating the same calculations for $\Vert \hat{\Delta}_L(T_i) \Vert_*$ results in
\begin{align*}
\mu_i^2=c_\mu \lambda_i^2,
\end{align*}
%Assuming $r < s$,
we have
\begin{align*}
\Vert \hat{\Delta}_S(T_i) \Vert_1^2 \leq c(s+r+\frac{s+r}{\gamma})H_1+c(s+r)(1+\frac{s+r}{p\gamma^2})\frac{\alpha^2}{p}+(s+r)(\frac{p^2}{\tau T_i^2}+\frac{\alpha}{\tau T_i}).
\end{align*}
Therefore,
\begin{align}
\label{eq:A68_l}
\Vert {\Delta_S^*}(T_i) \Vert_1^2  &\leq 2\Vert \hat{\Delta}_S(T_i) \Vert_1^2 +2\Vert S^*- \hat{S}_i\Vert_1^2 \\\nonumber  &\leq 2\Vert \hat{\Delta}(T_i) \Vert_1^2+8c(s+r)(s\frac{\lambda_i^2}{\gamma^2}+r\frac{\mu_i^2}{\gamma^2} ) \\ \nonumber
& \leq c(s+r+\frac{s+r}{\gamma})H_1+c(s+r)(1+\frac{s+r}{p\gamma^2})\frac{\alpha^2}{p}
+(s+r)(\frac{p^2}{\tau T_i^2}+\frac{\alpha}{\tau T_i}).
\end{align}
%where the last inequality follows since $\gamma \geq \bar{\gamma}$.
Finally, in order to use $\bar{S}(T_i)$ as the next prox center $\tilde{S}_{i+1}$, we would also like to control the error $\Vert \bar{S}(T_i)-\hat{S}_{i+1}\Vert_1^2$. Without loss of generality, we can design $\tilde{R}_i=c_r R_i$ for any $0\leq c_r  \leq 1$. The result only changes in a constant factor. Hence, we use $\tilde{R}_i=R_i$. Since $\lambda_{i+1}\leq \lambda_i$ by assumption, we obtain the same form of error bound as in~\eqref{eq:A68_l}. We want to run the epoch till all these error terms drop to $R_{i+1}^2 := R_i^2/2$. It suffices to set the epoch length $T_i$ to ensure that
 sum of all terms in~\eqref{eq:A68_l} is not greater that $R_i^2/2$.
All above conditions are met if we choose the epoch length
\begin{align*}
T_i& \simeq C  (s+r+\frac{s+r}{\gamma})^2\left[ \frac{\log p+12\beta^2(p)\sigma_i^2\log\frac{6}{\delta}}{R_i^2}\right]\\
&~~~+C(s+r+\frac{s+r}{\gamma})\!\left[\!\frac{{\beta(p)} G \sigma_i\sqrt{12\log\frac{6}{\delta}}}{R_i \sqrt{\log p}} \!+\!\frac{G}{R_i}+\rho_x\!\right],
\end{align*}
for a suitably large universal constant $C$. Then, we have that

\begin{align*}
\Vert {\Delta_S^*}(T_i) \Vert_1^2 \leq \frac{c'}{\sqrt{C}}R_i^2+c(s+r)(1+\frac{s+r}{p\gamma^2})\frac{\alpha^2}{p}.
\end{align*}
 Since the second part of the upper bound does not shrink in time, we stop where two parts are equal. Namely, $R_i^2=c(s+r)(1+\frac{s+r}{p\gamma^2})\frac{\alpha^2}{p}$.

With similar analysis for $L$, we get
\begin{align*}
\Vert {\Delta_L^*}(T_i) \Vert_*^2 \leq \frac{c'}{\sqrt{C}}\frac{1}{1+\gamma}R_i^2+c\frac{(s+r)^2}{p\gamma^2}\frac{\alpha^2}{p}.
\end{align*}
\subsection{Proof of Guarantees with Fixed Epoch Length, Sparse + Low Rank Case}\label{sec:const_matrix_proof}
This is a special case of Theorem~\ref{thm:latent_optimal} (Appendix).
Note that this fixed epoch length results in a convergence rate that is worse by a factor of $\log p$. The key difference between this case and optimal epoch length setting of Theorem~\ref{thm:latent_optimal} is that in the latter we guaranteed error halving by the end of each epoch whereas with fixed epoch length that statement may not be possible after the number of epochs becomes large enough. Therefore, we need to show that in such case the error does not increase much to invalidate our analysis. Let $k^*$ be the epoch number such that error halving holds true until then. Next we demonstrate that error does not increase much for $k>k^*$.
The proof follows the same nature as that of Theorem~\ref{thm:sparse_const_epoch} (in the main text), Section ~\ref{sec:const_sparse_proof}, with
\begin{align*}
k^*\!:=\!\sup\! \left\lbrace \! i:2^{\frac{j}{2}+1}\! \leq\! \frac{cR_1\gamma}{s+r}\sqrt{\frac{T_0}{\log p+\beta^2(p)\sigma_i^2w^2}} \right\rbrace,
\end{align*}
for all epochs  $j \leq i $ and
\begin{align*}
k_0=\log \left(\frac{R_1 \gamma}{s+r}\sqrt{\frac{T}{\log p+\beta^2(p)\sigma^2w^2}}\right).
\end{align*}

\subsection{Proof of Guarantees for Sparse + Low Rank Graphical Model selection Problem} \label{sec:latent_graph_proof}
Here we prove Corollary 2.  Proof follows by using the bounds derived in Appendix~\ref{sec:sparse_graph_proof} for Taylor series expansion and following the lines of Theorem~\ref{thm:latent_optimal} proof as in Appendix~\ref{sec:theorem2proof}.

According to D.1, in order to prove guarantees, we first need to bound  $\Vert z_{k+1}-z_k \Vert_1$ and $\Vert z_k \Vert_\infty$. According to Equation~\eqref{eq:dual} and considering the imposed $\ell_1$ bound, this is equivalent to bound $\Vert g_{k+1}-g_k \Vert_1$ and $\Vert g_k \Vert_\infty$.$\Vert g_{k+1}-g_k \Vert_1$ and $\Vert g_k \Vert_\infty$. The rest of the proof follows on lines of Theorem 2 proof. On the other hand, Lipschitz property requires a bound on $\| g_k\|_1$, which is much more stringent.

Assuming we are in a close proximity of $M^*$, we can use Taylor approximation to locally approximate $M^{-1}$ by ${M^*}^{-1} $ as in~\citep{Ravikumar&etal:08Arxiv}
\begin{align*}
M^{-1}={M^*}^{-1}-{M^*}^{-1}\Delta{M^*}^{-1}+\mathcal{R}(\Delta),
\end{align*}
where $\Delta=M-M^*$ and $\mathcal{R}(\Delta)$ is the remainder term.
We have
\begin{align*}
\Vert g_{k+1}-g_k \Vert_1 \leq \gennorm {\Gamma^*}_\infty \Vert M_{k+1}-M_k\Vert_1,
\end{align*}
and
\begin{align*}
\Vert g_k \Vert_\infty&\leq \Vert g_k-\mathbb{E}(g_k)\Vert_\infty+\Vert \mathbb{E}(g_k) \Vert_\infty  \\
&\leq \Vert e_k\Vert_\infty+\Vert \Sigma^*-M_k^{-1}\Vert_\infty \\ &\leq \sigma+\Vert \Gamma^* \Vert_\infty \Vert M_{k+1}-M_k\Vert_1.
\end{align*}
The term $\Vert M_{k+1}-M_k\Vert_1$ is bounded by $2\breve{R}$ by construction. We assume $\gennorm{\Gamma^*}_\infty$ and $\Vert \Gamma^*\Vert_\infty$ are bounded.

The error $\Delta$ needs to be ``small enough'' for the   $\mathcal{R}(\Delta)$ to be negligible, and we now provide the conditions for this.
By definition, $\mathcal{R}(\Delta)=\sum_{k=2}^{\infty} (-1)^k ({M^*}^{-1}\Delta)^k {M^*}^{-1}$.
Using triangle inequality and sub-multiplicative property for Frobenious norm,
\begin{align*}
\Vert \mathcal{R}(\Delta)\Vert_\mathbb{F} \leq \frac{\Vert {M^*}^{-1}\Vert_\mathbb{F}\Vert\Delta {M^*}^{-1}\Vert_\mathbb{F}^2}{1-\Vert\Delta {M^*}^{-1}\Vert_\mathbb{F}}.
\end{align*}
For $\|\Delta\|_{\mathbb{F}}\leq 2\breve{R} \leq  \frac{0.5}{\Vert {M^*}^{-1}\Vert_\mathbb{F}} $, we get
\begin{align*}
\Vert \mathcal{R}(\Delta)\Vert_\mathbb{F} \leq \Vert {M^*}^{-1}\Vert_\mathbb{F} .
\end{align*}
We assume  $\Vert \Sigma^* \Vert_\mathbb{F}$ is bounded.

%Note that $\lbrace R_i \rbrace_{i=1}^{k_T}$ is a decreasing sequence and we only need to bound $\breve{R}$.
Therefore, if the variables are closely-related we need to start with a small $\breve{R}$. For weaker correlations, we can start in a bigger ball.
The rest of the proof follows the lines of proof for Theorem~\ref{thm:latent_optimal}, by replacing $G^2$ by  $\gennorm{\Gamma^*}_\infty  \breve{R}(\sigma+\Vert \Gamma^*\Vert_\infty  \breve{R})$.

\section{Implementation} \label{sec:Implementation}
Here we discuss the updates for REASON 1 and REASON 2.
Note that for any vector $v$, $v_{(j)}$ denotes the $j$-th entry.
%, the updates are either in closed form or can be written in the following form
%\begin{align}
%\label{eq:project}
%\underset{w}{\min}~~ L(w)~~ \st~~ \Vert w \Vert_1 \leq R.
%\end{align}
%To start, consider the quadratic $L$ function. i.e., $L(w)=\Vert w-v \Vert_2^2$.
\subsection{Implementation details for REASON 1} \label{sec:implement1}
Let us start with REASON 1. We have already provided closed form solution for $y$ and $z$. The update rule for $\theta$ can be written as
\begin{align}
\label{eq:project}
&\underset{w}{\min}~~\Vert w-v \Vert_2^2~~ \st~~ \Vert w \Vert_1 \leq R,\\ \nonumber
%\end{align}
%& \text{with} \\ \nonumber
%\begin{align*}
&w=\theta-\tilde{\theta}_i,\\ \nonumber
&R=R_i,\\ \nonumber
%&L(w)=\Vert w-v\Vert^2,\\ \nonumber
&v=\frac{1}{\rho+\rho_x}[y_k-\tilde{\theta}_i-\frac{f(\theta_k)}{\rho}+\frac{z_k}{\rho}+\frac{\rho_x}{\rho}(\theta_k-\tilde{\theta}_i)].
\end{align}
We note that if $\Vert v\Vert_1 \leq R$, the answer is $w=v$. Else, the optimal solution is on the boundary of the constraint set and we can replace the inequality constraint with $\Vert w \Vert_1 = R$. Similar to~\citep{duchi2008efficient}, we perform Algorithm~\ref{l1_project} for solving~\eqref{eq:project}. The complexity of this Algorithm is $\mathcal{O}(d \log d)$, $d=p^2$.
\begin{algorithm}[t]
\caption{Implementation of $\theta$-update}
\label{l1_project}
\begin{algorithmic}
\STATE \textbf{Input:} A vector $v=\frac{1}{\rho+\rho_x}[y_k-\tilde{\theta}_i-\frac{\nabla f(\theta_k)}{\rho}+\frac{z_k}{\rho}+\frac{\rho_x}{\rho}(\theta_k-\tilde{\theta}_i)]$ and a scalar $R=R_i>0$
\IF {$\Vert v\Vert_1 \leq R$,}
{\STATE Output: $\theta=v+\tilde{\theta}_i$}
\ELSE
{
\STATE Sort $v$ into $\mu$: $\mu_1 \geq \mu_2 \geq \cdots \geq \mu_d$.
\STATE Find $\kappa=\max \lbrace j \in [d]: \mu_j-\frac{1}{j}\big(\sum_{i=1}^j \mu_i-R \big) > 0 \rbrace.$
\STATE Define $\zeta=\frac{1}{\kappa}\big(\sum_{i=1}^\kappa\mu_i-R \big)$
\STATE Output: $\theta$, where $\theta_{(j)}=\sign(v_{(j)})\max \lbrace v_{(j)}-\zeta,0\rbrace+(\tilde{\theta_i})_{(j)}$}
\ENDIF
\end{algorithmic}
\end{algorithm}

%This is very similar to projection into simplex with the difference that the sign matters. Let $u=\vert v \vert$. We can reformulate the problem as
%\begin{align}
%\label{eq:project_simplex}
%\underset{\beta}{\min}~~ \Vert \beta-u\Vert_2^2~~ \st~~ \Vert \beta \Vert_1 \leq R ~~\text{and}~~ \beta \geq 0.
%\end{align}
%The solution for~\eqref{eq:project} can be reached by solving~\eqref{eq:project_simplex} and setting $w_{(j)}=\text{sign}(v_{(j)})\beta_{(j)}$~\citep{duchi2008efficient}. Note that $w_{(j)}$ stands for $j$-th entry of $w$.
%Equation~\eqref{eq:project_simplex} can be solved easily using Algorithm 1~\citep{duchi2008efficient}. It should be noted that for cases where gradient of $L(w)$ is not necessarily sparse, Algorithms 2,3~\citep{duchi2008efficient} do not provide better performance.
%%

%Therefore, the solution can be easily reached as readily described.
%
\subsection{Implementation details for REASON 2} \label{sec:implement2}
For REASON 2, the update rule for $M$, $Z$, $Y$ and $U$ are in closed form.
%$Y$-update is readily the definition of projection into infinity norm. We can rewrite the $Y$-update as
%\begin{align*}
%\underset{\Vert Y \Vert_\infty \leq \frac{\alpha}{p}}{\min} \Vert Y-(L_{k+1}-U_k/\rho) \Vert_\mathbb{F}^2.
%\end{align*}
%Let $Y_{(j)}$ stand for $j$-th entry of $\text{vector}(Y)$. The for any $j$-th entry of $\text{vector}(Y)$, solution will be as follows
%\begin{align*}
%&\If \quad\vert (L_{k+1}-U_k/\rho)_{(j)}\vert \leq \frac{\alpha}{p},\quad \text{then}\quad Y_{(j)}=(L_{k+1}-U_k/\rho)_{(j)} \\
%&\Else \quad Y_{(j)}=\sign\big((L_{k+1}-U_k/\rho)_{(j)} - \frac{\alpha}{p}\big)\frac{\alpha}{p}
%%&\text{else} \quad Y_{(j)}=-\frac{\alpha}{p}.
%\end{align*}
%
%
\begin{algorithm}[t]
\caption{Implementation of $S$-update}
\label{S_project}
\begin{algorithmic}
\STATE \textbf{Input:} $W^{(1)}=\text{vector}(S_k-\tilde{S}_i)$
%A vector $v=\text{vector}(S_k-\tilde{S}_i-\eta_t \nabla^{(t)}(\lambda_i\Vert S-\tilde{S}_i+\tilde{S}_i\Vert_1+\frac{\rho}{2\tau_k}\Vert S_k-\tilde{S}_i-(S_k+\tau_kG_{M_k}-\tilde{S}_i\Vert_\mathbb{F}^2))$
and a scalar $R=R_i>0$\\
\FOR {$t=1$ to $t=t_s$}
{\STATE $v=W^{(t)}-\eta_t \left[\lambda_i \nabla^{(t)} \Vert W^{(t)}+\text{vector}({\tilde{S}_i})\Vert_1+\frac{\rho}{\tau_k}\left(W^{(t)}-\text{vector}(S_k+\tau_kG_{M_k}-\tilde{S}_i)\right)\right] $
\IF {$\Vert v\Vert_1 \leq R$,}
{\STATE $W^{(t+1)}=v$}
\ELSE
{
\STATE Sort $v$ into $\mu$: $\mu_1 \geq \mu_2 \geq \cdots \geq \mu_d$.
\STATE Find $\kappa=\max \lbrace j \in [d]: \mu_j-\frac{1}{j}\big(\sum_{i=1}^j \mu_i-R \big) > 0 \rbrace.$
\STATE Define $\zeta=\frac{1}{\kappa}\big(\sum_{i=1}^\kappa\mu_i-R \big)$
\STATE For $1 \leq j \leq d$, $W^{(t+1)}_{(j)}=\sign(v_{(j)})\max \lbrace v_{(j)}-\zeta,0\rbrace$
%+\text{vector}\\
%(\tilde{S_i}_{(j)})$
}
\ENDIF
}\ENDFOR
\STATE \textbf{Output}:$\text{matrix}(W^{(t_s)})+\tilde{S}_i$
\end{algorithmic}
\end{algorithm}
Consider the $S$-update. It can be written in form of~\eqref{eq:project} with
\begin{align*}
%\label{eq:projectS}
\underset{W}{\min}~~\lambda_i\Vert W+\tilde{S}_i\Vert_1+\frac{\rho}{2\tau_k}\Vert W-(S_k+\tau_kG_{M_k}-\tilde{S}_i)\Vert_\mathbb{F}^2.~~ \st~~ \Vert W \Vert_1 \leq R,
\end{align*}
\begin{align*}
W=S-\tilde{S}_i, \quad R=R_i.
\end{align*}
 Therefore, similar to~\citep{duchi2008efficient}, we generate a sequence of $\lbrace W^{(t)}\rbrace_{t=1}^{t_s} $ via
\begin{align*}
W^{(t+1)}=\Pi_1 \left[W^{(t)}-\eta_t \nabla^{(t)}\left(\lambda_i\Vert W+\tilde{S}_i\Vert_1+\frac{\rho}{2\tau_k}\Vert W-(S_k+\tau_kG_{M_k}-\tilde{S}_i)\Vert_\mathbb{F}^2\right)\right],
\end{align*}
where $\Pi_1$ is projection on to $\ell_1$ norm,
%\begin{align*}
%\Pi_1(W)=\underset{V}{\arg\min \Vert W-V\Vert_\mathbb{F} , \Vert
%\end{align*}
similar to Algorithm~\ref{l1_project}. In other words, at each iteration, $\text{vector}\left(W^{(t)}-\eta_t \left[\lambda_i \nabla^{(t)} \Vert W^{(t)}+\tilde{S}_i \Vert_1+\frac{\rho}{\tau_k}(W^{(t)}-(S_k+\tau_kG_{M_k}-\tilde{S}_i))\right]\right)$ is the input to Algorithm~\ref{l1_project} (instead of vector $v$) and the output is $ \text{vector}(W^{(t+1)}) $. The term $\nabla^{(t)} \Vert W^{(t)}+\tilde{S}_i \Vert_1$ stands for subgradient of the $\ell_1$ norm $\Vert W^{(t)}+\tilde{S}_i \Vert_1$. The $S$-update is summarized is Algorithm~\ref{S_project}. A step size of $\eta_t \propto 1/\sqrt{t}$ guarantees a convergence rate of $\mathcal{O}(\sqrt{\log p/T})$~\citep{duchi2008efficient}.

The $L$-update is very similar in nature to the $S$-update. The only difference is that the projection is on to nuclear norm instead of $\ell_1$ norm. It can be done by performing an SVD before the $\ell_1$ norm projection.  

The code for REASON 1 follows directly from the discussion in Section~\ref{sec:implement1}. For REASON 2 on the other hand, we have added additional heuristic modifications to improve the performance.  REASON 2 code is available at {\url{https://github.com/haniesedghi/REASON2}}.
The first modification is that we do not update the dual variable $Z$ per every iteration on $S$ and $L$. Instead, we update the dual variable once $S$ and $L$ seem to have converged to some value or after every $m$ iterations on $S$ and $L$. The reason is that once we start the iteration, $S$ and $L$ can be far from each other which results in a big dual variable and hence, a slower convergence. The value of $m$ can be set based on the problem. For the experiments discussed in the paper we have used $m=4$.

Further investigation on REASON 2 shows that performing one of the projections (either $\ell_1$ or nuclear norm) suffices to reach this performance. The same precision can be reached using only one of the projections. Addition of the second projection improves the performance only marginally.  Performing nuclear norm projections are much more expensive since they require SVD. Therefore, it is more efficient to perform the $\ell_1$ projection. In the code, we leave it as an option to run both projections.


\begin{thebibliography}{41}
\providecommand{\natexlab}[1]{#1}
\providecommand{\url}[1]{\texttt{#1}}
\expandafter\ifx\csname urlstyle\endcsname\relax
  \providecommand{\doi}[1]{doi: #1}\else
  \providecommand{\doi}{doi: \begingroup \urlstyle{rm}\Url}\fi

\bibitem[Agarwal et~al.(2012{\natexlab{a}})Agarwal, Negahban, and
  Wainwright]{agarwal2012noisy}
A.~Agarwal, S.~Negahban, and M.~Wainwright.
\newblock Noisy matrix decomposition via convex relaxation: Optimal rates in
  high dimensions.
\newblock \emph{The Annals of Statistics}, 40\penalty0 (2):\penalty0
  1171--1197, 2012{\natexlab{a}}.

\bibitem[Agarwal et~al.(2012{\natexlab{b}})Agarwal, Negahban, and
  Wainwright]{AgarwalNW12}
A.~Agarwal, S.~Negahban, and M.~J. Wainwright.
\newblock Stochastic optimization and sparse statistical recovery: Optimal
  algorithms for high dimensions.
\newblock In \emph{NIPS}, pages 1547--1555, 2012{\natexlab{b}}.

\bibitem[Anandkumar et~al.(2012)Anandkumar, Tan, Huang, and Willsky]{AnimaW}
A.~Anandkumar, V.~Tan, F.~Huang, and A.S. Willsky.
\newblock High-dimensional gaussian graphical model selection:walk summability
  and local separation criterion.
\newblock \emph{Journal of Machine Learning}, 13:\penalty0 2293–2337, August
  2012.

\bibitem[Boyd et~al.(2011)Boyd, Parikh, Chu, Peleato, and
  Eckstein]{boyd2011distributed}
S.~Boyd, N.~Parikh, E.~Chu, B.~Peleato, and J.~Eckstein.
\newblock Distributed optimization and statistical learning via the alternating
  direction method of multipliers.
\newblock \emph{Foundations and Trends{\textregistered} in Machine Learning},
  3\penalty0 (1):\penalty0 1--122, 2011.

\bibitem[Cand{\`e}s et~al.(2011)Cand{\`e}s, Li, Ma, and
  Wright]{candes2011robust}
E.~J. Cand{\`e}s, X.~Li, Y.~Ma, and J.~Wright.
\newblock Robust principal component analysis?
\newblock \emph{Journal of the ACM (JACM)}, 58\penalty0 (3):\penalty0 11, 2011.

\bibitem[Chandrasekaran et~al.(2011)Chandrasekaran, Sanghavi, Parrilo, and
  Willsky]{chandrasekaran2011rank}
V.~Chandrasekaran, S.~Sanghavi, Pablo~A Parrilo, and A.~S Willsky.
\newblock Rank-sparsity incoherence for matrix decomposition.
\newblock \emph{SIAM Journal on Optimization}, 21\penalty0 (2):\penalty0
  572--596, 2011.

\bibitem[Chandrasekaran et~al.(2012)Chandrasekaran, Parrilo, and
  Willsky]{chandrasekaran2012latent}
V.~Chandrasekaran, P.~A Parrilo, and Alan~S Willsky.
\newblock Latent variable graphical model selection via convex optimization.
\newblock \emph{The Annals of Statistics}, 40\penalty0 (4):\penalty0
  1935--1967, 2012.

\bibitem[d'Aspremont et~al.(2008)d'Aspremont, Banerjee, and
  El~Ghaoui]{d2008first}
Alexandre d'Aspremont, Onureena Banerjee, and Laurent El~Ghaoui.
\newblock First-order methods for sparse covariance selection.
\newblock \emph{SIAM Journal on Matrix Analysis and Applications}, 30\penalty0
  (1):\penalty0 56--66, 2008.

\bibitem[Deng(2012)]{deng2012global}
W.~Deng, W.and~Yin.
\newblock On the global and linear convergence of the generalized alternating
  direction method of multipliers.
\newblock Technical report, DTIC Document, 2012.

\bibitem[Duchi et~al.(2008)Duchi, Shalev-Shwartz, Singer, and
  Chandra]{duchi2008efficient}
J.~Duchi, S.~Shalev-Shwartz, Y.~Singer, and T.~Chandra.
\newblock Efficient projections onto the $\ell_1$-ball for learning in high
  dimensions.
\newblock In \emph{Proceedings of the 25th international conference on Machine
  learning}, pages 272--279. ACM, 2008.

\bibitem[Esser et~al.(2010)Esser, Zhang, and Chan]{esser2010general}
E.~Esser, X.~Zhang, and T.~Chan.
\newblock A general framework for a class of first order primal-dual algorithms
  for convex optimization in imaging science.
\newblock \emph{SIAM Journal on Imaging Sciences}, 3\penalty0 (4):\penalty0
  1015--1046, 2010.

\bibitem[Goldstein et~al.(2012)Goldstein, O’Donoghue, and
  Setzer]{goldstein2012fast}
T~Goldstein, B.~O’Donoghue, and S.~Setzer.
\newblock Fast alternating direction optimization methods.
\newblock \emph{CAM report}, pages 12--35, 2012.

\bibitem[Hsieh et~al.(2013)Hsieh, Sustik, Dhillon, Ravikumar, and
  Poldrack]{bigquick}
C.~Hsieh, M.~A Sustik, I.~Dhillon, P.~Ravikumar, and R.~Poldrack.
\newblock Big $\&$ quic: Sparse inverse covariance estimation for a million
  variables.
\newblock In \emph{Advances in Neural Information Processing Systems}, pages
  3165--3173, 2013.

\bibitem[Hsieh et~al.(2011)Hsieh, Sustik, Dhillon, and
  Ravikumar]{hsieh2011sparse}
Cho-Jui Hsieh, Matyas~A Sustik, Inderjit~S Dhillon, and Pradeep~D Ravikumar.
\newblock Sparse inverse covariance matrix estimation using quadratic
  approximation.
\newblock In \emph{NIPS}, pages 2330--2338, 2011.

\bibitem[Hsu et~al.(2011)Hsu, Kakade, and Zhang]{hsu2011robust}
Daniel Hsu, Sham~M Kakade, and Tong Zhang.
\newblock Robust matrix decomposition with sparse corruptions.
\newblock \emph{Information Theory, IEEE Transactions on}, 57\penalty0
  (11):\penalty0 7221--7234, 2011.

\bibitem[Johnson and Zhang(2013)]{NIPS2013_4937}
R.~Johnson and T.~Zhang.
\newblock Accelerating stochastic gradient descent using predictive variance
  reduction.
\newblock In C.J.C. Burges, L.~Bottou, M.~Welling, Z.~Ghahramani, and K.Q.
  Weinberger, editors, \emph{Advances in Neural Information Processing Systems
  26}, pages 315--323. 2013.

\bibitem[Lauritzen(1996)]{Lauritzen:book}
S.L. Lauritzen.
\newblock \emph{{Graphical models}}.
\newblock Clarendon Press, 1996.

\bibitem[Lerman et~al.(2012)Lerman, McCoy, Tropp, and Zhang]{lerman2012robust}
Gilad Lerman, Michael McCoy, Joel~A Tropp, and Teng Zhang.
\newblock Robust computation of linear models, or how to find a needle in a
  haystack.
\newblock \emph{arXiv preprint arXiv:1202.4044}, 2012.

\bibitem[Lin et~al.(2010)Lin, Chen, and Ma]{lin2010augmented}
Z.~Lin, M.~Chen, and Y.~Ma.
\newblock The augmented lagrange multiplier method for exact recovery of
  corrupted low-rank matrices.
\newblock \emph{arXiv preprint arXiv:1009.5055}, 2010.

\bibitem[Loh and Wainwright(2013)]{loh2013regularized}
Po-Ling Loh and Martin~J Wainwright.
\newblock Regularized m-estimators with nonconvexity: Statistical and
  algorithmic theory for local optima.
\newblock In \emph{Advances in Neural Information Processing Systems}, pages
  476--484, 2013.

\bibitem[Luo(2012)]{linearconv}
Zhi-Quan Luo.
\newblock On the linear convergence of the alternating direction method of
  multipliers.
\newblock \emph{arXiv preprint arXiv:1208.3922}, 2012.

\bibitem[Ma et~al.(2012)Ma, Xue, and Zou]{ADMM_latent}
S.~Ma, L.~Xue, and H.~Zou.
\newblock Alternating direction methods for latent variable gaussian graphical
  model selection.
\newblock \emph{arXiv preprint arXiv:1206.1275v2}, 2012.

\bibitem[Malioutov et~al.(2006)Malioutov, Johnson, and
  Willsky]{malioutov2006walk}
Dmitry~M Malioutov, Jason~K Johnson, and Alan~S Willsky.
\newblock Walk-sums and belief propagation in gaussian graphical models.
\newblock \emph{The Journal of Machine Learning Research}, 7:\penalty0
  2031--2064, 2006.

\bibitem[Mota et~al.(2012)Mota, Xavier, Aguiar, and
  Puschel]{mota2012distributed}
J.~FC Mota, J.~MF Xavier, P.~MQ Aguiar, and M.~Puschel.
\newblock Distributed admm for model predictive control and congestion control.
\newblock In \emph{Decision and Control (CDC), 2012 IEEE 51st Annual Conference
  on}, pages 5110--5115. IEEE, 2012.

\bibitem[Negahban et~al.(2012)Negahban, Ravikumar, Wainwright, and
  Yu]{negahban2012unified}
S.~Negahban, P.~Ravikumar, M.~Wainwright, and B.~Yu.
\newblock A unified framework for high-dimensional analysis of {{M}}-estimators
  with decomposable regularizers.
\newblock \emph{Statistical Science}, 27\penalty0 (4):\penalty0 538--557, 2012.

\bibitem[Ouyang et~al.(2013)Ouyang, He, Tran, and Gray]{ouyang2013stochastic}
H.~Ouyang, N.~He, L.~Tran, and A.~G Gray.
\newblock Stochastic alternating direction method of multipliers.
\newblock In \emph{Proceedings of the 30th International Conference on Machine
  Learning (ICML-13)}, pages 80--88, 2013.

\bibitem[Pitman and Ross(2012)]{pitman2012archimedes}
Jim Pitman and Nathan Ross.
\newblock Archimedes, gauss, and stein.
\newblock \emph{Notices AMS}, 59:\penalty0 1416--1421, 2012.

\bibitem[Raskutti et~al.(2011)Raskutti, Wainwright, and
  Yu]{raskutti2011minimax}
G.~Raskutti, M.~J. Wainwright, and B.~Yu.
\newblock Minimax rates of estimation for high-dimensional linear regression
  over $\ell_q$-balls.
\newblock \emph{IEEE Trans. Information Theory}, 57\penalty0 (10):\penalty0
  6976---6994, October 2011.

\bibitem[Ravikumar et~al.(2011)Ravikumar, Wainwright, Raskutti, and
  Yu]{Ravikumar&etal:08Arxiv}
P.~Ravikumar, M.J. Wainwright, G.~Raskutti, and B.~Yu.
\newblock {High-dimensional covariance estimation by minimizing
  $\ell_1$-penalized log-determinant divergence}.
\newblock \emph{Electronic Journal of Statistics}, \penalty0 (4):\penalty0
  935--980, 2011.

\bibitem[Roux et~al.(2012)Roux, Schmidt, and Bach]{roux2012stochastic}
Nicolas~Le Roux, Mark Schmidt, and Francis Bach.
\newblock A stochastic gradient method with an exponential convergence rate for
  strongly-convex optimization with finite training sets.
\newblock Technical report, 2012.

\bibitem[Shalev-Shwartz(2011)]{shalev2011online}
S.~Shalev-Shwartz.
\newblock Online learning and online convex optimization.
\newblock \emph{Foundations and Trends in Machine Learning}, 4\penalty0
  (2):\penalty0 107--194, 2011.

\bibitem[Shalev-Shwartz and Zhang(2013)]{NIPS2013_4938}
S.~Shalev-Shwartz and T.~Zhang.
\newblock Accelerated mini-batch stochastic dual coordinate ascent.
\newblock In C.J.C. Burges, L.~Bottou, M.~Welling, Z.~Ghahramani, and K.Q.
  Weinberger, editors, \emph{Advances in Neural Information Processing Systems
  26}, pages 378--385. 2013.

\bibitem[Tropp(2012)]{tropp2012user}
J.~Tropp.
\newblock User-friendly tail bounds for sums of random matrices.
\newblock \emph{Foundations of Computational Mathematics}, 12\penalty0
  (4):\penalty0 389--434, 2012.

\bibitem[Vershynin(2010)]{vershynin2010introduction}
Roman Vershynin.
\newblock Introduction to the non-asymptotic analysis of random matrices.
\newblock \emph{arXiv preprint arXiv:1011.3027}, 2010.

\bibitem[Vu(2005)]{vu2005spectral}
Van~H Vu.
\newblock Spectral norm of random matrices.
\newblock In \emph{Proceedings of the thirty-seventh annual ACM symposium on
  Theory of computing}, pages 423--430. ACM, 2005.

\bibitem[Wahlberg et~al.(2012)Wahlberg, Boyd, Annergren, and
  Wang]{wahlberg2012admm}
B.~Wahlberg, S.~Boyd, M.~Annergren, and Y.~Wang.
\newblock An admm algorithm for a class of total variation regularized
  estimation problems.
\newblock \emph{arXiv preprint arXiv:1203.1828}, 2012.

\bibitem[Wang et~al.(2013{\natexlab{a}})Wang, Chen, Smola, and
  Xing]{NIPS2013_5034}
C.~Wang, X.~Chen, A.~Smola, and E.~Xing.
\newblock Variance reduction for stochastic gradient optimization.
\newblock In C.J.C. Burges, L.~Bottou, M.~Welling, Z.~Ghahramani, and K.Q.
  Weinberger, editors, \emph{Advances in Neural Information Processing Systems
  26}, pages 181--189. 2013{\natexlab{a}}.

\bibitem[Wang and Banerjee(2013)]{BADMM}
H.~Wang and A.~Banerjee.
\newblock Bregman alternating direction method of multipliers.
\newblock \emph{arXiv preprint arXiv:1306.3203}, 2013.

\bibitem[Wang et~al.(2013{\natexlab{b}})Wang, Hong, Ma, and
  Luo]{wang2013solving}
X.~Wang, M.~Hong, S.~Ma, and Z.~Luo.
\newblock Solving multiple-block separable convex minimization problems using
  two-block alternating direction method of multipliers.
\newblock \emph{arXiv preprint arXiv:1308.5294}, 2013{\natexlab{b}}.

\bibitem[Wang et~al.(2013{\natexlab{c}})Wang, Liu, and Zhang]{wang2013optimal}
Zhaoran Wang, Han Liu, and Tong Zhang.
\newblock Optimal computational and statistical rates of convergence for sparse
  nonconvex learning problems.
\newblock \emph{arXiv preprint arXiv:1306.4960}, 2013{\natexlab{c}}.

\bibitem[Watson(1992)]{Watson1992}
G.~Watson.
\newblock Characterization of the subdifferential of some matrix norms.
\newblock \emph{Linear Algebra and its Applications}, 170\penalty0
  (0):\penalty0 33--45, 1992.

\end{thebibliography}
\end{document}